\documentclass{article}
\usepackage{authblk}
\usepackage{tgpagella}
\usepackage[dvips,letterpaper,margin=1in]{geometry}

\usepackage{silence}
\usepackage[utf8]{inputenc}
\usepackage[export]{adjustbox}
\usepackage[T1]{fontenc}
\usepackage{tabularx}
\usepackage{tcolorbox}
\usepackage{changepage}
\usepackage{hyperref}
\usepackage{lipsum}
\usepackage{url}
\usepackage{placeins}
\usepackage{booktabs}
\usepackage{soul}
\usepackage{amsfonts}
\usepackage{nicefrac}
\usepackage{microtype}
\usepackage[numbers,sort]{natbib}
\usepackage{proba}
\usepackage{wrapfig}
\usepackage{caption}
\usepackage{amssymb}
\usepackage{pifont}

\usepackage{subcaption}
\usepackage{float}
\usepackage{rotating}
\usepackage{etoolbox}
\usepackage{xparse}
\usepackage{grffile}
\usepackage{amsmath}
\usepackage{xspace}
\usepackage{euscript}
\usepackage{enumitem}
\usepackage{mathtools}
\usepackage{tikz}
\usepackage{tablefootnote}
\usepackage[flushleft]{threeparttable}
\usepackage{multirow}
\usepackage[toc,page,header,title]{appendix}

\usepackage{minitoc}
\usepackage{arydshln}
\usepackage{array, booktabs}
\usepackage{comment}
\usepackage{graphicx}
\usepackage{adjustbox}
\usepackage{dsfont}
\usepackage{amsmath,amsthm}
\usepackage{balance}
\usepackage{multicol}
\usepackage{nameref}
\usepackage{pdfpages}
\usepackage{braket}
\usepackage[normalem]{ulem}
\usepackage{bbding}
\usepackage[capitalise,noabbrev]{cleveref}   
\usepackage[usestackEOL]{stackengine}
\usepackage{indentfirst}
\usepackage{csquotes}
\usepackage{pgf}
\usepackage{varwidth}
\usepackage{verbatimbox}
\usepackage{verbatim}
\usepackage{bm}
\usepackage{thm-restate}
\usepackage{environ}
\usepackage{orcidlink}

\definecolor{pastelblue}{RGB}{76,113,175}
\definecolor{pastelgreen}{RGB}{144,238,144}
\definecolor{pastelred}{RGB}{196,78,82}
\definecolor{pastelgrey}{RGB}{230,230,230}
\definecolor{pastelbeige}{RGB}{243,236,221}
\definecolor{pastelpurple}{RGB}{154,139,192}
\definecolor{salmon}{RGB}{250, 128, 114}
\definecolor{darkgreen}{rgb}{0,0.6,0}
\definecolor{darkred}{rgb}{0.5,0,0}
\definecolor{verylightgreen}{HTML}{F6FFF9}
\definecolor{verylightred}{HTML}{FFF4F3}
\definecolor{verylightgray}{HTML}{F4F6F6}
\definecolor{navyblue}{HTML}{473992}
\definecolor{redorange}{HTML}{ED135A}
\definecolor{babyblueeyes}{rgb}{0.63, 0.79, 0.95}
\definecolor{lightpink}{rgb}{1.00, 0.714, 0.757}

\usepackage{tikzpeople}
\usetikzlibrary{shapes,decorations,arrows,calc,arrows.meta,fit,positioning, matrix}
\tikzset{ 
    table/.style={
        matrix of nodes,
        row sep=-\pgflinewidth,
        column sep=-\pgflinewidth,
        nodes={
            rectangle,
            draw=black,
            align=center
        },
        minimum height=0.5em,
        minimum width=4ex,
        text depth=6ex,
        text height=4ex,
        text width=4ex,
        nodes in empty cells,
        every even row/.style={
            nodes={fill=gray!20}
        },
        column 1/.style={
            nodes={text width=2em,font=\bfseries}
        },
        row 1/.style={
            nodes={
                fill=black,
                text=white,
                font=\bfseries
            }
        }
    }
}

\tikzset{
    -Latex,auto,node distance =1 cm and 1 cm,semithick,
    state/.style ={ellipse, draw, minimum width = 0.7 cm},
    point/.style = {circle, draw, inner sep=0.04cm,fill,node contents={}},
    bidirected/.style={Latex-Latex,dashed},
    el/.style = {inner sep=2pt, align=left, sloped}
}

\newcommand\Warning{%
 \makebox[1.4em][c]{%
 \makebox[0pt][c]{\raisebox{.1em}{\small!}}%
 \makebox[0pt][c]{\color{red}\Large$\bigtriangleup$}}}%

\newtheorem{assumption}{Assumption}

\newtheorem{corollary}{Corollary}

\newtheorem{definition}{Definition}

\newtheorem{proposition}{Proposition}

\makeatletter
\def\thmt@refnamewithcomma #1#2#3,#4,#5\@nil{%
	\@xa\def\csname\thmt@envname #1utorefname\endcsname{#3}%
	\ifcsname #2refname\endcsname
	\csname #2refname\expandafter\endcsname\expandafter{\thmt@envname}{#3}{#4}%
	\fi}
\makeatother
\Crefname{conjecture}{Conjecture}{Conjectures}
\Crefname{definition}{Definition}{Definitions}
\Crefname{observation}{Observation}{Observations}
\Crefname{assumption}{Assumption}{Assumptions}
\Crefname{axiom}{Axiom}{Axioms}
\Crefname{case}{Case}{Cases}
\Crefname{claim}{Claim}{Claims}
\Crefname{conclusion}{Conclusion}{Conclusions}
\Crefname{condition}{Condition}{Conditions}
\Crefname{criterion}{Criterion}{Criteria}
\Crefname{exercise}{Exercise}{Exercises}
\Crefname{example}{Example}{Examples}
\Crefname{notation}{Notation}{Notations}
\Crefname{problem}{Problem}{Problems}
\Crefname{property}{Property}{Properties}
\Crefname{remark}{Remark}{Remarks}
\Crefname{solution}{Solution}{Solutions}
\Crefname{summary}{Summary}{Summaries}
\Crefname{motivation}{Motivation}{Motivations}
\newcommand{\mathcolorbox}[2]{\colorbox{#1}{$\displaystyle #2$}}

\DeclareMathOperator*{\argmin}{argmin}
\DeclareMathOperator*{\argmax}{argmax}
\newcommand*\dbar[1]{\overline{\overline{\lower0.2ex\hbox{$#1$}}}}
\newcommand{\supp}{\mathrm{supp\xspace}}

\newcommand{\btheta}{\boldsymbol{\theta}}

\newcommand{\setsize}[1]{\mid \! #1 \! \mid}

\renewcommand{\Var}{\mathrm{Var}}

\newcommand{\Exp}{\mathds{E}}
\newcommand{\eqdist}{\stackrel{d}{=}}

\newcommand{\indep}{\perp\!\!\!\perp}

\newcommand{\Parents}{\mathrm{Pa}}
\newcommand{\parents}{\mathrm{pa}}
\newcommand{\iso}{\cong}

\newcommand{\aut}[1]{\mathrm{Aut}(#1)}

\def\cB{{\mathcal{B}}}

\def\cD{{\mathcal{D}}}

\def\cL{{\mathcal{L}}}

\def\cN{{\mathcal{N}}}
\def\cO{{\mathcal{O}}}

\def\cR{{\mathcal{R}}}

\def\bW{{\mathbf{W}}}

\def\bY{{\mathbf{Y}}}

\def\by{{\mathbf{y}}}

\DeclareFontFamily{U}{BOONDOX-calo}{\skewchar\font=45 }
\DeclareFontShape{U}{BOONDOX-calo}{m}{n}{
  <-> s*[1.05] BOONDOX-r-calo}{}
\DeclareFontShape{U}{BOONDOX-calo}{b}{n}{
  <-> s*[1.05] BOONDOX-b-calo}{}
\DeclareMathAlphabet{\mathcalb}{U}{BOONDOX-calo}{m}{n}
\SetMathAlphabet{\mathcalb}{bold}{U}{BOONDOX-calo}{b}{n}
\DeclareMathAlphabet{\mathbcalb}{U}{BOONDOX-calo}{b}{n}

\newcommand{\ecal}[1]{ \EuScript{#1} }

\renewcommand{\paragraph}[1]{{\noindent \textbf{#1.}}}
\newcommand{\Appendix}{supplement\xspace}

\renewcommand{\G}[1]{G^{(#1)}}
\newcommand{\Vt}[1]{V^{(#1)}}
\newcommand{\n}[1]{n^{(#1)}}
\newcommand{\Et}[1]{E^{(#1)}}
\renewcommand{\E}[1]{\ecal{E}^{(#1)}}
\renewcommand{\X}[1]{\ecal{X}^{(#1)}}
\newcommand{\e}[1]{e^{(#1)}}
\newcommand{\x}[1]{x^{(#1)}}
\newcommand{\ep}[1]{e^{(#1)'}}
\newcommand{\xp}[1]{x^{(#1)'}}
\renewcommand{\A}[2]{A^{(#1)}_{#2}}
\renewcommand{\a}[2]{a^{(#1)}_{#2}}
\renewcommand{\IJ}[0]{(I,J)}
\newcommand{\IJt}[1]{(I^{(t_{#1})},J^{(t_{#1})})}
\newcommand{\ijt}[1]{(i^{(t_{#1})},j^{(t_{#1})})}
\newcommand{\UV}[0]{(U,V)}
\newcommand{\pol}[0]{\mu(\G{t_0})}
\renewcommand{\Y}[1]{Y_{#1}^{(t_1)}}
\newcommand{\sY}[0]{\bY^{(M)}}
\newcommand{\sy}[0]{\by^{(M)}}
\newcommand{\Yt}[3]{Y_{{#1}^{(t_{#3})}{#2}^{(t_{#3})}}^{ ( t_{#3} ) } }
\newcommand{\domA}[0]{\mathbb{A}}
\renewcommand{\ij}[0]{(i,j)}
\newcommand{\uv}[0]{(u,v)}
\newcommand{\orb}[1]{\cO_{#1}}
\renewcommand{\U}[0]{\ecal{U}}
\newcommand{\m}[0]{\ecal{C}}
\renewcommand{\M}[0]{\mathbb{C}}
\renewcommand{\F}[0]{\ecal{F}}
\renewcommand{\V}[0]{\ecal{V}}
\newcommand{\Ux}[1]{\ecal{U}^{(#1)}_{\ecal{X}}}
\newcommand{\ux}[1]{u^{(#1)}_{\ecal{X}}}

\newcommand{\Ue}[1]{\ecal{U}^{(#1)}_{\ecal{E}}}
\newcommand{\fx}[1]{f^{(#1)}_{\ecal{X}}}

\newcommand{\fe}[1]{f^{(#1)}_{\ecal{E}}}
\def\th{^{\text{th}}}

\newcommand{\GGP}{\m}
\newcommand{\ORB}[1]{\cO^{(t_0)}_{#1}}
\def\XE{{\Xi^{(t_0)}}}
\newcommand{\Gammastar}{\Gamma^{\star}}
\newcommand{\Gammamostexp}{\Gamma^{\star}}
\newcommand{\Gammajoint}{\Gamma^{\text{(joint)}}}
\ActivateWarningFilters[pdftoc]

\hypersetup{
    colorlinks = true,
    linkcolor = navyblue,
    urlcolor  = black, 
    citecolor = redorange,
    anchorcolor = navyblue
}

\usepackage{authblk}

\makeatletter
\def\blfootnote{\G{t_0}def\@thefnmark{}\@footnotetext}
\makeatother

\title{\Huge Causal Lifting and Link Prediction}

\author[1]{Leonardo Cotta~\orcidlink{0000-0002-4751-7643}\footnote{Work partially done at Purdue University and Intel Labs.}\footnote{Correspondence to: \texttt{leonardo.cotta@vectorinstitute.ai}}}
\author[2]{Beatrice Bevilacqua~\orcidlink{0000-0002-9462-4599}}
\author[3]{Nesreen Ahmed~\orcidlink{0000-0002-7913-4962}}
\author[2]{Bruno Ribeiro~\orcidlink{0000-0002-3527-6192}}

\affil[1]{Vector Institute}
\affil[2]{Purdue University}
\affil[3]{Intel Labs}

\begin{document}

\doparttoc 
\faketableofcontents 

\maketitle
\begin{abstract}
Existing causal models for link prediction assume an underlying set of inherent node factors ---an innate characteristic defined at the node's birth--- that governs the causal evolution of links in the graph. In some causal tasks, however, link formation is \emph{path-dependent}: The outcome of link interventions depends on existing links.
Unfortunately, these existing causal methods are not designed for path-dependent link formation, as the cascading functional dependencies between links (arising from \emph{path dependence}) are either unidentifiable or require an impractical number of control variables. To overcome this, we develop the first causal model capable of dealing with path dependencies in link prediction.

In this work we introduce the concept of causal lifting, an invariance in causal models of independent interest that, on graphs, allows the identification of causal link prediction queries using limited interventional data. Further, we show how structural pairwise embeddings exhibit lower bias and correctly represent the task's causal structure, as opposed to existing node embeddings, \textit{e.g.}, graph neural network node embeddings and matrix factorization. Finally, we validate our theoretical findings on three scenarios for causal link prediction tasks: knowledge base completion, covariance matrix estimation and consumer-product recommendations.
\end{abstract}


\section{Introduction}

Predicting links between entities via latent factors has captivated the scientific community ever since Charles Spearman published {\em The Abilities of Man}~\cite{spearmanabilities} in 1927, where he described the mathematical tools to uncover latent {\em common factors} of intelligence just by observing a subject $i$ perform a task $j$ successfully ($A_{ij}=1$) or unsuccessfully ($A_{ij}=0$) over $n$ subjects and $m$ tasks.
Spearman's work started a revolution that gave us, among other things,  matrix and tensor factorizations, Principal Component Analysis (PCA), and Independent Component Analysis (ICA).
Simultaneously, {\em The Abilities of Man} also warned us about interpreting the factors of subject $i$ as innate rather than acquired abilities.
For instance, regarding Woolley and Fischer's observation that ``boys are {\em enormously superior} [to girls] at $[\dots]$ spatial relations'' (\textit{i.e.}, in how objects relate in space)~\cite{woolley1914mental}, Spearman warns that ``{\em evidence of this difference being really innate [rather than acquired] is still dubious}''.

Today, we can describe Spearman's warning as being about two competing causal hypotheses that describe link formation between young children and their abilities. 
A {\em path-dependent hypothesis} where past links influence future links~\cite{liben2007link} and an {\em innate factors} hypothesis where link formation is just a manifestation of latent innate factors~\cite{wang2020causal}.
In Woolley and Fischer's experiments, both hypotheses are able to describe the data: Either boys are innately better than girls at spatial reasoning (innate factors hypothesis), or boys in 1914 just happened to have had more playtime with spatial tasks than girls, with each task further improving their skills (path-dependent hypothesis).
If we were to hide the performance of girl $i$ on task $j$ (\textit{i.e.}, hide $A_{ij}$), both hypotheses could give equally accurate predictive models for the missing association $A_{ij}$.
Observing boys and girls perform these tasks over time does not disambiguate these two hypotheses, since inner factors can also evolve over time.

Under a causal task, however, incorrect innate factor assumptions may lead to incorrect predictions about the effect of interventions.
Encouraging girls to perform spatial tasks in playtime (an intervention) could improve their spatial reasoning skills under the path-dependent hypothesis but not under an innate factors hypothesis.
Similarly, performing an intervention to probe into whether a consumer will buy a computer keyboard (through an ad or a top search ranking)~\cite{joachims2021recommendations, wang2020causal}, who then buys it, will lead to assuming that (a) maybe the consumer no longer needs (another) keyboard under a path-dependent model; or (b) the consumer will keep buying more keyboards (if we keep showing more ads and top search results)~\cite{ji2022recommender} under an innate factor hypothesis.
Current causal models for link prediction operate under the assumption of innate factors~\cite{radhakrishnan2022simple,wang2020causal}. Path-dependent causal models are able to react to the evolution of the graph structure. Part of the reason for the absence of path-dependent models in the literature are the relational cross-dependencies of path dependency that are hard to address with existing causal methods.

The key insight of our work is observing that graph symmetries can encode the relationship between an evolving graph structure and a link formation process that reacts to it. To illustrate this, let us look at the social network example in \Cref{fig:intro}. Consider the graph formed by the observed friendships between users (black edges). We then intervene by recommending a friendship between users 1 and 8. Next, we observe a new friendship (link) appearing (green edge). Finally, we wonder what would have happened had we recommended a friendship between users 3 and 9 instead (counterfactual query). From the figure, we can see that both pairs of users are symmetric (isomorphic), \textit{i.e.}, indistinguishable if we do not consider their identifiers (cf.\ \Cref{sec:background}). Thus, answering the counterfactual query about 3 and 9 with the same outcome observed in 1 and 8 seems intuitively correct. However, without causal modeling assumptions it is not possible to answer counterfactual queries in general~\cite{Bar+2020}. Hence, our work establishes sufficient causal modeling assumptions for this intuition to hold in practice, \textit{i.e.}, for graph symmetries to govern the causal link formation process of a graph. Note how previous interventions in the graph formation process will be reflected in the structural symmetries, which itself governs the link formation process. As such, this can be a path-dependent model. Finally, we generalize this setting to a supervised learning task, where graph embedding models are trained to predict intervention outcomes from the observed graph structure. We show that graph embeddings capturing such symmetries, \textit{i.e.}, assigning the same embeddings to symmetric node pairs, are the unbiased estimators capturing the underlying causal mechanisms of the task. Next, we formalize these notions and detail our contributions.
\begin{figure}[ht]
\begin{center}
\includegraphics[scale=0.15]{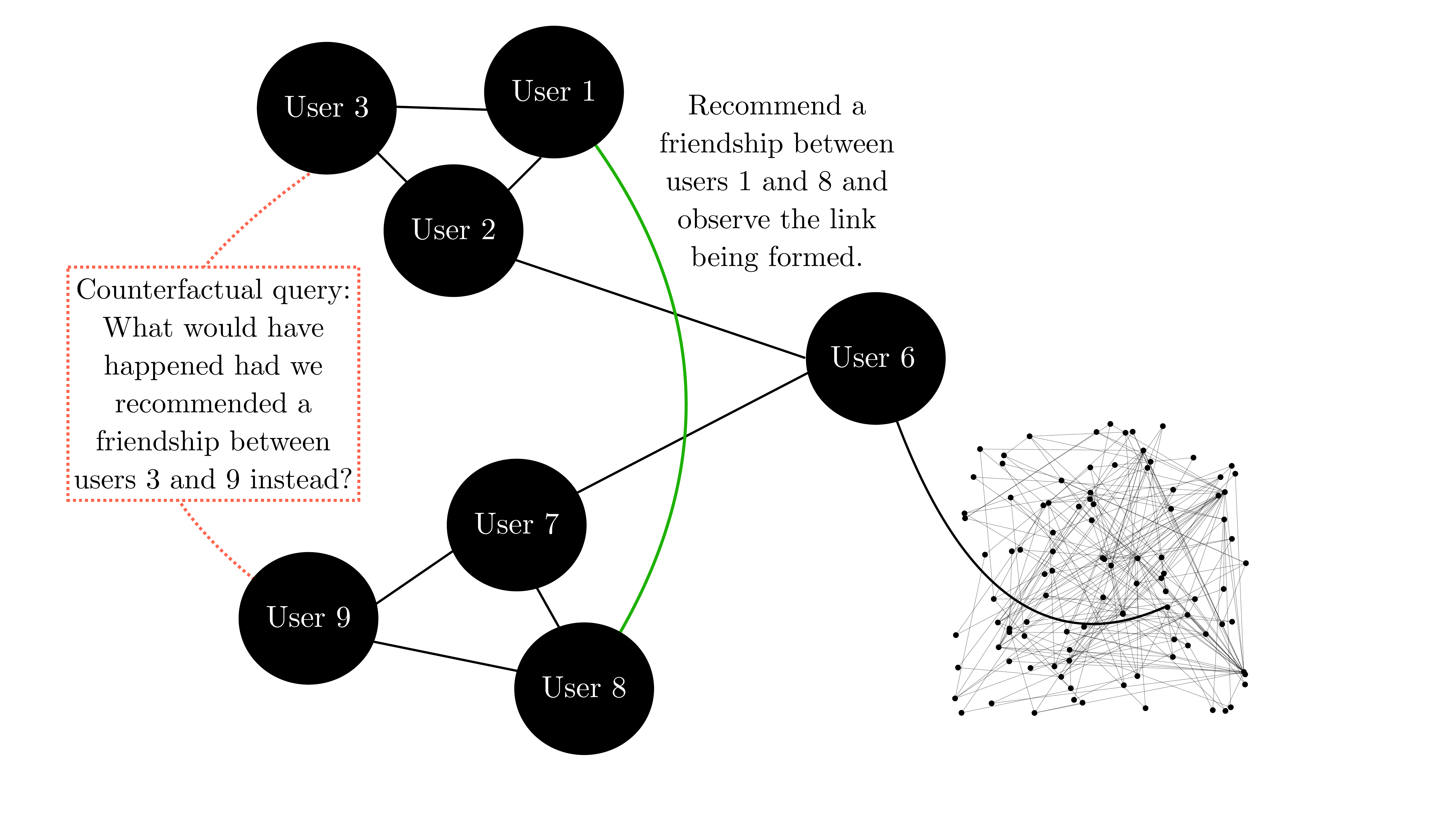}
\end{center}
\caption{\label{fig:intro} We illustrate the role of symmetry in a social network. Given that we observe the social graph formed by the black edges, we intervene and recommend a friendship between users 1 and 8. After observing that a new link is created between them, we ask whether a link would have been created had we recommended a friendship between users 3 and 9 instead.}
\end{figure}
\\~\\
\paragraph{Task setup} We consider first observing at time $t_0$ a (possibly static and attributed) graph $\G{t_0}$. Next, at a later time $t_1$ we are able to intervene in a pair of nodes $\E{t_1}$ and observe whether the link appears in $\G{t_1}$. Finally, we are interested in answering counterfactual queries of the form: \emph{What would have happened had we intervened in this other pair of nodes?} We outline the needed assumptions on the causal generating process $\m$ of $\G{t_0}$ such that we can intervene on more pairs and train a graph embedding model to answer the counterfactual queries. The causal model considered is (possibly) path-dependent, \textit{i.e.}, the link formation in $\G{t_1}$ can depend on how $\G{t_0}$ was created, including previous interventions (before $t_0$) that we do not observe. Moreover, the graph embedding is trained in a supervised learning fashion using the observational data $\G{t_0}$ to predict the interventional data $\G{t_1}$, \textit{i.e.}, the targets are the links and non-links in which we intervened.
\\~\\
\paragraph{Contributions} Our contributions are centered around using invariances to both i) define a set of sufficient causal modeling assumptions for causal identification and to ii) define the needed graph embedding for unbiased estimation of causal links. Regarding identification, we develop the concept of causal lifting, an invariance property of causal models that is able to serve link prediction under interventions in both path-dependent and innate factor models. Causal lifting serves more than a causal link prediction tool, it is also an experimental design tool and identification strategy for invariant data.
The key insight of causal lifting is to identify causal quantities by assuming invariances in the causal mechanisms of the task, rather than relying on variable controls or covariate-based adjustments common in do-calculus and potential outcomes analyses. Further, on the estimation side, we show how structural pairwise embeddings, a type of graph embedding that incorporates the known causal invariances of the task, achieves lower bias and variance than node embedding methods. Finally, we validate our theoretical findings on four datasets under three different scenarios of causal link prediction tasks.

\subsection*{A Family of Causal Link Prediction Tasks}

In our task, we observe graph data at some (pre-trial) time $t_0$ from an unknown (causal) graph generation process $\GGP$ with potential path dependencies. 
As such, future links and probes might depend on previous states of the graph. We denote the observed graph at pre-trial time $t_0$ by $\G{t_0}$, with corresponding adjacency matrix $\a{t_0}{}$, node set $\Vt{t_0} \text{ of size }\n{t_0}:=\setsize{\Vt{t_0}}$ and edge set $\Et{t_0}$. Without loss of generality, we refer to graphs at any other time points of interest $t \in \N$ using the same notation ($\G{t}, \a{t}{}, \Vt{t}, \n{t}, \Et{t}$). Moreover, we use capital letters to denote the corresponding random variable of an observation, \textit{e.g.}, $\A{t_0}{}$ is the random variable of the observed adjacency $\a{t_0}{}$. Finally, unless otherwise stated, our results consider an adjacency $ \a{t_0}{} \in \domA^{\n{t_0} \times \n{t_0}}$ with arbitrary finite domain $\domA$. Thus, $\a{t_0}{}$ possibly contains information about time, node and edge features ---which implies that $\a{t_0}{}$ is not necessarily a matrix (it can be a tensor representing this heterogeneous node adjacency).

After observing the graph $\G{t_0}$, we are often interested in \emph{probing into a certain relation}.
More specifically, we probe into the relation of $\IJ \sim \pol$, where $p\big(\IJ;\G{t_0}\big)$ is a distribution over the node pairs in $\G{t_0}$.  We refer to $\pol$ as the probe policy, \textit{i.e.}, it (stochastically) defines which relation from $\G{t_0}$ we wish to probe into. For instance, in online social networks we can probe into a friendship between two users by making a friendship recommendation. Note that the act of probing induces an intervention in the graph formation process, \textit{i.e.}, despite of being friends two users might not add each other at that point of time without an explicit recommendation from the system. We will refer to interventions of this type, \textit{i.e.}, probing into a relation, simply by \textbf{probes}. 
Finally, note that a probe can be seen as an experimental trial as well, \textit{e.g.}, in the \textit{recommendations as treatments} framework~\cite{joachims2021recommendations} products are seen as treatments that the recommender system chooses or not to administer to each of its users.

We consider the act of probing and its outcome happening simultaneously at a posterior time $t_1 > t_0$. That is, the difference in time between an intervention and its outcome is small enough that variables associated with other pairs are not impacted by the intervention. As mentioned, the probe is an intervention in the graph formation process. Hence, we define as $\E{t_1}$ the random variable of the graph process we make interventions in. We can then define the random variable of the outcome of a probe in $\IJ \sim \pol$ as 

\begin{equation}\label{eq:int}
  \underbrace{\Y{IJ}}_{\substack{\text{Outcome of}\\\text{probe in} \text{ $\IJ.$}}}\!\!\!\!\!:= \A{t_1}{\E{t_1}}\!\!\underbrace{\Big( \E{t_1} = \IJ \Big)}_{\text{ Probe in $\IJ$.}} \mid  \G{t_0},
\end{equation}
where $t_1$ is both the time of the probe and when we see its effect (post-trial time), $\A{t_1}{}$ is the adjacency of the graph $\G{t_1}$ (after the probe), and $\E{t_1}$ is the probe random variable. In our work we will use the potential outcomes notation~\cite{rubin2005causal} and Pearl's Causal Hierarchy~\cite{elias2020PCH} framework to describe causal tasks.

In real-world systems, we can actively probe and observe an outcome for \Cref{eq:int}, \textit{e.g.}, make a recommendation and observe whether two users add each other in an online social network. However, we also would like to keep probes to a minimum. Either because they are expensive to perform or because they interfere with users' experience~\cite{ferrara2022link} or the system's normal operation.
Rather, we describe our task as the following \textbf{idealized learning task}: 
Perform one probe in $\IJ \sim \pol$, then predict what would have happened had we probed into the relation of a different pair $\UV \sim \pol, \UV \neq \IJ$. We refer to this task as \textbf{causal link prediction}, and more formally define it as estimating the counterfactual quantity
\begin{equation}\label{eq:cf}
\begin{split}
    P\Big( \underbrace{\A{t_1}{\E{t_1}}\big( \E{t_1} = \UV \big)}_{ \substack{ \text{ What would have } \\ \text{happened had we probed } \\ \text{ in $\UV$ instead?} } } \mid \A{t_1}{\E{t_1}}(\E{t_1} = \IJ) , \G{t_0} \Big) \equiv P\Big( \Y{UV} \mid \Y{IJ}\Big).
\end{split}
\end{equation}
Note that $\E{t_1}$ is an individual quantity, \textit{i.e.}, for an observation $\G{t_0}$ it can only take one value, thus \Cref{eq:cf} is, without further assumptions, a strict counterfactual query (see \cite{elias2020PCH} for a precise definition).

\Cref{eq:cf}'s query is of interest to a wide range of applications, for instance (we will present experiments in \Cref{sec:res}): (i) determining what would have been the outcome of recommending product $V$ to user $U$, given user $I$ was recommended product $J$ and bought ($\A{t_1}{\E{t_1}}(\E{t_1} = \IJ) = 1$) or didn't buy ($\A{t_1}{\E{t_1}}(\E{t_1} = \IJ)=0$) it; (ii) knowledge base completion, where we manually check the relation between a pair of entities $(I,J)$ in a knowledge base and then use this experiment to predict what would have happened had we manually checked the relation between entities $(U,V)$;
(iii) in refining estimations of the covariance matrix between random variables from experiments, \textit{i.e.}, if $\A{t_1}{\E{t_1}}(\E{t_1} = \IJ) \approx \text{cov}(I,J)$ is the empirical covariance between two random variables $I$ and $J$, and our policy $\pol$ decides to collect more field data to improve the covariance estimates between $I$ and $J$, \Cref{eq:cf} allows us to predict what would have been the improved covariance estimate of $\text{cov}(U,V)$ had the policy chosen $\UV$ as the pair of random variables. 

We highlight that our counterfactual query is with respect to the probing policy $\pol$, \textit{i.e.}, we require interventional data collected under $\pol$ and can only answer counterfactual queries about pairs sampled from $\pol$. A possible interesting extension of this work is showing how to merge data from different policies to answer queries about one of them. Finally, we note that counterfactual questions such as ``what would have happened had we not intervened?'' are also out of the scope of this work. These questions involve a different causal quantity ($\A{t_1}{\E{t_1}} \mid \A{t_1}{\E{t_1}}(\E{t_1} = \IJ)$) than the one in \Cref{eq:cf} and would need either further causal assumptions or different data. Moreover, we can see this query as one about a different probing policy, \textit{i.e.}, where edges would have been generated according to the graph's natural evolution process.
\\~\\
\paragraph{Learning from multiple probes}
For ease of exposition, we have so far referred to an idealized task where we learn from a single probe $\Y{IJ}$ at time $t_1$ ---learning predictions from a single intervention would result in high variance estimates. To consider a more practical solution, we first define the random variable of the outcome of a probe at time $t_M \geq t_1$ in $\IJt{M}\sim \pol$ as

\begin{equation}\label{eq:ytm}
\Yt{I}{J}{M} := \A{t_M}{ \E{t_M} } \Big( \E{t_M} = \IJt{M} \Big) \mid  \Yt{I}{J}{1}, \ldots, \Yt{I}{J}{M-1}, \G{t_0},
\end{equation}
where each probe outcome $ \Yt{I}{J}{m}, 1 \leq m \leq M  $ is defined recursively in the same way. Now, we are ready to define the random variable of a sequence of $M$ probes in $\G{t_0}$ at times $t_1, \ldots, t_M$ as
\begin{equation}\label{eq:Y}
    \sY := \Big( \Yt{I}{J}{m} \Big)_{m=1}^M.
\end{equation}
Then, we are finally left with the generalized task of estimating the counterfactual quantity
\begin{equation}\label{eq:cf-multi}
    P\Big( \underbrace{\A{t_1}{\E{t_1}}\big( \E{t_1} = \UV \big)}_{ \substack{ \text{ What would have } \\ \text{happened had we probed } \\ \text{ in $\UV$ at time $t_1$ instead?} } } \mid \bigwedge_{m=1}^M \A{t_m}{ \E{t_m} } \Big( \E{t_m} = \IJt{m} \Big), \G{t_0} \Big) \equiv P\Big( \Y{UV} \mid \sY \Big).
\end{equation}

\paragraph{Challenges} Estimating \Cref{eq:cf-multi} brings the following challenges:

\begin{enumerate}[leftmargin=*]
    \item The (causal) data generating process $\GGP$ of $\G{t_0}$ is often unknown. Thus, there might have been unknown probes prior to $t_0$ and the (causal) evolution of the graph may be path-dependent (\textit{e.g.}, new links may depend on existing ones~\cite{albert2002statistical,fabrikant2002heuristically,perc2014matthew,puffert2002path}).
    
    \item Due to spillover~\cite{bakshy2012social,christakis2007spread,eckles2017design,holland1981exponential,hong2015causality,manski2013identification,rosenbaum2007interference,ugander2013graph,yuan2021causal} and 
    carryover effects~\cite{kohavi2012trustworthy}, subsequent probes might be affected by previous ones. As such, multiple probes ($\sY$) cannot be treated as i.i.d.\ data.
    
    \item Observing the outcome of a probe in $\IJt{m}$ might change our belief of what would have been the outcome of a probe in $\UV$, \textit{i.e.}, probes ($\sY$) and our prediction $\UV$ cannot be treated as i.i.d.\ data.
    
\end{enumerate}

\paragraph{Outline of contributions} In what follows we outline this work's contributions to tackle the above challenges.

\begin{enumerate}[label=\Alph*., leftmargin=*]

    \item We define the general concept of causal lifting: A generalization of probabilistic lifting~\cite{poole2003first} to different layers of Pearl's Causal Hierarchy. These definitions are of independent interest, not tied to link prediction.
    
    
    \item {\em We define the causal link prediction learning task. We show that lifting in causal link prediction is essentially the structural task of finding symmetries in $\G{t_0}$, which is more akin to finding logical rules than the positional goal of finding nearby similar relations in $\G{t_0}$. We then show why node embeddings are generally undesirable for these tasks and argue that structural (permutation invariant or equivariant) pairwise embeddings ---a special type of joint node pair representation--- present low learning bias and capture the correct causal structure from the task, as opposed to existing node embedding methods (e.g., matrix factorization and GNNs).}
    
    
    \item Finally, we introduce three situations where our theoretical results can be leveraged: (i) to extrapolate knowledge in knowledge bases, (ii) to learn to improve covariance matrix estimations with partially collected data, and (iii) to make recommendations in recommender systems. Overall, results confirm that invariant pairwise embedding methods consistently outperform node embedding ones in causal link prediction tasks.
    
    \item \textit{Contribution in the \Appendix}: We present ---to the best of our knowledge--- the first universal family of Structural Causal Models (SCMs) for finite graphs (finite exchangeable two-dimensional arrays). 
    
\end{enumerate}

\section{Existing Link Prediction Literature}

Link prediction started with Spearman's 1927 work but has more recently gained traction in different applications, \textit{e.g.}, matrix denoising~\cite{candes2010matrix}, social networks~\cite{adamic2003friends}, recommender systems~\cite{mnih2008probabilistic} and many others~\cite{ghasemian2020stacking}. In each problem domain, link prediction has been either treated as an associational (layer 1 of Pearl's Causal Hierarchy) or as a causal task (layer 2 and 3 of Pearl's Causal Hierarchy). In what follows, we review each of these perspectives, highlighting relevant related works while contrasting with our approach.
\\~\\
\paragraph{Associational link prediction methods} Since the first half of the last century, link prediction as an associational task has been studied through what can be formally described as a self-supervised learning task, where a masking process (a.k.a., noise process) hides edges as nonedges (and in some applications also vice-versa). The goal is to predict the true edges and nonedges in the adjacency matrix. By defining the task as self-supervised learning, \Cref{sec:selfsuper} shows how existing methods either implicitly or explicitly assume the masking process is known in training.

To illustrate our self-supervised learning perspective, let us consider matrix factorization on the observed graph $\G{t_0}$. More specifically, we let our loss function be given by the mean squared error $\Vert \a{t_0}{} - U V^T \Vert_2^2 = \sum_{(i,j) \in V^2} (\a{t_0}{ij} - U_{i\cdot} V^T_{j\cdot})^2$ with $U \in \R^{n \times d}$ and $V \in \R^{n \times d} $ as our learned low-rank matrices. In this scenario, as further detailed in \Cref{sec:selfsuper}, we assume a uniform mask over $\G{t_0}$'s node pairs, while our link reconstruction model is given by a normal distribution around its latent factors' dot product, \textit{i.e.}, $\A{t_0}{\IJ} \mid \a{t_0}{-\IJ}  \sim \cN(U_{I\cdot} V^T_{J\cdot},1)$. Note that in this case, apart from the usual causal restrictions imposed by latent factor models, we also have an observational limitation: At test time we have to sample node pairs uniformly at random ---an unlikely setting in some applications.


State-of-the-art tensor factorization methods used in knowledge base completion~\cite{ren2020query2box,sun2019rotate,trouillon2017knowledge} force the embeddings to encode properties of logical rules, such as symmetry/antisymmetry and composition. More recently, Graph Neural Networks (GNNs)~\cite{Sca+2009} have emerged as an alternative way to perform link prediction.
Unlike matrix factorization, a GNN embedding is not sensitive to permutations of the node identifiers (e.g., user id in the system).
Matrix factorization and GNN node embeddings are the modern workhorses of link prediction.
These node embeddings encode associations between graph topology (edges of the graph) as well as node and edge attributes, but they are designed to handle observational instead of interventional data.
The framing of link prediction as a self-supervised learning task helps us understand why even state-of-the-art link prediction approaches are associational and unable to predict the query in \Cref{eq:cf-multi}.

There are alternative associational approaches, often referred to as \textit{model-based} methods, that fit a random graph model to $\G{t_0}$ in different ways, \textit{e.g.}, considering hierarchical structures \cite{clauset2008hierarchical} or stochastic block models~\cite{holland1983stochastic,holland1981exponential}. These models can output the probability of any link in the observed graph. Note, however, that these methods are by definition associational, \textit{i.e.}, they directly compute a (associational) probability distribution from which $\G{t_0}$ was generated. As such, they do not provide much flexibility for causal extensions outside from the process defined by the model.
\\~\\
\paragraph{Existing causal link prediction methods} 
{\em Mechanistic (network creation) methods:} Link prediction methods have also been derived from models of network formation~\cite{liben2007link}. For instance, the common neighbors mechanism~\cite{newman2001clustering} assumes links are more likely created by nodes that have many common neighbors. The vast majority of other mechanisms such as the Jaccard, the Adamic-Adar~\cite{adamic2003friends}, the Katz~\cite{katz1953new}, and the preferential attachment~\cite{newman2001clustering} indices are all variations of the common neighbors method. All of such methods assume a fixed network formation process and predict links based on it. In the context of this work such methods can be interpreted as defining the interventional policy $\pol$. The effect of interventions is only captured by these methods if the network formation process follows their assumptions. Otherwise, one needs to estimate a method's effects by extensively performing experiments or, as we propose here, learning to estimate the effects from a few experiments (probes) using the method as our policy $\pol$.
\\ 

\noindent
{\em Inner factors literature.} Recent works have proposed to merge matrix factorization and interventional data in recommender systems~\cite{wang2020causal,bonner2018causal,xu2021causal}. In all of such works, both the interventional policy and the outcomes of interventions are given by inner (latent) factor models. What is missing in these methods is a causal structure where the link creation reacts to the current state of the graph (path-dependent models). It is understandable that these existing works choose not to model such co-dependence since it is unclear how to encode it in a causal model and still have a practical method. Later in this work we show how, under mild assumptions, one can develop a practical causal predictor that is robust to this co-evolution. More recently, the authors of \cite{zhao2022learning} proposed to learn the essential factors determining the existence of links by asking the question ``would the link still exist had the graph structure been different?''. Note how this is essentially a different causal question than the one we approach in this work (see \Cref{eq:cf}).
\\ 

\noindent
{\em Recommendations as treatments literature.} Estimating the effect of recommendations, an application of \Cref{eq:int}, has recently gained traction in the literature denoted as {\em recommendations as treatments}~\cite{joachims2021recommendations}. Under this framework, we can see $\pol$ as a (probabilistic) recommendation algorithm and the expected value of \Cref{eq:int} as its effectiveness, \textit{i.e.}, a higher expected value corresponds to links with higher probability of recommendation being formed. Solutions to estimate such effectiveness often comes in two flavors: Online and off-line A/B testing. \textit{Online A/B testing}, or simply A/B testing, is the most widely used method to evaluate recommender systems in industry. In this context, A/B testing can be seen as Randomized Controlled Trials (RCTs), where one part of the population (graph) is probed according to $\pol$ while the another is not, \textit{i.e.}, it is the control group. Unfortunately, this method is often expensive and time consuming, while possibly exposing the population to unwanted risks. To avoid running into these problems, methods under the umbrella of \textit{off-line A/B testing} propose to evaluate $\pol$ based on past user experience data, \textit{i.e.}, without running new experiments. The essence of these methods lies in the use of inverse propensity scores~\cite{joachims2021recommendations}, which re-weights the importance of previous interventions according to the policy used at the time and a new evaluated policy. Although such methods do not explicitly encode causal modeling assumptions, they implicitly make strong causal assumptions: Both the graph evolution process and the effect of probing into it are time-homogeneous, \textit{i.e.}, the probability of a link being created, under a probe or not, is the same at any point in time. Furthermore, all previous probes are observed, \textit{i.e.}, we have all past interventional and observational data. In this work, we consider a more general scenario where we only have past observational data and no assumptions about time homogeneity in the past.
\\ 

\noindent
{\em Causal effect estimation on networks.} Another relevant set of works aims at estimating the difference in potential outcomes of two treatments given in a network ---also known as causal effect estimation. In this scenario, a given treatment might not just affect the treated individual, but also their neighbors. Recent work in \cite{song2022estimating} formulates the problem as a multi-task one, which is then solved by a GNN-based framework. The authors of \cite{guo2020learning} propose to use the network's structure to control confounding bias and learn individual treatment effects. More recently, the work of \cite{ma2022learning} focuses on hypergraphs, in order to represent group interactions, and learns how to control for confounders and model higher-order interactions to estimate treatments' effects. All of such works fundamentally differ from ours for a main reason: They do not intervene in the graph structure, but rather on the nodes of a given graph. To the best of our knowledge, \cite{sherman2020intervening} is the only method that investigates interventions on the graph structure, but, differently to our approach, it aims to study the effect of changes resulting from creating or damaging network ties, rather than probing into relations.

\section{Causal Lifting}\label{sec:lift}

We commence our discussion by presenting one of our key contributions: {\em Causal lifting}, a general concept of independent interest beyond link prediction. {\em Causal lifting} is a causal extension to the associational definition of lifting in probabilistic inference~\cite{van2015lifted,poole2003first}. Here, instead of defining symmetries in the associational distribution, we define them in different layers of Pearl's Causal Hierarchy (associational, interventional and counterfactual). Let us first recall the classical definition of lifting in associational distributions.

\begin{definition}[Associational lifting~\cite{van2015lifted,poole2003first}]\label{def:assoclift} Let $X$ be our random variable of interest and $\ecal{G}$ a group where $\cdot$ is the left action of $\ecal{G}$ onto $\supp(X)$ (i.e., the function $\cdot: \ecal{G} \times \supp(X) \to \supp(X)$ satisfies the following axioms: (i) $e \cdot x =x$, where $e \in \ecal{G}$ is the identity element; (ii) ${g \cdot (h \cdot x)=(g h)\cdot x}$, where $g,h \in \ecal{G}$).
We say that $\ecal{G}$ lifts the associational distribution of $X$ if, $\forall x \in \supp(X), \forall g \in \ecal{G},$
\begin{equation} P\Big( X = x \Big) = P\Big( X = g \cdot x \Big).
\end{equation}
\end{definition}
%
Now, since an intervention defines an interventional distribution (see \Cref{sec:background}), we can directly extend \Cref{def:assoclift} to interventional distributions in \Cref{def:intlift}.

\begin{definition}[Interventional lifting]\label{def:intlift} Consider $X$ and $\ecal{G}$ as in \Cref{def:assoclift}, while $Y$ is a target random variable of interest. We say that $\ecal{G}$ lifts the interventional distribution of $X$ and $Y$ if, $\forall x \in \supp(X), \forall g \in \ecal{G},$
\begin{equation} P\Big( Y (X =  x) \Big) =  P\Big( Y (X = g \cdot x ) \Big). \end{equation}
\end{definition}

Finally, following \Cref{def:intlift} it is also straightforward to extend lifting to counterfactual distributions as we do in \Cref{def:cflift} below.

\begin{definition}[Counterfactual lifting]\label{def:cflift} Consider $X,Y$ and $\ecal{G}$ as in \Cref{def:intlift}. We say that $\ecal{G}$ lifts the counterfactual distribution of $X$ and $Y$ if, $\forall x,x' \in \supp(X), \forall g \in \ecal{G},$
\begin{equation} P\Big( Y (X =  x) \mid X=x' \Big) =  P\Big( Y (X = g \cdot x )\mid X=x' \Big).
\end{equation}
\end{definition}

\Cref{def:assoclift} is used in probabilistic inference algorithms~\cite{van2015lifted,poole2003first} to avoid unnecessary computations: One can replace the need to estimate $\setsize{\ecal{G}}$ quantities in the marginal probability $ \sum_{g \in \ecal{G}} P\big( X = g \cdot x \big)$ by estimating a single quantity $P\big( X = x \big)$. While in probabilistic inference we are interested in efficiently computing associational distributions, in causal inference our main challenge is to \textit{identify a causal quantity}: 
Without causal lifting, our data may not be enough to answer the causal query.
Next, we show how interventional lifting can play a key role in identifying causal link prediction tasks.

\section{Interventional Lifting for Link Prediction}

Our goal in this section is to show how interventional lifting can serve as an identification strategy and experimental design tool for the causal link prediction task (\Cref{eq:cf-multi}). To provide the reader with the basic tools used in our solution, in \Cref{sec:universal} we introduce  a universal family of structural causal models (SCMs), and 
in \Cref{sec:our-lift} we describe the needed invariance assumptions in the SCM mechanisms that allow us to apply interventional lifting for link prediction. 
Then, \Cref{sec:single} presents our main result (\Cref{thm:main}) in the context of our idealized single probe task (\Cref{eq:cf}) with $\IJ$ and $\UV$ as isomorphic node pairs.
Finally, in \Cref{sec:learn,sec:beyondprobes} we show extra causal mechanism assumptions that are sufficient to learn the general task of \Cref{eq:cf-multi} with multiple probes.

\subsection{A universal family of causal models for graphs}
\label{sec:universal}
Our work considers an underlying causal model that generates edges (together with edge and node attributes) and nonedges sequentially, \textit{i.e.} at time step $t \in \N$ it decides whether an specific pair of nodes $\ij$ has an edge (and its value) or not (node attributes are recorded in the pairs $(i,i)$, $i \in \Vt{t_0}$). The causal mechanism deciding which pair of nodes will be assigned an edge or a nonedge at time $t$ is denoted by $\fe{t}$ while the mechanism deciding whether it is an edge vs.\ nonedge is given by $\fx{t}$. Note that such a model can generate attributed graphs by outputting the edge attribute with $\fx{t}$. Generally speaking, both mechanisms take as input the sequence of all previously generated edges and nonedges, \textit{i.e.}, a path-dependent causal model execution. At observation time ($t_0$) the (observed) node identifiers are uniformly permuted from the hidden true identifiers and, if the graph is undirected, $\a{t_0}{}$ is symmetrized. In \Cref{fig:scm-ex} we can see a running example of such an SCM generating a graph of four nodes. 
\\~\\
\paragraph{Path-dependency} Note that since the mechanisms $\fe{t}, \fx{t}$ generating (non)edges at any time $t$ take as input all previously generated variables, a causal model $\m$ in this class can be path-dependent. We say that $\m$ \emph{can be} path-dependent since we are not specifying its mechanisms, thus it might be that $\fe{t}, \fx{t}$ ignore the input and use, \textit{e.g.}, a fixed set of inner factors. This set of inner factors might also be dependent on $t$, which implies that $\M$ contains both path-dependent and (temporal) inner factor models. As we will see in \Cref{sec:our-lift}, our mechanism assumptions are not implying independence of previously generated variables, thus the class of causal models we finally consider remains able to represent both path-dependent and (temporal) inner factor models.
\\~\\
The set of SCMs $\M$ denotes the class of all SCMs taking this form, with any mechanisms ($\fx{t}, \fe{t}$) and exogenous variable distributions. \Cref{def:g-scm,def:g-obs} in \Cref{sec:scm}  introduce a formal and complete description of the above family of SCMs. 
We now show that $\M$ is a universal family of exchangeable graph models, that is, any exchangeable distribution over (countable) graphs can be obtained by some SCM $\m \in \M$.
\begin{restatable}[Universality of our graph SCM]{thm}{zeromainthm}
\label{thm:main0}
Let $\M$ be the family of SCMs as defined in \Cref{def:g-scm,def:g-obs} in \Cref{sec:scm} and $\domA$ be the domain of the entries of the adjacency matrices of the graphs generated by it. Then,
\begin{itemize}[leftmargin=*]
    \item[i.] For every SCM $ \m \in \M$ at an arbitrary observation time $t_0 \geq 0$, $\m$ always generates observed graphs $\G{t_0}$ where $P(\A{t_0}{}=a)=P(\A{t_0}{}=a')$ for any two isomorphic graphs with adjacencies $a, a' \in \domA$;
    \item[ii.] For all finite (jointly) exchangeable graph distributions $P(\A{t_0}{})$, if $\domA$ is a countable set there exists an SCM $\m \in \M$ and an observation time $t_0 \geq 0$ that induces it.
\end{itemize}
\end{restatable}
\begin{figure}[ht]
\begin{center}
\tikzstyle{vertex}=[font=\scriptsize, draw, fill=white, circle, minimum size=20pt,inner sep=1pt, align=left]
\tikzstyle{selected vertex} = [draw, vertex, fill=green!24]
\tikzstyle{edge} = [draw,thick]
\tikzstyle{weight} = [font=\small, align=center, above, sloped, inner sep=2pt, pos=0.3]
\tikzstyle{selected edge} = [draw,line width=5pt,-,green!50]
\tikzstyle{ignored edge} = [draw,line width=5pt,-,black!20]
\usetikzlibrary{matrix,calc}
\pgfdeclarelayer{bottom}
\pgfdeclarelayer{middleb}
\pgfdeclarelayer{middlet}
\pgfdeclarelayer{top}
\pgfsetlayers{bottom,middleb,main,middlet,top}

\begin{tikzpicture}[scale=3.2]
\begin{pgfonlayer}{top}

\node[vertex](E1) at (0,0) {$\E{1}=$\\$\mathcolorbox{red!10}{(1,1)}$};
\node[vertex](E2) at (0.5,0) {$\E{2}=$\\$\mathcolorbox{orange!10}{(1,2)}$};
\node[vertex](E3) at (1,0) {$\E{3}=$\\$\mathcolorbox{green!10}{(3,1)}$};
\node[vertex](E4) at (1.5,0) {$\E{4}=$\\$\mathcolorbox{yellow!30}{(1,4)}$};
\node[vertex](E5) at (2,0) {$\E{5}=$\\$\mathcolorbox{babyblueeyes}{(3,2)}$};
\node[vertex](E6) at (2.5,0) {$\E{6}=$\\$\mathcolorbox{blue!10}{(2,4)}$};
\node[vertex](X1) at (0,-0.85) {$\X{1}=\mathcolorbox{red!10}{0}$};
\node[vertex](X2) at (0.5,-0.85) {$\X{2}=\mathcolorbox{orange!10}{0}$};
\node[vertex](X3) at (1,-0.85) {$\X{3}=\mathcolorbox{green!10}{1}$};
\node[vertex](X4) at (1.5,-0.85) {$\X{4}=\mathcolorbox{yellow!30}{1}$};
\node[vertex](X5) at (2,-0.85) {$\X{5}=\mathcolorbox{babyblueeyes}{1}$};
\node[vertex](X6) at (2.5,-0.85) {$\X{6}=\mathcolorbox{blue!10}{1}$};

\begin{pgfonlayer}{bottom}
\path[edge] (E1) edge [bend left=0] (E2);
\path[edge] (E1) edge [bend left=30] (E3);
\path[edge] (E1) edge [bend left=30] (E4);
\path[edge] (E1) edge [bend left=30] (E5);
\path[edge] (E1) edge [bend left=30] (E6);

\path[edge] (E3) edge [bend left=30] (E5);
\path[edge] (E3) edge [bend left=30] (E6);

\path[edge] (E2) -> (E3);
\path[edge] (E2) edge [bend left=30] (E4);
\path[edge] (E2) edge [bend left=30] (E5);
\path[edge] (E2) edge [bend left=30] (E6);

\path[edge] (E3) -> (E4);
\path[edge] (E3) edge [bend left=30] (E5);
\path[edge] (E3) edge [bend left=30] (E6);

\path[edge] (E4) -> (E5);
\path[edge] (E4) edge [bend left=30] (E6);

\path[edge] (E5) -> (E6);

\path[edge] (X1) -> (X2);
\path[edge] (X1) edge [bend right=30] (X3);
\path[edge] (X1) edge [bend right=30] (X4);
\path[edge] (X1) edge [bend right=30] (X5);
\path[edge] (X1) edge [bend right=30] (X6);

\path[edge] (X2) -> (X3);
\path[edge] (X2) edge [bend right=30] (X4);
\path[edge] (X2) edge [bend right=30] (X5);
\path[edge] (X2) edge [bend right=30] (X6);

\path[edge] (X3) -> (X4);
\path[edge] (X3) edge [bend right=30] (X5);
\path[edge] (X3) edge [bend right=30] (X6);

\path[edge] (X4) -> (X5);
\path[edge] (X4) edge [bend right=30] (X6);

\path[edge] (X5) -> (X6);

\path[edge] (E1) -> (X1);
\path[edge] (E1) -> (X2);
\path[edge] (E1) -> (X3);
\path[edge] (E1) -> (X4);
\path[edge] (E1) -> (X5);
\path[edge] (E1) -> (X6);

\path[edge] (E2) -> (X2);
\path[edge] (E2) -> (X3);
\path[edge] (E2) -> (X4);
\path[edge] (E2) -> (X5);
\path[edge] (E2) -> (X6);

\path[edge] (E3) -> (X3);
\path[edge] (E3) -> (X4);
\path[edge] (E3) -> (X5);
\path[edge] (E3) -> (X6);

\path[edge] (E4) -> (X4);
\path[edge] (E4) -> (X5);
\path[edge] (E4) -> (X6);

\path[edge] (E5) -> (X5);
\path[edge] (E5) -> (X6);

\path[edge] (E6) -> (X6);

\path[edge] (X1) -> (E2);
\path[edge] (X1) -> (E3);
\path[edge] (X1) -> (E4);
\path[edge] (X1) -> (E5);
\path[edge] (X1) -> (E6);

\path[edge] (X2) -> (E3);
\path[edge] (X2) -> (E4);
\path[edge] (X2) -> (E5);
\path[edge] (X2) -> (E6);

\path[edge] (X3) -> (E4);
\path[edge] (X3) -> (E5);
\path[edge] (X3) -> (E6);

\path[edge] (X4) -> (E5);
\path[edge] (X4) -> (E6);

\path[edge] (X5) -> (E6);

\end{pgfonlayer}

\node[vertex, fill=black!10, font=\normalsize](1) at (3,0.19) {$1$};
\node[vertex, fill=black!10, font=\normalsize](2) at (3.33,0.19) {$2$};
\node[vertex, fill=black!10, font=\normalsize](3) at (3,-0.14) {$3$};
\node[vertex, fill=black!10, font=\normalsize](4) at (3.33,-0.14) {$4$};
\node (rect) at (3.16,-0.27) {$\G{6}$}; 
\path[edge,-] (1) -- (2);
\path[edge,-] (1) -- (3);
\path[edge,-] (2) -- (4);
\path[edge,-] (3) -- (4);

\matrix (m) [matrix of nodes,
             nodes={draw,minimum size=1em, outer sep=0pt,inner sep=0,line width=0.5pt},
             nodes in empty cells,column sep=0.0em, row sep=0.0em, fill=white,
              ] at (3.18, -0.77)
{
    $\mathcolorbox{black!10}{0}$ & $\mathcolorbox{black!10}{1}$ & $\mathcolorbox{black!10}{1}$ & $\mathcolorbox{black!10}{0}$  \\
    $\mathcolorbox{blue!10}{1}$ & $\mathcolorbox{black!10}{0}$ & $\mathcolorbox{black!10}{0}$ & $\mathcolorbox{black!10}{1}$ \\
    $\mathcolorbox{yellow!30}{1}$ & $\mathcolorbox{orange!10}{0}$ & $\mathcolorbox{red!10}{0}$ & $\mathcolorbox{black!10}{1}$ \\
    $\mathcolorbox{black!10}{0}$ & $\mathcolorbox{babyblueeyes}{1}$ & $\mathcolorbox{green!10}{1}$ & $\mathcolorbox{black!10}{0}$ \\
};

\node (rect) at (3.16,-1.25) {$\a{6}{}$}; 

\node[draw,align=left, dashed] (rect) at (3.16,0.47) {Possible hidden \\ causal dependencies  \Warning };

\node[draw,align=left](rect) at (0.33,0.7) {Observation time $t_0 = 6$ \\ 
                                        Permutation $\pi = ( 3,2,4,1) $\\
                                        Undirected $= \mathtt{True}$}; 
\end{pgfonlayer}

\begin{pgfonlayer}{bottom}
\node[vertex, fill=black!10, font=\normalsize](a) at (2.2,0.7) {$\a{6}{}$};
\path[edge] (rect) -- (a);
\path[edge] (E1) -- (a);
\path[edge] (E2) -- (a);
\path[edge] (E3) -- (a);
\path[edge] (E4) -- (a);
\path[edge] (E5) -- (a);
\path[edge] (E6) -- (a);
\path[edge] (X1) -- (a);
\path[edge] (X2) -- (a);
\path[edge] (X3) -- (a);
\path[edge] (X4) -- (a);
\path[edge] (X5) -- (a);
\path[edge] (X6) -- (a);
\end{pgfonlayer}

\end{tikzpicture}
\end{center}
\caption{ We show the execution of an SCM as described in \Cref{sec:scm} where, together with the observation parameters $t_0=6$ and $\pi=(3,2,4,1)$, generate the observed graph $\G{6}$ with adjacency matrix $\a{6}{}$. We color the entries in $\a{6}{}$ according to the corresponding variables in the unobserved SCM that generated them.  Gray nodes represent observed random variables, while white nodes represent unobserved ones. Note that for simplicity of exposition we omit the exogenous variables in the causal model.\label{fig:scm-ex}}
\end{figure}
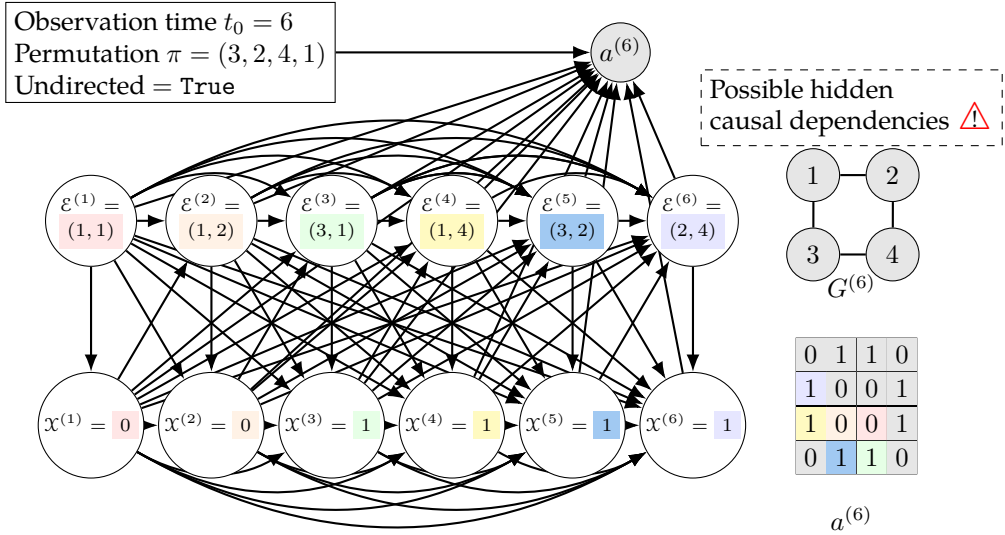

\subsection{Sufficient mechanism invariances for interventional lifting in link prediction}\label{sec:our-lift} 
After introducing a universal family of causal graph models $\M$ in \Cref{sec:universal}, we now restrict this family with four causal modeling assumptions that will allow us to apply interventional lifting in our causal link prediction task. We highlight how these are not causal independence (structural) assumptions as it is commonly assumed in causal identification procedures. Instead, we define four mechanism invariances present in the underlying SCM of the task. The mechanism invariance assumptions are a key feature of our work since, by avoiding causal independence assumptions, our models consider possible path dependencies. These invariances are presented informally in the main text; we refer readers to \Cref{sec:link-lift} for a corresponding formal mathematical description.

\begin{assumption}[Time gap ignorability (informal)]\label{ass:time-igno}
 We say that our SCM satisfies time gap ignorability if the mechanism $\fx{t_1}$ is invariant to the SCM intermediate states between the time the intervention probe is performed $t_0$ and the instant before we see its effect  in $t_1$.
\end{assumption}

\begin{assumption}[Time exchangeability (informal)]\label{ass:time-exc}
 We say that our SCM satisfies time exchangeability if the mechanism $\fx{t_1}$ is invariant to the order in which edges and nonedges have been generated.
\end{assumption}

\begin{assumption}[Non-link ignorability (informal)]\label{ass:link}
We say that our SCM satisfies non-link ignorability if the mechanism $\fx{t_1}$ is invariant to which pairs of nodes were generated as non-links or were not generated at all at time $t_0$.
\end{assumption}

\begin{assumption}[Identifier exchangeability (informal)]\label{ass:id-exc}
We say that our SCM satisfies identifier exchangeability if the mechanism $\fx{t_1}$ is invariant to permutations of the node identifiers.
\end{assumption}

We now discuss how the above mechanism invariance assumptions induce a simplified causal DAG.
First, note that $t_1 \geq t_0 + 1$, which means that the observation time ($t_0$) and the time we see the effect of the probe ($t_1$) are not necessarily consecutive (in the model execution time).
In theory, links and non-links generated between the observation of the graph and the probe could influence the outcome of the probe. 
Since, by having only access to $\G{t_0}$, we cannot account for the graph evolution in the interval $(t_0,t_1)$, we need time gap ignorability (\Cref{ass:time-igno}) ---which assumes that the outcome of the probe is invariant to whatever happened between $t_0$ and the moment before $t_1$. This assumption is pervasive in causal inference since it is not possible to identify causal queries without observing the true distribution of outcomes~\cite{brand200711}.

Time exchangeability (\Cref{ass:time-exc}) is the assumption that the order in which links and non-links have been created until time $t_0$ does not affect the probe. This is necessary because we only observe a static graph $\G{t_0}$ at time $t_0$. Note that if the order is an important aspect of the task, \Cref{ass:time-exc} would still hold if we represent $\G{t_0}$ as a temporal graph (with time stamps as edge attributes in the static graph $\G{t_0}$, see \cite{Gao+2022}).

\Cref{ass:link} is important since at $t_0$ we are not aware of whether an observed non-link is indeed a non-link generated by the causal model or a node pair not-yet-executed, hence we need the outcome of the probe to be invariant to this difference. 
By assuming non-link ignorability (\Cref{ass:link}), we do not need to distinguish not-yet-executed from non-links. For instance, in recommender systems \Cref{ass:link} makes the simplifying assumption that new purchases are causally influenced by past purchases, not by what one has been exposed but chose not to buy. 

It is important to note how \Cref{ass:id-exc} is not equivalent to $\m$ being an exchangeable SCM (\Cref{thm:main0} i.). The mechanism $\fx{t_1}$ can use the node identifiers as input and because they are shuffled by $\pi$ the observed graph is finite exchangeable regardless of $\fx{t_1}$ or any of its mechanisms. Thus, \Cref{ass:id-exc} guarantees that not only the observational distribution is finite exchangeable, but also that its graph generating process also is exchangeable to node ids (at time $t_1$). This assumption is a reasonable one, but {\em if we believe identifiers are relevant to the causal model mechanisms we should use them as node features, and then \Cref{ass:id-exc} still holds}.

Finally, we discuss the concept of i.i.d.\ exogenous variables between node pairs in an SCM $\m \in \M$. Independent and identical exogenous variables are central to identify our counterfactual query. The exogenous variables can be interpreted as the unknown context of a node pair. Then, the i.i.d.\ assumption between two node pairs holds if we believe they belong to different, non-interfering, but identical contexts. For instance, in a video streaming platform a pair containing a person and a movie and another isomorphic pair containing both a person and a movie from a different geographic location are likely to have both independent and identical contexts. Later, we will assume this for subsets of node pairs, \textit{i.e.}, we never require that this is true for every node pair in $\G{t_0}$.

\subsection{Causal link prediction with single probe lifting}
\label{sec:single}
Now we are ready to describe how interventional lifting can allow us to answer our counterfactual query of \Cref{eq:cf}.
We say that two node pairs $\ij, \uv$ from $\G{t_0}$ are isomorphic if there exists a permutation $\pi \in \aut{\G{t_0}}$ in the automorphism group of $\G{t_0}$ such that $(u,v)=\pi \cdot (i,j)$, \textit{i.e.}, $\ij$ and $\uv$ are structurally indistinguishable in $\G{t_0}$. 
Having this notion in mind, under certain conditions we will show we can obtain the interventional lifting (\Cref{def:intlift})
\begin{equation}\label{eq:link-lift}
\begin{split}
    P\Big(\A{t_1}{\E{t_1}} \Big( \E{t_1} = \IJ \Big) \mid  \G{t_0}\Big) = P\Big( \A{t_1}{\E{t_1}} \Big( \E{t_1} = \pi \cdot \IJ \Big) \mid  \G{t_0}\Big),
    \forall \pi \in \aut{\G{t_0}}.
\end{split}
\end{equation}
In words, the above equation states that node pairs isomorphic in $\G{t_0}$ (\textit{i.e.}, pairs indistinguishable without node ids) must have the same distribution of probe outcomes. Intuitively, this means that under certain invariance conditions on the causal mechanisms, the observed graph $\G{t_0}$ is sufficiently expressive of the causal link formation process. In fact, as we state in \Cref{thm:main}, these conditions are precisely the ones discussed in \Cref{sec:our-lift}. 
\Cref{thm:main} shows that if $\uv$ is in $\ij$'s orbit in $\G{t_0}$, under the causal mechanism invariance conditions from \Cref{ass:time-igno,ass:time-exc,ass:link,ass:id-exc} and identical exogenous variables, the probe outcome $\Y{ij}$ has the same distribution as the probe outcome $\Y{uv}$ (conditioned on $\G{t_0}$). That is, the invariances in causal mechanisms coupled with identical exogenous variables ensure interventional lifting, \textit{i.e.}, identical distributions in the orbit. The lifting is better observed in the equivalent causal DAG shown in \Cref{fig:gamma-scm}. We then leverage i.i.d.\ exogenous variables to ensure independence between probes in $\ij$ and in $\uv$ to then build an estimator in \Cref{col:estimator}.
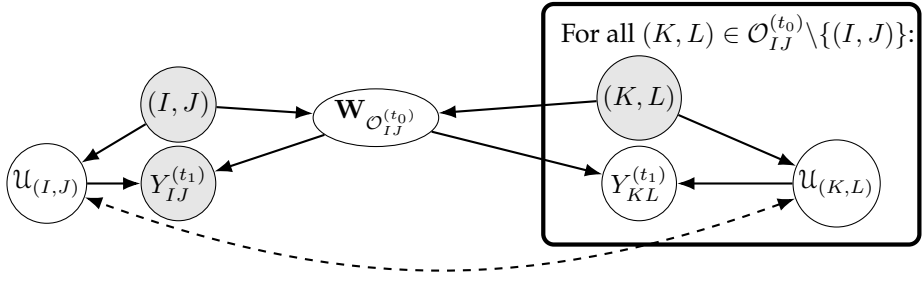
\begin{figure}[ht]
\begin{center}
\tikzstyle{vertex}=[font=\normalsize, draw, fill=white, circle, minimum size=20pt,inner sep=1pt, align=left]
\tikzstyle{selected vertex} = [draw, vertex, fill=green!24]
\tikzstyle{edge} = [draw,thick]
\tikzstyle{weight} = [font=\small, align=center, above, sloped, inner sep=2pt, pos=0.3]
\tikzstyle{selected edge} = [draw,line width=5pt,-,green!50]
\tikzstyle{ignored edge} = [draw,line width=5pt,-,black!20]
\pgfdeclarelayer{bottom}
\pgfdeclarelayer{middleb}
\pgfdeclarelayer{middlet}
\pgfdeclarelayer{top}
\pgfsetlayers{bottom,middleb,main,middlet,top}
\begin{tikzpicture}[scale=1.75]
\node[vertex,fill=black!10](kl) at (0,-0.35) {$(K,L)$};
\node[vertex,fill=black!10](ij) at (-3.5,-0.4) {$\IJ$};
\node[draw, minimum height=3.2cm,minimum width=5cm, ultra thick, rounded corners] at (0.7,-0.55) {};
\node[] at (.75,0.15) {For all $(K,L) \in \ORB{IJ} \backslash \{\IJ\} $: };
\node[vertex,ellipse, minimum size=20pt](W) at (-2,-0.5) {$\bW_{\ORB{IJ}}$};
\node[vertex](U) at (1.5,-1.) {$\ecal{U}_{(K,L)}$};
\node[vertex](Uij) at (-4.5,-1.) {$\ecal{U}_{\IJ}$};
\node[vertex](A) at (0,-1.) {$\Y{KL}$};
\node[vertex,fill=black!10](Aij) at (-3.5,-1.) {$\Y{IJ}$};
\path[edge] (kl) -> (W);
\path[edge] (ij) -> (W);
\path[edge] (ij) -> (Uij);
\path[edge] (kl) -> (U);
\path[edge] (U) -> (A);
\path[edge] (Uij) -> (Aij);
\path[edge] (W) -> (A);
\path[edge] (W) -> (Aij);
\path[bidirected] (U) edge[bend left=20, thick]  (Uij);
\end{tikzpicture}






\caption{(\Cref{thm:main}(i)) Causal DAG of an equivalent data generating process of a probe in $(i,j)$ (\textit{left}) and in its orbit (\textit{right}). As usual, we represent observed and unobserved variables with gray and white nodes respectively.\label{fig:gamma-scm}}
\end{center}
\end{figure}

\begin{restatable}[Invariances for interventional lifting in link prediction]{thm}{mainthm}
\label{thm:main}
Let $\M$ be the family of SCMs as defined in \Cref{def:g-scm,def:g-obs} in \Cref{sec:scm} and $\domA$ be the domain of the entries of the adjacency matrices of the graphs generated by it. Then,
\begin{itemize}[leftmargin=*]
    \item[i.] if $\m \in \M$ has the mechanism invariances described in \Cref{ass:time-igno,ass:time-exc,ass:link,ass:id-exc} and exogenous variable independence at time $t_0+1$, \textit{i.e.}, $\Ux{t_0+1} \indep (\Ux{t})_{t=1}^{t_0} \mid \IJ $ in \Cref{def:g-scm}.
    Then, the effect of an intervention in $\m$ can be equivalently described by \Cref{fig:gamma-scm}'s causal DAG, where $\IJ \sim \pol$ is the node pair we intervene in, $\ORB{IJ} := \{ \pi \cdot \IJ \colon \pi \in \aut{\G{t_0}} \} $ is the set of all structurally indistinguishable pairs to $\IJ$ in $\G{t_0}$, and $\bW_{\ORB{IJ}}$ is a latent variable tied to $\ORB{IJ}$ (the orbit of $\IJ$ in $\G{t_0}$).
    \item[ii.] Under the conditions in (i) and assuming the extra symmetry $P(\U_{\IJ})=P(\U_{\pi \cdot \IJ})$ in \Cref{fig:gamma-scm}'s DAG, we have the following interventional lifting result (\Cref{def:intlift}), where $\forall \pi \in \aut{\G{t_0}}$
    \begin{align*}
    P\left(\Y{IJ}\right) =  P\Big(\A{t_1}{\E{t_1}} \Big( \E{t_1} = \IJ \Big) \mid  \G{t_0}\Big) &= P\Big( \A{t_1}{\E{t_1}} \Big( \E{t_1} = \pi \cdot \IJ \Big) \mid  \G{t_0}\Big) \\&= P\left(\Y{\pi \cdot \IJ}\right).
    \end{align*}
\end{itemize}
\end{restatable}

In summary, \Cref{thm:main} outlines a procedure to derive causal modeling conditions in which a probe in $\ij$ can be used as an unbiased estimate of \Cref{eq:cf} on a subset of pairs: $\ij$'s orbit. 
This is an interventional lifting for link prediction (see \Cref{eq:link-lift}). Then, finally, if the exogenous variables of probes in the same orbit $\U_{\ij}$ and $\U_{\uv}$ are i.i.d., we can introduce the following result.
\begin{corollary}[Symmetries for interventional learning]\label{col:estimator}
Under conditions of \Cref{thm:main}, let $(i,j)\in \Vt{t_0} \times \Vt{t_0}$ be an arbitrary node pair and $\uv \in \ORB{ij}$ (i.e., $\uv$ is structurally indistinguishable from $\ij$ in $\G{t_0}$).
Further, assume $\U_{\ij}$ and $\U_{\uv}$ are i.i.d.. Then, by \Cref{thm:main}(ii)'s DAG in \Cref{fig:gamma-scm}
\begin{equation}\label{eq:learning_inter}
P(\Y{uv} \mid \Y{ij}) = \int_{\bW_{\ORB{ij}}} P(\Y{uv} \mid \bW_{\ORB{ij}}) P(\bW_{\ORB{ij}} \mid \Y{ij}) \,d\bW_{\ORB{ij}}.
\end{equation}
If $P(\bW_{\ORB{ij}} \mid \Y{ij})$ has a single mode and low variance, the above equation can be approximated by
$$
P(\Y{uv} \mid \Y{ij}) \approx P(\Y{uv} \mid \bW^\star_{\ORB{ij}}),
$$
where $\bW^\star_{\ORB{ij}} = \argmax_{\bW_{\ORB{ij}}} P(\bW_{\ORB{ij}} \mid \Y{ij})$ is a Maximum A Posteriori (MAP) estimate.
\end{corollary}

\Cref{thm:main} together with \Cref{col:estimator} provide an identification and estimation procedure for causal link prediction. In \Cref{fig:gamma-scm} we depict how the hidden parameters associated with the orbit $\ORB{ij}$ of a node pair $\ij$ in $\G{t_0}$ d-separates the probe from the generating process of $\G{t_0}$ . As such, we can see \Cref{eq:learning_inter} as a sequence of {\em abduction, action, and prediction} (Theorem 7.1.7 in Pearl~\cite{pearl2009causality}), using structural information to compose the admissible set of variables in the task. Note, however, that such a procedure learns from a single probe in an isomorphic node pair. Next, we will extend our causal assumptions to allow us to learn from multiple probes in different orbits. 

\subsection{Lifting in causal link prediction with multiple probes}\label{sec:learn} 
%
\Cref{col:estimator} shows how a probe in $\ij$ can serve as unbiased estimate of the counterfactual query (\Cref{eq:cf}) on a pair $\uv$ in $\ij$'s orbit. 
A single training example (probe), however, is not an ideal learning setting: We have the outcome of a probe in a single pair and we can answer counterfactual queries only about its orbit. Now, we show how to extend our causal assumptions to learn to answer counterfactual queries about \emph{any pair} using \emph{multiple probes} as training data (\Cref{eq:cf-multi}).

We start with the condition that allows us to use multiple probes to estimate the expected value in, for instance, \Cref{col:estimator}: Non-interfering probes (c.f.\ \Cref{def:y-int}, formalized in \Cref{sec:link-lift}). In summary, non-interference in probes means that the outcome of other probes does not interfere in each others' results. This assumption is widely used in causal inference for graph data, see for instance the work of Eckles et al.~\cite{eckles2017design}. In \Cref{def:y-int}, however, we provide ---to the best of our knowledge--- the first formalization of non-interference with respect to the the underlying causal mechanisms of the task. Note that non-interfering probes is an invariance condition depending on both the causal mechanisms and on the pairs $\IJt{m}$'s we probe at each time $t_m$. Regarding the mechanisms, it needs to be taken as a causal modeling assumption. As of how to choose probes that don't interfere with each other, we can treat it as an experimental design choice. For instance, non-interference tends to be satisfied with higher probability if the probes are performed in  pairs far away in the graph~\cite{leung2022causal}. Finally, we refer the reader to \cite{eckles2017design} for a thorough analysis of experimental design choices that result in non-interfering probes.

\begin{assumption}[Non-interfering probes (informal)]\label{def:y-int}
We say that a sequence of probes in $M$ pairs $\big(\IJt{m}\big)_{m=1}^M$ is non-interfering if every mechanism $\fx{t_m} \colon 1 < m \leq M $ is invariant to probes performed between $t_1$ and $t_m-1$.
\end{assumption}

Now, we can extend \Cref{col:estimator} to use $\sY$, where under non-interference, multiple probes in the orbit of $\ij$ can be used to estimate the counterfactual query about any pair $\uv$ also in $\ij$'s orbit.

\begin{corollary}[Symmetries for interventional learning with multiple non-interfering interventions]\label{col:multi-estimator}
Under the same conditions as \Cref{thm:main}(ii), let $\sY := ( \Yt{i}{j}{m} )_{m=1}^M $ be a sequence of outcomes of probes in node pairs in the same orbit ($\ORB{i^{(1)}j^{(1)}}=\ORB{i^{(2)}j^{(2)}}=\ldots=\ORB{i^{(M)}j^{(M)}}$). Then, for $\uv \in \ORB{i^{(1)}j^{(1)}}$, if the exogenous variables $\{ \ecal{U}_{\ijt{m}} \colon m \in [M]\} \cup \{\ecal{U}_{\uv}\}$ are i.i.d.\ and $\sY$ is a sequence of non-interfering interventions, then
\begin{equation}\label{eq:learning_inter_multi}
P(\Y{uv} \mid \sY ) = \int_{\bW_{\ORB{ij}}} P(\Y{uv} \mid \bW_{\ORB{ij}}) P(\bW_{\ORB{ij}} \mid \sY ) \,d\bW_{\ORB{ij}}.
\end{equation}
If $\bW_{\ORB{ij}} \mid \sY$ has low variance, the above equation can be well-approximated by
$$
P(\Y{uv} \mid \sY) \approx P(\Y{uv} \mid \bW^\star_{\ORB{ij}}),
$$
where $\bW^\star_{\ORB{ij}} = \argmax_{\bW_{\ORB{ij}}} P(\bW_{\ORB{ij}}) \prod_{m=1}^M P( \Y{ i^{(m)}j^{(m)} } \mid \bW_{\ORB{ij}}) $ is a MAP estimate.
\end{corollary}
Although we can now use multiple probes to estimate our counterfactual query, we are still restricted to probing in the orbit $\ORB{i^{(1)}j^{(1)}}$ of the first queried pair $(i^{(1)},j^{(1)})$. Such a setting tends to be impractical since i) both random and real-world graphs tend to be nearly asymmetric (most orbits are of size one with high probability)~\cite{erdHos1963asymmetric,zhou2022ood} and ii) computing the orbit of a pair of nodes is as hard as solving the graph isomorphism problem~\cite{flum2006parameterized}. Therefore, we finally consider a more practical solution, where we sample pairs according to an arbitrary policy $\pol$, probe into their relationships and \emph{learn a model} capable of predicting our counterfactual query (\Cref{eq:cf-multi}) for any $\UV \sim \pol$.

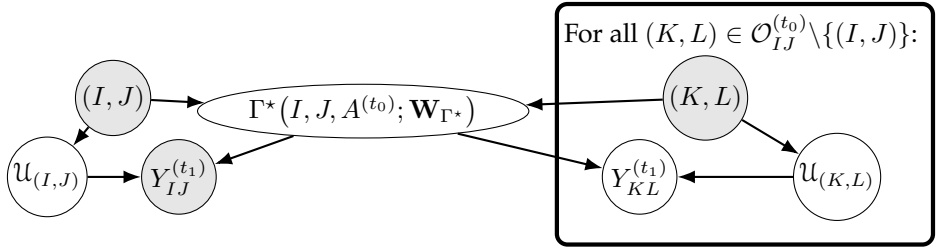
\begin{figure}[ht]
\begin{center}
\tikzstyle{vertex}=[font=\normalsize, draw, fill=white, circle, minimum size=20pt,inner sep=1pt, align=left]
\tikzstyle{selected vertex} = [draw, vertex, fill=green!24]
\tikzstyle{edge} = [draw,thick]
\tikzstyle{weight} = [font=\small, align=center, above, sloped, inner sep=2pt, pos=0.3]
\tikzstyle{selected edge} = [draw,line width=5pt,-,green!50]
\tikzstyle{ignored edge} = [draw,line width=5pt,-,black!20]
\pgfdeclarelayer{bottom}
\pgfdeclarelayer{middleb}
\pgfdeclarelayer{middlet}
\pgfdeclarelayer{top}
\pgfsetlayers{bottom,middleb,main,middlet,top}
\begin{tikzpicture}[scale=1.75]
\node[vertex,ellipse, minimum size=20pt, fill=white](Gamma) at (-2.1,-0.5) {$\Gammamostexp\big( I,J,\A{t_0}{};\bW_{\Gamma^\star} \big)$};
\node[vertex,fill=black!10](kl) at (0.5,-0.4) {$(K,L)$};
\node[vertex,fill=black!10](ij) at (-4,-0.4) {$\IJ$};
\node[draw, minimum height=3.2cm,minimum width=5cm, ultra thick, rounded corners] at (0.8,-0.6) {};
\node[] at (.77,0.1) {For all $(K,L) \in \ORB{IJ} \backslash \{\IJ\} $: };
\node[vertex](U) at (1.5,-1.) {$\ecal{U}_{(K,L)}$};
\node[vertex](Uij) at (-4.5,-1.) {$\ecal{U}_{\IJ}$};
\node[vertex](A) at (0,-1.) {$\Y{KL}$};
\node[vertex,fill=black!10](Aij) at (-3.5,-1.) {$\Y{IJ}$};
\path[edge] (kl) -> (Gamma);
\path[edge] (ij) -> (Gamma);
\path[edge] (ij) -> (Uij);
\path[edge] (kl) -> (U);
\path[edge] (U) -> (A);
\path[edge] (Uij) -> (Aij);
\path[edge] (Gamma) -> (A);
\path[edge] (Gamma) -> (Aij);
\end{tikzpicture}






\caption{The causal diagram of the data generating process of a probe in $(i,j)$ (\textit{left}) and in its orbit (\textit{right}) that allows us to learn from interventions in a supervised learning fashion (see \Cref{eq:obj} and \Cref{prop:2}). \label{fig:w-scm}}
\end{center}
\end{figure}

Such a general solution relies on two extra assumptions: 
i) exogenous variables from the pairs sampled during training, $\{\ecal{U}_{\ijt{m}}\}_{m=1}^M$, and the tested pair, $\ecal{U}_{\uv}$, are i.i.d.\ and ii) the parameters of the probes' distributions are shared across all orbits. In particular, for ii) we consider a shared set of parameters $\bW_{\Gamma^\star}$, which parameterizes a representation function $\Gamma^\star$. Note that, in order to express all possible probe distributions, the set of parameters $\bW_{\Gamma^\star}$ must assign the same value to $\ij$ and $\uv$ if and only if $\uv$ is in $\ij$'s orbit. We refer to such as a most-expressive representation, defined next.

\begin{definition}[Most-expressive pairwise representation] \label{def:joint} \label{def:inv-pair-emb}
A most-expressive pairwise representation is given by the functional $\Gammastar \colon \Vt{t_0} \times \Vt{t_0} \times \domA^{\n{t_0} \times \n{t_0}} \times \R^p \to \R^d$, parameterized by $\bW_{\Gammastar} \in \R^p$, $p\geq 1$, where
\begin{itemize}[leftmargin=*]
    \item[i.] for any parameter $\bW_{\Gammastar}$, any input graph $\a{t_0}{} \in \domA^{\n{t_0} \times \n{t_0}}$ and any pair $(i,j) \in \Vt{t_0} \times \Vt{t_0}$, we have that $\Gammastar(i,j,\A{t_0}{};\bW_{\Gammastar}) = \Gammastar(k,l,\A{t_0}{};\bW_{\Gammastar}), \forall \; (k,l) \in \ORB{i,j}$, and
    \item[ii.] $\exists \; \bW'_{\Gammastar} \in \R^p$ such that for any input graph $\a{t_0}{} \in \domA^{\n{t_0} \times \n{t_0}}$ and any pair $(i,j) \in \Vt{t_0} \times \Vt{t_0}$, we have that $\Gammastar(i,j,\A{t_0}{};\bW'_{\Gammastar}) \neq \Gammastar(r,s,\A{t_0}{};\bW'_{\Gammastar}), \forall \; (r,s) \in (\Vt{t_0} \times \Vt{t_0}) \setminus \ORB{i,j}$.
\end{itemize}
\end{definition}
 Leveraging \Cref{def:joint}, we depict this modification in the SCM in \Cref{fig:w-scm}. Now, we are ready to state our main learning result and summarize its practical implications in experimental design and assumptions.

\begin{proposition}\label{prop:2}
Under conditions of \Cref{thm:main}(ii), if $\sY$ is a sequence of outcomes of non-interfering interventions (\Cref{def:y-int}), the exogenous variables $\{\ecal{U}_{\uv}\} \cup \{\ecal{U}_{\ijt{m}} \colon 1 \leq m \leq M\}$ are i.i.d., and $\Gammastar(\cdot,\cdot,\cdot;\bW_{\Gamma^\star}) \colon \Vt{t_0} \times \Vt{t_0} \times \domA^{\n{t_0} \times \n{t_0}} \to \R^d$ be a most-expressive pairwise representation function (cf.\ \Cref{def:inv-pair-emb}), we have that the causal DAG in \Cref{fig:gamma-scm} can be equivalently described by the causal DAG in \Cref{fig:w-scm}, which by \Cref{col:estimator} yields 
\begin{equation}\label{eq:learning_general}
P(\Y{uv} \mid \sY ) = \int_{\bW_{\Gamma^\star}} P(\Y{uv} \mid \bW_{\Gamma^\star}) P(\bW_{\Gamma^\star} \mid \sY ) \,d\bW_{\Gamma^\star}.
\end{equation}
\end{proposition}
Note that in practice we will approximate the above equation by
$$
P(\Y{uv} \mid \sY) \approx P(\Y{uv} \mid \bW^\star_{\Gamma^\star} ),
$$
where $\bW^\star_{\Gamma^\star} = \argmax_{\bW_{\Gamma^\star}} P(\bW_{\Gamma^\star}) \prod_{m=1}^M P( \Y{ i^{(m)}j^{(m)} } \mid \bW_{\Gamma^\star}) $ is a MAP estimate and the prior $P(\bW_{\Gamma^\star})$ is a hyperparameter of our model.\\
\paragraph{\Cref{prop:2} in practice} Let us now summarize the practical conditions in which the estimator presented in \Cref{prop:2} (leveraged in our final solution) can be used. We can break down the assumptions into three sets: i) causal model and mechanism invariances (\Cref{sec:scm}, \Cref{ass:time-igno,ass:time-exc,ass:link,ass:id-exc}); ii) parameter-sharing in orbit representation (\Cref{def:joint}); and iii) non-interference and i.i.d.\ exogenous variables (\Cref{def:y-int}). 

The causal modeling assumptions (i) are world models that we must believe to be true so we can identify the causal quantities in the task at hand ---as generally needed in causal inference settings~\cite{pearl2009causality}. Regarding (ii), the assumption is related to the intuition we maintained throughout the paper: Structure induces link formation. That is, in \Cref{prop:2} we are assuming that two isomorphic node pairs have the same probe distribution and that their orbit representations are given by a common set of parameters. In theory, since $\Gammastar$ is most-expressive, we can assign arbitrary and unique representations to each orbit and thus they are not necessarily related. However, in practice the parameter-sharing will induce similar predictions to node pairs that have similar, but not necessarily identical, structure. Therefore, generally speaking, under the assumption that structure can predict the (interventional) formation of  links, this assumption is justified. One way to test this in practice is to probe into node pairs with similar structure and (statistically) test whether the observed probe distributions are the same. In \Cref{sec:res-struct} we do  this for real-world recommender systems data and confirm the hypothesis.

Finally, we have the assumptions about non-interference and i.i.d.\ exogenous variables (iii). These are assumptions may not hold in practice but, unlike (i,ii), they can be enforced by experimental design. Non-interference can be induced by sampling training (interventional) data from distant (\textit{e.g.}, different clusters) parts of the graph, as extensively discussed in Eckles et al.~\cite{eckles2017design}. Non-interference is important for our counterfactual query since it allows us to use multiple interventions as training data. Note that our query is about $t_1$ and we can only interfere sequentially. As for exogenous variables, we can first think of them as the hidden, \textit{i.e.}, unobserved, contexts of each pair. Having this in mind, it is clear that independent contexts can expected to hold by sampling training (interventional) data from distant parts of the graph. Lastly, having identical context (exogenous) distributions is a challenge often arising in causal inference~\cite{joachims2021recommendations}. A common way to accomplish it through experimental design is to stratify, \textit{i.e.}, train a model for each population stratum expected to have similar context (exogenous) distributions, \textit{e.g.} user demographics.
\subsection{Causal lifting induces a supervised learning solution to causal link prediction}
\label{sec:beyondprobes}
\Cref{prop:2} describes when it is possible to estimate our general causal link prediction task (\Cref{eq:cf-multi}) with multiple probes in different orbits. Let us now leverage such result to present our \emph{final practical solution to the task}. We are interested in a supervised learning-based approach, where the examples are the probed pairs $(i^{(t_m)},j^{(t_m)})$ with their observed outcomes $(\sy_m)$ as their labels. We are also interested in graph embedding models, \textit{i.e.}, our prediction of \Cref{eq:cf-multi} is given by $\rho\big(\Gamma(U,V,\A{t_0}{};\bW_\Gammastar);\bW_\rho^\star\big)$, where $\Gamma(U,V , \A{t_0}{} ; \bW_\Gammastar )$ outputs a pairwise representation (embedding) of $(U,V)$ in $\G{t_0}$ and $ \rho(\cdot ; \bW_{\rho} )$ is a link function (\textit{e.g.}, a downstream neural network) where the parameters $\bW_\rho^\star,\bW_\Gamma^\star$ are learned by solving
\begin{equation}\label{eq:obj}
    \bW_\rho^\star,\bW_\Gamma^\star :=
\argmin_{\bW_\rho,\bW_\Gamma} \frac{1}{M} \sum_{m=1}^M
\cL\Big(\sy_m, \rho\big(\Gamma(i^{(t_m)},j^{(t_m)}, \A{t_0}{};\bW_\Gamma);\bW_\rho\big)\Big),
\end{equation}
with $\IJt{m} \sim \pol$, where $\pol$ is an arbitrary distribution over node pairs, $\sy_m$ is the outcome of probe $(i^{(t_m)},j^{(t_m)})$ at time $t_m$ and $\cL$ a nonnegative loss function that optimizes $\bW_\rho^\star,\bW_\Gammastar$ towards the MLE estimates of the task. 


We can now see how our solution in \Cref{eq:obj} is estimating $\bW^\star_{\Gamma}$ according to \Cref{prop:2} ---where we assume a non-informative prior over $\bW_{\Gamma}$. We can also interpret $\rho(\cdot, \bW_\rho^\star)$ as our estimation of the mechanism that takes $\Gamma(I,J,\A{t_0}{};\bW_\Gammastar)$ and outputs $\A{t_1}{IJ}$. Finally, we highlight that \Cref{eq:obj} relies on the assumptions from \Cref{prop:2} to guarantee that $\rho(\cdot, \bW_\rho^\star)$ is an unbiased estimate of our counterfactual query from \Cref{eq:cf-multi}.

With the myriad of existing graph embedding methods, we are finally left with the question: What are good choices for $\Gamma$? As usual in machine learning, a choice of $\Gamma$ is better than another if it achieves lower error with the same (or less) number of samples. Next, we show how classical node embedding methods, such as matrix factorization and graph neural networks, fail at either achieving low error or capturing the correct causal structure of the task. As an alternative, we show how structural (joint) pairwise embeddings can overcome the existing issues with node embeddings.

\section{Graph Embeddings for Causal Link Prediction}\label{sec:embeddings}

We now examine our general identification result (\Cref{prop:2}) when using structural and (strictly) positional node embeddings ---two widely-used family of graph embedding predictors--- and structural (joint) pairwise embeddings. Node embedding methods build $\Gamma$ by separately computing the node embeddings of the two nodes in the represented pair. Then, they are merged by some binding function. For instance, if the graph is undirected, such binding can be done using a Hadamard product of the two node embeddings. We note that our following analysis is agnostic to the choice of the binding function. Therefore, without loss of generality, in node embedding methods we will consider $\Gamma$ as the concatenation of the two embeddings. The (arbitrary) binding operation is then incorporated by the link function $\rho$.

Overall, we find that \textbf{structural (joint) pairwise embeddings present lower bias than structural node embeddings and, unlike strictly positional node embeddings, correctly represents the causal structure of the task}.
Having the causal model from \Cref{fig:w-scm} in mind, it is natural to turn to structural representations as a candidate graph embedding choice.

\subsection{Structural (joint) pairwise embeddings are the ideal graph embeddings for causal link prediction tasks}

In \Cref{prop:2}, and its causal DAG in \Cref{fig:w-scm}, we see that the natural representation for our causal link prediction task is a most-expressive pairwise representation as in \Cref{def:joint}, which for any input pair $(i,j) \in \Vt{t_0} \times \Vt{t_0}$ it is a (possibly unique) representation of the orbit $\ORB{ij}$. 
That is, under the conditions of \Cref{thm:main}(ii), most-expressive pairwise embeddings capture all (causal) invariances in the task. 
In practice, however, one generally chooses a less expressive model family, since most-expressive graph models necessarily incur high computational costs (since they must solve the graph isomorphism task~\cite{Che+2019}).

Composing structural node embeddings is a popular attempt to design structural representations of node pairs~\cite{hu2020ogb}. However, as we later show in \Cref{sec:struct-node}, even most-expressive structural node embeddings can induce model bias in causal link prediction tasks. Thus, to distinguish structural representations acting jointly on the node pair from structural representations acting separately on its nodes, we next define structural joint pairwise embeddings.

\begin{definition}[Structural joint pairwise embeddings]\label{def:pair-struct}
    A structural joint pairwise embedding is given by the functional $\Gammajoint \colon \Vt{t_0} \times \Vt{t_0} \times \domA^{\n{t_0} \times \n{t_0}} \times \R^p \to \R^d$, parameterized by $\bW_{\Gammajoint} \in \R^p$, $p\geq 1$, where for some graph $\a{t_0}{} \in \domA^{\n{t_0} \times \n{t_0}}$ and some parameter $\bW_{\Gammajoint} \in \R^p$ the joint representation encodes more than the node orbits of $i \in \Vt{t_0}$ and $j \in \Vt{t_0}$ in $\a{t_0}{}$ separately, it encodes $\ij$'s the joint orbit $\ORB{\ij}$.
    More precisely, $\Gammajoint$ is such that 
    \begin{itemize}[leftmargin=*]
    \item[i.] for any parameter $\bW_{\Gammajoint}$, any input graph $\a{t_0}{} \in \domA^{\n{t_0} \times \n{t_0}}$ and any pair $(i,j) \in \Vt{t_0} \times \Vt{t_0}$, we have that $\Gammajoint(i,j,\A{t_0}{};\bW_{\Gammajoint}) = \Gammajoint(k,l,\A{t_0}{};\bW_{\Gammajoint}), \forall \; (k,l) \in \ORB{i,j}$, and
    \item[ii.] $\exists \bW_{\Gammajoint}' \in \R^p$, $\exists \a{t_0}{} \in \domA^{\n{t_0} \times \n{t_0}}$, and $\exists \ij \in \Vt{t_0} \times \Vt{t_0}$ such that it does \textbf{not} exist a function over the set of node orbits ($\bigcup_{v \in V} \{\ORB{v}\}$) that can be equivalent to $\Gammajoint$, that is, for all $f \colon \bigcup_{v \in V} \{\ORB{v}\} \times \bigcup_{v \in V} \{\ORB{v}\} \to \R^d$ we have that $f(\ORB{i}, \ORB{j}) \neq \Gammajoint(i,j,a;\bW'_{\Gammajoint})$.
    \end{itemize}
\end{definition}

Note how the above definition guarantees that for some input graph $\G{t_0}$, its structural pairwise embedding could not have been computed from the nodes' individual orbits. At first sight it might seem that only most-expressive pairwise representations are encompassed by \Cref{def:pair-struct}, but this is not true. For instance, pairwise distances cannot be computed from node orbits (see \cite{li2020distance}), thus any structural pairwise representation capturing distances would be a joint embedding as in \Cref{def:pair-struct}. In practice, we have seen how even for observational link prediction tasks (\textit{e.g.}, see the OGB~\cite{hu2020ogb} suite\footnote{\url{https://ogb.stanford.edu/docs/linkprop/}}) architectures under \Cref{def:pair-struct} consistently outperform others using only structural node representations. This serves as evidence that joint properties, \textit{i.e.}, as in \Cref{def:pair-struct}(ii), such as distances and common neighbors indeed play a central role in link formation mechanisms. Generally speaking, the benefits enjoyed by estimators using \Cref{def:pair-struct} are tied to how much structural properties govern the (interventional) link formation process. In \Cref{sec:res-struct} we test this assumption in practice. Finally, for simplicity we will often refer to representations satisfying \Cref{def:pair-struct} simply as {\em structural pairwise embeddings}.

\subsection{Structural node embeddings are undesirable for causal link prediction due to model bias}\label{sec:struct-node}

Unlike joint embeddings (\Cref{def:pair-struct}), Graph Neural Networks (GNNs)~\cite{Sca+2009} are the most common permutation-invariant\footnote{Invariant node embeddings can also be seen as equivariant embeddings, where the input of the model is the adjacency $a$ and the output is a matrix of embeddings $H \in \R^{n \times d}$. Then, equivariance is achieved by having the action of permutation $\pi$ in the input $\pi \cdot a$ resulting in permuting the rows of the output matrix of embeddings accordingly, \textit{i.e.}, $\pi \cdot H$.} node embeddings (a.k.a.\ structural node embeddings), where each node gets its own representation as described in \Cref{def:inv-node-emb}. A structural node embedding outputs the same node representations to any two isomorphic nodes in $\G{t_0}$. That is, if $\pi \in \aut{\G{t_0}}$ is an automorphism of $\G{t_0}$, then for any $i \in \Vt{t_0}$ we have that $\pi \cdot \a{t_0}{} = \a{t_0}{}$ and thus nodes $i$ and $\pi \cdot i$ are assigned the same representation in $\a{t_0}{}$. Overall, structural node embeddings encompass not only GNNs but also classical structural node roles~\cite{henderson2012rolx}.
\begin{definition}[Structural (permutation-invariant) node embeddings~\cite{srinivasan2020equivalence}] \label{def:inv-node-emb}
    A structural node embedding is given by the functional $Z \colon \Vt{t_0} \times \domA^{\n{t_0} \times \n{t_0}} \times \R^p \to \R^d $ parameterized by $\bW_{Z} \in \R^p, p \geq 1$, where $Z(i,\a{t_0}{};\bW_{Z}) =  Z(\pi \cdot i,\pi \cdot \a{t_0}{};\bW_{Z}) \; \forall \; \pi \in \S_{\n{t_0}}, \forall \; \a{t_0}{} \in \domA^{\n{t_0} \times \n{t_0}} , \forall \; i \in \Vt{t_0}$.
\end{definition}



We now highlight that \Cref{def:inv-node-emb,def:inv-pair-emb} are not equivalent. In fact, next we prove that the family of all models satisfying \Cref{def:pair-struct} have lower bias than those satisfying \Cref{def:inv-node-emb}. We first define the bias of a graph embedding solution to our problem below.

\begin{definition}[Graph embedding bias]\label{def:bias} We define the bias of a representation $\Gamma(\cdot;\bW_\Gamma)$ together with a link function $\rho(\cdot;\bW_\rho)$ for a fixed $t_0$ as the expectation
\[
 \cB_{\Gamma,\rho} := \min_{\bW_{\Gamma}, \bW_{\rho}} \Exp \Bigg[
\cL\Big(\Y{IJ}, \rho\big(\Gamma(I,J, \A{t_0}{};\bW_\Gamma);\bW_\rho\big)\Big) \Bigg]
\]
taken over $\IJ$, $\Y{IJ}$ and $\A{t_0}{}$ with $\cL$ being a loss function as in \Cref{eq:obj}.
\end{definition}

Now, we define a theoretically-relevant class of graphs in link prediction tasks: Pairwise symmetric graphs (c.f.\ \Cref{def:pair-graph}). In essence, pairwise symmetric graphs are symmetric graphs that contain pairs of nodes that are node-wise isomorphic. That is, there exist permutations that map nodes individually from one pair to the other, but there are at least two of these pairs that are not isomorphic, \textit{i.e.}, it does not exist a single permutation bringing one pair to the other. These graphs are important since structural node embeddings will mistakenly assign symmetries to non-isomorphic pairs ---and thus they fail to fulfill \Cref{def:pair-struct}(ii). Instead, we need to consider pairwise symmetries by using structural pairwise embeddings as in \Cref{def:pair-struct}.

\begin{definition}[Pairwise symmetric graphs]\label{def:pair-graph}
    We say that a symmetric graph $a$ is pairwise symmetric if $\exists \; \pi,\pi' \in \aut{a}, \exists \; i,j,u,v \in V$ such that $ u = \pi \cdot i, v = \pi' \cdot j $, but $\not\! \exists \; \pi^\star \in \aut{\A{t_0}{}}$ such that $ (u,v) = \pi^\star \cdot (i,j)$.
\end{definition}

We are now finally ready to state when structural pairwise embeddings have lower bias than structural node embeddings in \Cref{{thm:bias}}.

\begin{restatable}[]{thm}{biasthm}\label{thm:bias}
Let $\cB_{Z,\rho}$ and $\cB_{\Gammajoint,\rho}$ be the respective biases (c.f.\ \Cref{def:bias}) of models using structural node embeddings (c.f.\ \Cref{def:inv-node-emb}) and structural pairwise embeddings (c.f.\ \Cref{def:pair-struct}) in our causal link prediction task as described in \Cref{prop:2}. Then, we have that if $\supp(\A{t_0}{})$ contains at least one pairwise symmetric graph (cf.\ \Cref{def:pair-graph}),
\[ \cB_{Z,\rho} \geq \cB_{\Gammajoint,\rho} = 0. \]
\end{restatable}

The core of \Cref{thm:bias}'s proof (c.f.\ \Cref{sec:appx-learn}) relies on showing (i) how structural node embeddings cannot achieve zero error in pairwise symmetric graphs and (ii) that an SCM satisfying the conditions of \Cref{prop:2} can generate pairwise symmetric graphs ---thus, structural node embeddings cannot achieve zero-bias in causal link prediction tasks. We highlight how this result encompasses even most-expressive structural node embeddings, \textit{i.e.}, representations unique to each node orbit. A natural follow-up question is: What about matrix factorization-based methods? Since they do not explicitly capture invariances, they do not seem to suffer from the same higher-bias problems as structural node embeddings. Next, we show how matrix factorization and other types of (strictly) positional node embeddings, unlike structural pairwise embeddings, do not capture the correct causal structure of our task.

\subsection{(Strictly) positional node embeddings (\textit{e.g.}, matrix factorization embeddings) incorrectly encode the causal structure for interventional lifting}

Beyond structural node embeddings, it is also natural to consider positional node embeddings as our choice for building $\Gamma$. The formal and most general definition of positional node embeddings is given in \cite{srinivasan2020equivalence}, where the embeddings are samples from an equivariant distribution.

\begin{definition}[Positional node embeddings~\cite{srinivasan2020equivalence}]\label{def:pos}
    The positional node embeddings $(\theta_i)_{i = 1}^{\n{t_0}}$ of a graph $\G{t_0}$ with adjacency $\A{t_0}{}$ are defined as joint samples $(\theta_i)_{i = 1}^{\n{t_0}}\mid  \A{t_0}{} =\a{t_0}{} \sim P( \btheta \mid \A{t_0}{} = \a{t_0}{})$ with $\supp(\theta_i) \subseteq \R^d$ and $ P(\pi \cdot \btheta \mid \A{t_0}{} =  \a{t_0}{}) = P(\btheta  \mid \A{t_0}{} = \pi \cdot \a{t_0}{} )$. 
\end{definition}

Deterministic algorithms for matrix factorization may not appear to follow the above definition, however, once we consider that node identifiers are arbitrarily assigned, even deterministic factorization methods follow \Cref{def:pos}. We refer the reader to Srinivasan \& Ribeiro~\cite{srinivasan2020equivalence} for a more comprehensive discussion.
At this point, the attentive reader has noted how the presented definition of positional node embeddings is quite general. As such, as previously noted, the embeddings possibly do not encode any graph symmetry, \textit{e.g.}, assign the same representation to isomorphic pairs in the graph. However, the definition is so general that the opposite might also be true, \textit{i.e.}, structural node embeddings can be framed as positional in this framework. Thus, to capture the positional node embeddings that do not encapsulate any symmetry, we next define strictly positional node embeddings.

\begin{definition}[Strictly positional node embeddings]\label{def:strict}
Given a symmetric graph $\G{t_0}$ ($\setsize{\aut{\G{t_0}}} > 1$), we say that the positional node embeddings $(\theta^{(\text{pos+})}_i)_{i=1}^n$ of nodes in $\G{t_0}$ are 
strictly positional if for every $\ij \in \Vt{t_0} \times \Vt{t_0}, i \neq j$ we have that $\theta_i^{(\text{pos+})} \neq \theta_j^{(\text{pos+})} $ almost everywhere.
\end{definition}

In words, \Cref{def:strict} is defining the set of positional node embedding distributions that always assign distinct representations to every node in a symmetric graph. Note that we focus on symmetric graphs here since in asymmetric graphs ($\setsize{\aut{\G{t_0}}} = 1$) even structural node embeddings can assign distinct representations to all nodes. We start by further noting how, just like structural pairwise embeddings, strictly positional node embeddings can achieve zero-bias in any causal link prediction task. It follows from 
Theorem 2 in \cite{srinivasan2020equivalence} that, with a powerful enough link function, strictly positional embeddings can in expectation recover most-expressive pairwise representations (cf.\ \Cref{def:inv-pair-emb}) ---which achieve zero-bias as shown in \Cref{thm:bias}.


Although strictly positional node embeddings do not necessarily impose a model bias to the causal link prediction task, they do suffer from a different challenge: \Cref{fig:pos-scm} shows that because strictly positional node embeddings generally assigns different embeddings for pairs of nodes in the same orbit, they do not capture the correct causal DAG of \Cref{fig:scm-ex}.
\begin{figure}[ht]
\begin{center}
\tikzstyle{vertex}=[font=\normalsize, draw, fill=white, circle, minimum size=20pt,inner sep=1pt, align=left]
\tikzstyle{selected vertex} = [draw, vertex, fill=green!24]
\tikzstyle{edge} = [draw,thick]
\tikzstyle{weight} = [font=\small, align=center, above, sloped, inner sep=2pt, pos=0.3]
\tikzstyle{selected edge} = [draw,line width=5pt,-,green!50]
\tikzstyle{ignored edge} = [draw,line width=5pt,-,black!20]
\pgfdeclarelayer{bottom}
\pgfdeclarelayer{middleb}
\pgfdeclarelayer{middlet}
\pgfdeclarelayer{top}
\pgfsetlayers{bottom,middleb,main,middlet,top}
\begin{tikzpicture}[scale=1.5]
\node[vertex,ellipse, minimum size=20pt, fill=white](Gamma) at (-2.1,-0.5) {$(\theta^{(\text{pos+})}_I,\theta^{(\text{pos+})}_J)$};

\node[vertex,ellipse, minimum size=20pt, fill=white](GammaKL) at (0.55,-0.5) {$(\theta^{(\text{pos+})}_K,\theta^{(\text{pos+})}_L)$};

\node[vertex,fill=black!10](kl) at (2.35,-0.4) {$(K,L)$};
\node[vertex,fill=black!10](ij) at (-4,-0.2) {$\IJ$};
\node[draw, minimum height=3.2cm,minimum width=7cm, ultra thick, rounded corners] at (1.5,-0.6) {};
\node[] at (1.47,0.1) {For all $(K,L) \in \ORB{IJ} \backslash \{\IJ\} $: };
\node[vertex](U) at (3.05,-1.) {$\ecal{U}_{(K,L)}$};
\node[vertex](Uij) at (-4.5,-1.) {$\ecal{U}_{\IJ}$};
\node[vertex](A) at (1.6,-1.) {$\Y{KL}$};
\node[vertex,fill=black!10](Aij) at (-3.5,-1.) {$\Y{IJ}$};
\path[edge] (ij) -> (Gamma);
\path[edge] (ij) -> (Uij);
\path[edge] (kl) -> (U);
\path[edge] (kl) -> (GammaKL);
\path[edge] (U) -> (A);
\path[edge] (Uij) -> (Aij);
\path[edge] (GammaKL) -> (A);
\path[edge] (Gamma) -> (Aij);
\end{tikzpicture}






\caption{Strictly positional node embedding's incorrect causal structure of probes in $(i,j)$ (\textit{left}) and other pairs in the orbit of $\ORB{\ij}$ (\textit{right}).}
\label{fig:pos-scm}
\end{center}
\end{figure}
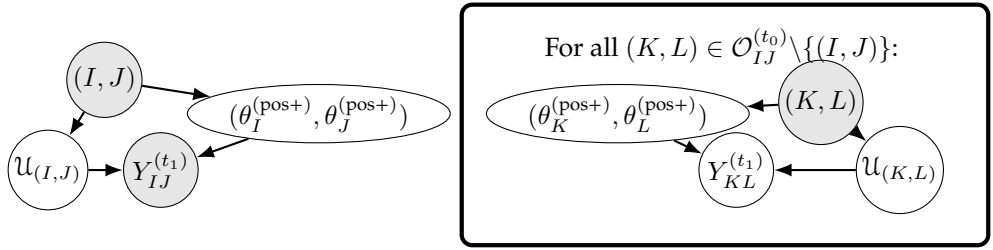
Further, having the causal model from \Cref{fig:scm-ex} in mind, we can see how strictly positional node embeddings make it difficult to generalize to unseen node pairs. In practice, unless we choose or learn a link function (a link function is the function that takes the embeddings and outputs a link prediction) that assigns the same prediction to the different embeddings of node pairs in the same orbit, the lifting does not occur.

\paragraph{Matrix factorization gives strictly positional node embeddings} 
When considering (strictly) positional node embeddings, we are interested in matrix factorization and other factor models that reflect some notion of node distances in the graph in their node embeddings, such as metric embeddings~\cite{chen2012playlist,feng2015personalized} and similar traditional (associational) link prediction methods~\cite{quadrana2018sequence,grover2016node2vec,perozzi2014deepwalk}. Thus, to illustrate how the analysis of strictly positional node embeddings is insightful for factor models, we now prove that, in a wide family of graphs, {\em SVD behaves as a strictly positional embedding}.
~\\

SVD node embeddings use the eigenvectors of $\a{t_0}{}{\a{t_0}{}}^T$ and ${\a{t_0}{}}^T\a{t_0}{}$ ---here we consider concatenating both right and left eigenvectors. SVD justifies the positional node embedding name since it does reflect the distances between nodes in the graph in its embedding space. As isomorphic nodes can be arbitrarily distant in a graph, it is easy to see that SVD can assign them different embeddings. Thus, it has been believed and recently conjectured in \cite{zhu2021node}[Theorem 4.2] that SVD in general does not assign the same embedding to isomorphic nodes. To fill this gap in literature, we next present the first result ---to the best of our knowledge--- on the exact invariances of SVD: When SVD assigns the same embedding to a node pair. The proof of \Cref{thm:svd} is presented in \Cref{sec:appx-learn}.
 
\begin{restatable}[The invariances of SVD]{thm}{svdthm}\label{thm:svd} 
Let $G$ be a graph with adjacency $a \in \{0,1\}^{n \times n}$ and $\theta^{(\text{SVD})}_i$ the SVD embedding of node $i \in V$. Then, two nodes $i,j \in V$ get the same SVD embedding $\theta^{(\text{SVD})}_i = \theta^{(\text{SVD})}_j$
if and only if
\begin{enumerate}[label=(\alph*)]
    \item the nodes are isomorphic $i \iso j$; and
    \item they have the exact same neighborhood: $a_{iv} = a_{jv},  a_{vi} = a_{vj}, \; \forall v \in V.$
\end{enumerate}
\end{restatable}

From \Cref{thm:svd} we can directly derive \Cref{col:strict-svd}, in which we outline the exact set of symmetric graphs where SVD embeddings are strictly positional. These symmetric graphs are quite specific: they have duplicate nodes, \textit{i.e.}, nodes that are not only symmetric, but that \textbf{carry the exact same information: equal neighborhood}.

\begin{corollary}\label{col:strict-svd}
    Given a symmetric and unattributed ($\domA = \{0,1\}$) graph $\G{t_0}$, its SVD embeddings $(\theta^{(\text{SVD})}_i)_{i=1}^{\n{t_0}}$ are strictly positional if and only if there exist two nodes $i,j \in \Vt{t_0}$ such that $j \in \ORB{i}$ and $\a{t_0}{iv} = \a{t_0}{jv}, \; \forall v \in \Vt{t_0}$.
\end{corollary}
%
%
%
%
Note that even in graphs with duplicate nodes, SVD will assign different embeddings to all other nodes. Therefore, although it can capture some degree of invariance from very specific nodes, SVD in general behaves as a strictly positional embedding. As such, it encodes the causal process from \Cref{fig:pos-scm}, \textit{i.e.}, it does not encapsulates the known invariance of the task: two isomorphic pairs must be assigned the same prediction.

Apart from SVD, we can also relate strictly positional node embeddings to neural matrix factorization~\cite{dziugaite2015neural}. Neural matrix factorization can be seen as using a one-hot encoding of the node identifier as the node embedding. Then, it uses a multi-layer perceptron in the link function ---in general if the graph is undirected there exists some level of weight-sharing between the two node embeddings in the link function as well. Since node identifiers are by definition unique, it is straightforward to see that {\em neural matrix factorization produces strictly positional node embeddings}.

Overall, our theory implies that the use of strictly positional node embeddings, such as SVD and neural matrix factorization, entails an incorrect causal structure under our assumptions that does not result in a causal lifting for the causal link prediction task in \Cref{eq:cf}. As a result, models using strictly positional node embeddings will not generalize as well as models using pairwise structural embeddings. More specifically, the sensitivity to training data here is attached to the fact that positional node embeddings might overfit to the training probes while not capturing its symmetry (invariance) properties of their orbits. Next, in the experimental section, we show how this result indeed translates into practice: (Strictly) Positional node embeddings struggle with generalizing to node pairs not probed during training.



\newcommand{\movetoappx}[1]{Details in the appendix}

\section{Experimental Results}\label{sec:res}

We now evaluate our theoretical findings on identifying and estimating the counterfactual query from \Cref{eq:cf-multi} with experimental data using our supervised learning solution from \Cref{eq:obj}. Concretely, we investigate how accurately (strictly) positional node embeddings (\Cref{def:strict}), structural node embeddings (\Cref{def:inv-node-emb}), and pairwise invariant embeddings (\Cref{def:pair-struct}) can estimate our counterfactual query from \Cref{eq:cf-multi}. We are interested in empirically answering the following four questions. 

\begin{enumerate}[leftmargin=*]
    \item[\textbf{ (Q1)}] How do structural node embeddings perform when our test set contains tuples $\UV$ that are node-wise isomorphic to the ones used in training $\IJt{1}, \ldots, \IJt{M}$? In \Cref{thm:bias} we proved that structural node embeddings are not unbiased estimators of \Cref{eq:cf-multi} in this setting. We are now interested in evaluating the practical impact of this result.
    \item[\textbf{ (Q2)}] How do (strictly) positional node embeddings perform when our test set contains tuples $\UV$ such that neither $U$ or $V$ participated in the probes $\IJt{1}, \ldots, \IJt{M}$ used for training? In \Cref{fig:pos-scm} we can see that (strictly) positional node embeddings cannot learn a single representation for multiple pairs. Does such a property translate into poor practical performance when new pairs are tested? That is, in an inductive setting where we have to extrapolate our knowledge from training, do we need to explore symmetries?
    \item[\textbf{ (Q3)}] How do (strictly) positional and structural (node and pairwise) embedding solutions perform in real-world scenarios where we are not aware of the underlying causal model? Here we are interested in testing all the assumptions involved in \Cref{prop:2} in real-world data.
    \item[\textbf{ (Q4)}] (\Cref{sec:res-struct}) In \Cref{fig:w-scm} we can see a central assumption in \Cref{prop:2}: The (interventional) link formation process can be retrieved by a common set of parameters representing the structure of the node pairs. Here we ask: Is this assumption reasonable with real-world data? To answer this, we test whether structurally similar node pairs have the same probe outcome distribution (\Cref{eq:int}).
\end{enumerate}

We evaluate \textbf{Q1} and \textbf{Q2} in a knowledge base completion task, \textbf{Q2} in a covariance matrix estimation task and \textbf{Q2}, \textbf{Q3} and \textbf{Q4} in two user-item recommendation tasks. For each scenario, we present the problem as the counterfactual query in causal link prediction and discuss the conditions for its identification. Next, we briefly introduce the graph embedding methods we evaluate in the proposed tasks.
\\~\\
\paragraph{Node embedding methods}
As representatives of (strictly) \textbf{positional node embedding} methods we use Nonnegative Matrix Factorization (NMF)~\cite{lee1999learning}, SVD~\cite{halko2009finding} and positional GCN~\cite{Kip+2017}, which is obtained by using node identifiers as node features. As for \textbf{structural node embeddings}, we choose the classical GCN~\cite{Kip+2017} model and the more recent UniMP~\cite{shi2020masked} GNN. Specifically for the knowledge base experiment, we also choose some popular \textbf{knowledge graph embeddings}, namely TransE~\cite{bordes2013translating}, DistMult~\cite{yang2014embedding} and ComplEX~\cite{trouillon2016complex}. For all methods, we use multi-layer perceptrons as link functions ($\rho$ in \Cref{eq:obj}). Implementation details can be found in \Cref{sec:data}.
\\~
\paragraph{Structural pairwise embedding methods}
The importance of structural pairwise embeddings has been showed only recently in \cite{srinivasan2020equivalence} and thus its literature is still underexplored. Here, we use SEAL~\cite{zhang2018link} and Neo-GNNs~\cite{yun2021neo} as representative methods. SEAL extends GNNs to output structural pairwise embeddings. The idea behind SEAL was generalized and named labeling trick~\cite{zhang2021labeling}. Essentially, methods of this kind apply a GNN over the graph but mark (only) the two nodes in the represented pair. In its simplest form, this marking can be done by a single bit added to the node features ---indicating whether the node is in the pair. Here, we also consider what we call Label-GCN ---an approximate and more scalable application of the labeling trick. Originally, the labeling trick computes the GNN representation of each pair separately. This process is considerably less scalable than a standard GNN model where only one GNN computation is performed to represent all pairs. To reduce the computational burden of the labeling trick, in Label-GCN we mark all the pairs in a mini-batch of our (stochastic) optimization procedure. By using a small mini-batch size, we expect representations from a sparse graph to not interfere with each other ---providing a good approximation of the labeling trick method. Note that during test we also need to use the same small mini-batch size. Implementation details can be found in \Cref{sec:data}.

\subsection{Impact on knowledge graph queries}\label{sec:res-kg}
Here we consider the causal link prediction task in knowledge graphs. In this scenario, we are interested in questions as those described earlier:
\begin{displayquote} ``Given our current knowledge about the world and some fact-finding mission that added a new piece of information, \textit{i.e.}\ new relations, what other pieces would have been added had we investigated these relations in other parts of the graph?''
\end{displayquote}

\begin{figure*}[ht]
\centering
\begin{minipage}[t]{.4\textwidth}
\centering
\includegraphics[width=1\linewidth]{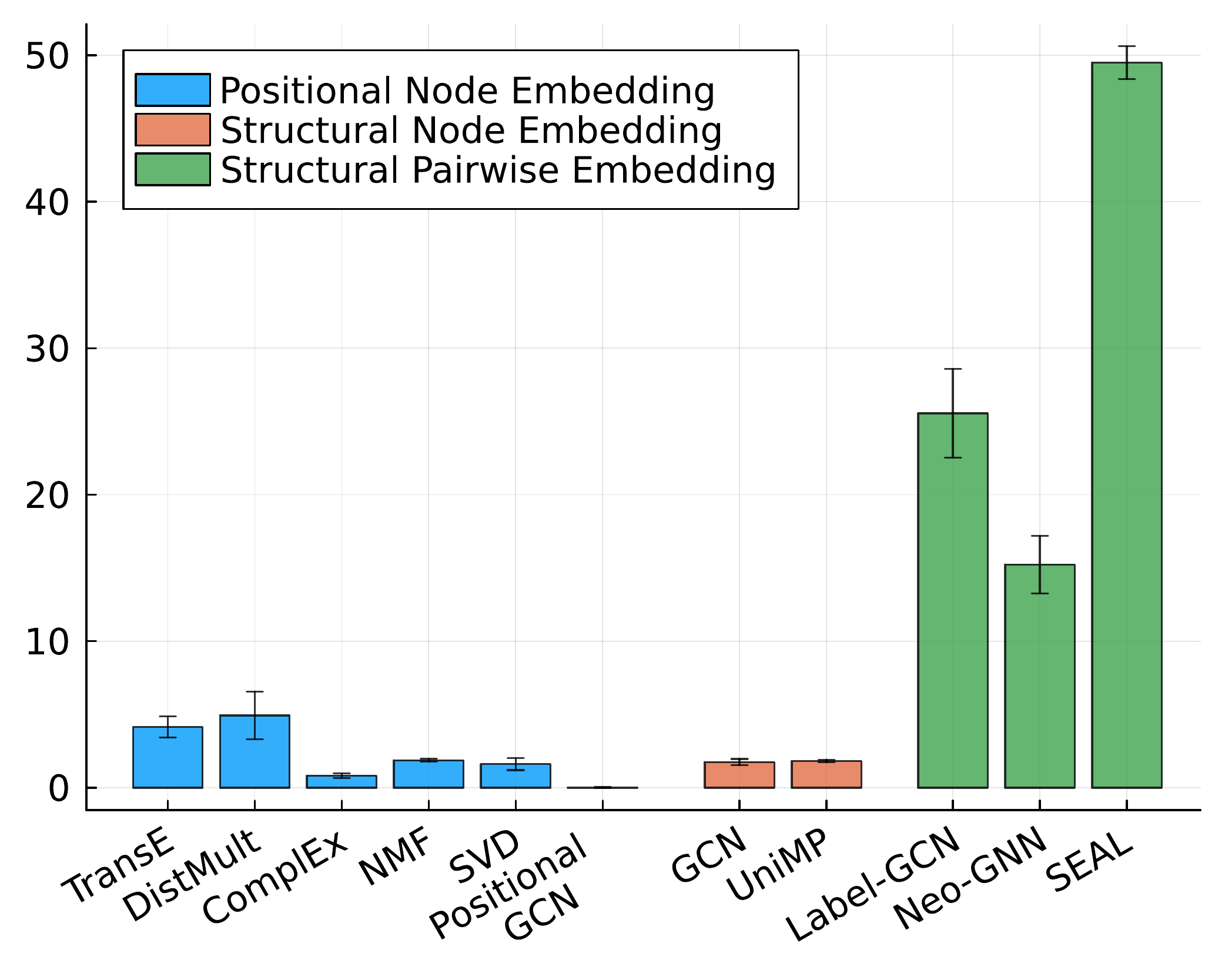}
\caption{Results for the family tree dataset. We present the average and the standard deviation over five runs for the Hits@500 (out of 10,538 potential edges) metric in each method. Note that the larger $\uparrow$ the better the performance.}
\label{fig:family-tree}
\end{minipage}\qquad
\begin{minipage}[t]{.55\textwidth}
\centering
\resizebox{1\linewidth}{12em}{
\tikzstyle{vertex}=[font=\scriptsize, draw, fill=black!10, circle, minimum size=20pt,inner sep=1pt, align=center]
\tikzstyle{selected vertex} = [draw, vertex, fill=green!24]
\tikzstyle{edge} = [draw,thick]
\tikzstyle{weight} = [font=\scriptsize, align=center, above, sloped, inner sep=2pt, pos=0.5]
\tikzstyle{green edge} = [draw,line width=4pt,-,green!50]
\tikzstyle{red edge} = [draw,line width=4pt,-,red!50]
\tikzstyle{ignored edge} = [draw,line width=5pt,-,black!20]
\pgfdeclarelayer{bottom}
\pgfdeclarelayer{middleb}
\pgfdeclarelayer{middlet}
\pgfdeclarelayer{top}
\pgfsetlayers{bottom,middleb,main,middlet,top}

\begin{tikzpicture}[scale=3.8]

\node[vertex](John) at (0,0) {John};
\node[vertex](Bob) at (0,-0.5) {Bob};
\node[vertex](Anna) at (-0.5,-0.5) {Anna};
\node[vertex](Sarah) at (0,-1.0) {Sarah};
\node[vertex](Alex) at (-0.5,-1.0) {Alex};

\path[edge] (John) -> node[weight] {\texttt{parentOf}} (Bob);
\path[edge] (Bob) -> node[weight] {} (Alex);
\path[edge] (Bob) -> node[weight] {\texttt{parentOf}} (Sarah);
\path[edge] (Anna) -> node[weight] {\texttt{parentOf}} (Alex);
\path[edge] (Anna) -> node[weight] {} (Sarah);
\node[fill=white] (pa1) at (-0.25,-0.66) {\scriptsize \texttt{parentOf}};

\node[vertex](Mary) at (0.5,0) {Mary};
\node[vertex](Alice) at (0.5,-0.5) {Alice};
\node[vertex](David) at (1.0,-0.5) {David};
\node[vertex](Emily) at (1.0,-1.0) {Emily};
\node[vertex](Ben) at (0.5,-1.0) {Ben};

\path[edge] (Mary) -> node[weight] {\texttt{parentOf}} (Alice);
\path[edge] (Alice) -> node[weight] {} (Emily);
\path[edge] (Alice) -> node[weight] {\texttt{parentOf}} (Ben);
\path[edge] (David) -> node[weight] {\texttt{parentOf}} (Emily);
\path[edge] (David) -> node[weight] {} (Ben);
\node[fill=white] (pa1) at (0.75,-0.66) {\scriptsize \texttt{parentOf}};

\path[edge] (John) -> node[weight] {} (Alice);
\path[edge] (Mary) -> node[weight] {} (Bob);
\node[fill=white] (pa1) at (0.25,-0.167) {\scriptsize \texttt{parentOf}};

\path (Alex) ->  node[weight] {\texttt{sibling}} (Sarah);
\path[green edge] (Alex) -> (Sarah);
\path[-]  (Alex)  edge   [bend right=20] node[weight] {\texttt{sibling}} (Emily);
\path[-]  (Alex)  edge [red edge, bend right=20] (Emily);
\path (Sarah) ->  node[weight] {\texttt{sibling}} (Ben);
\path[red edge] (Sarah) -> (Ben);
\path (Ben) ->  node[weight] {\texttt{sibling}} (Emily);
\path[green edge] (Ben) -> (Emily);

\node[draw,align=left, dashed] (pa1) at (1.3,0) {
\scriptsize 
\colorbox{red!50}{ Possible negative interventions at $t_1$ } \\ 
\scriptsize
\colorbox{green!50}{ Possible positive interventions at $t_1$ } }
;
\end{tikzpicture}
}
\caption{Synthetic illustrative example of a knowledge graph with two isomorphic subtrees as the knowledge base observed at (pre-trial) time $t_0$. In green and red we highlight possible new relations or non-existent relations after an intervention at (post-trial) time $t_1$.}
\label{fig:ft-ex}
\end{minipage}
\end{figure*}

To answer queries of this type, we construct a knowledge base which comprises 100 non-isomorphic family trees, built using the methods provided in Hohenecker et al.~\cite{Hohenecker2020}, whose procedure we follow verbatim. In total, the dataset contains 28 relation types. The \texttt{parentOf} relation represents the current knowledge about the world and it is used to construct the observed graph $\G{t_0}$. The goal is to predict the other family relations, which can be entirely determined just from the \texttt{parentOf} relations, \textit{i.e.}\ the observed graph $\G{t_0}$. Thus, it is straightforward that our model completely satisfies all the conditions from \Cref{thm:main,prop:2}.

In this dataset, 30\% of the family trees contain two isomorphic subtrees, that we split between train and test such that all the relations within one subtree form the training probes and all the relations in the isomorphic subtree form the test relations. Structural node embeddings assign the same representation to isomorphic nodes and therefore will not distinguish, for example, \texttt{girlCousinOf} and \texttt{sisterOf} relations, see Alex and Sarah vs.\ Alex and Emily in the pairwise symmetric graph from \Cref{fig:ft-ex}. This last example is indeed the type of situation from \Cref{thm:bias} where we evaluate \textbf{Q1}. To induce this kind of problem in the task, we consider test non-links having each end-point in a different subtree. Strictly positional node embeddings can distinguish those two relations, but they may not make the same prediction for isomorphic pairs due to higher variance (\Cref{fig:pos-scm}, \textbf{Q2}),
and therefore the predictions in one subtree might be different from the ones in the other isomorphic subtree, impacting generalization. As predicted by \Cref{thm:bias}, structural pairwise embeddings do not suffer from these issues, as shown by the results in \Cref{fig:family-tree}, where both (strictly) positional node embeddings and structural node embeddings obtain poor performances when compared to structural pairwise embeddings. We show the results using the Hits@500 metric where we get the $500$ node pairs with the highest probability of forming a link given by the model and compute the ratio that do form a link according to the data, \textit{i.e.}, their labels are positive.

\subsection{Impact on covariance matrix estimation}
We now present a task mostly unexplored in the causal link prediction literature. Consider observing samples from a joint distribution. For instance, each sample is a list of measurements from a subject (e.g., a patient) and the distribution is over their corresponding (medical) attributes. We can then use these samples to construct an estimated covariance matrix for the attribute variables. We now present the following counterfactual query:
\begin{displayquote}
``At a later time we want to refine our estimated covariance matrix by adding more data from more subjects but only for a subset of the attribute variables.
What would have been the re-estimated covariance values for the attributes that we did not collect extra data?''
\end{displayquote}
Note that in this setting the experiments are made simultaneously, since a node in a graph is an attribute variable and for a subject all pairs measurements are collected at the same time (\Cref{def:y-int}). Further, since each subject is (presumably) not related to other subjects, measurements do not interfere with the ones from other subjects (\Cref{def:y-int}). Here, the only assumption needed is that exogenous variables from subjects are i.i.d.~\Cref{thm:main}. Such an assumption should hold if subjects are selected independently at random, \textit{i.e.}, they are not related. Time exchangeability (cf.\ \Cref{ass:time-exc}) is also satisfied since the links are created simultaneously and the order in which they change (estimation is refined) is irrelevant due to the samples being i.i.d.. Further, since the covariance matrix forms a complete graph, the non-link ignorability assumption (cf.\ \Cref{ass:link}) is not needed to apply our result. Finally, the identifier of an attribute, \textit{i.e.}, the name we give to it, does interfere in its probe outcomes (cf.\ \Cref{ass:id-exc}).

\begin{figure*}[ht]
\centering
\begin{minipage}[t]{.4\textwidth}
\centering
\includegraphics[width=1\linewidth]{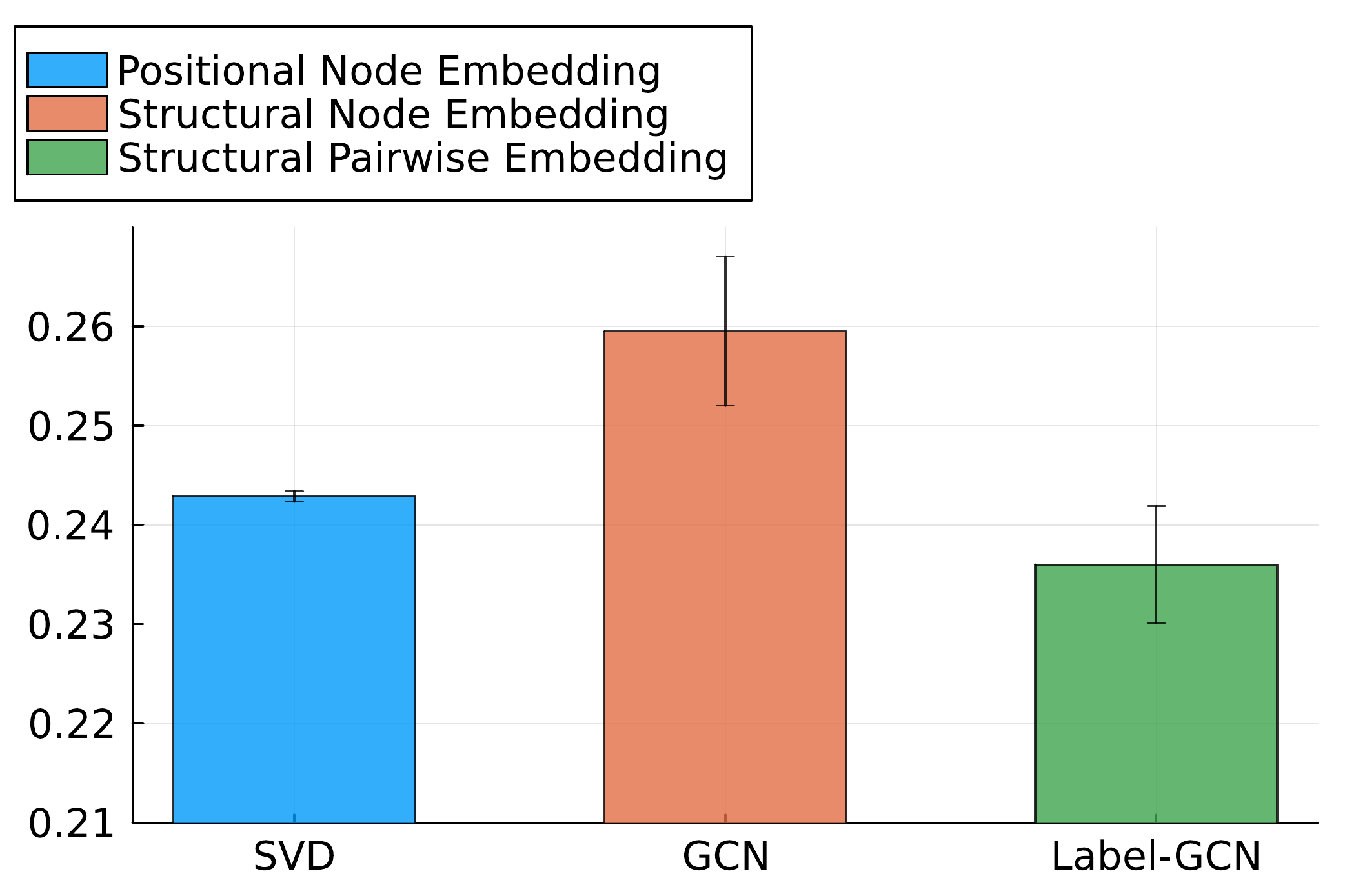}
\caption{Results for the covariance matrix dataset. We present the average and the standard deviation over five runs for the MSE metric in each method. Note that the smaller $\downarrow$ the better the performance.}
\label{fig:covariance}
\end{minipage}\qquad
\begin{minipage}[t]{.55\textwidth}
\centering
\resizebox{1\linewidth}{10em}{
\tikzstyle{weight} = [font=\small, align=center, above, sloped, inner sep=2pt, pos=0.3]
\begin{tikzpicture}[scale=0.75]
\node at (2.2,4.5) {\small Attributes};
\fill[blue!40!white] (2,2) rectangle (4,0);
\fill[orange!40!white] (2,2) rectangle (4,4);
\fill[orange!40!white] (0,0) rectangle (2,2);
\fill[orange!40!white] (0,4) rectangle (2,2);

\node (a) at (2,3) {};
\node[align=left] (b) at (9,3.5) {\scriptsize What would have been the\\ \scriptsize new estimated covariance between these \\ \scriptsize  attributes had we intervened in them?};
\path[draw, very thick, dashed] (a) -> (b);

\node (c) at (2.5,1.5) {};
\node[align=left] (d) at (9,1) {\scriptsize Covariance updated with extra\\ \scriptsize measurements after intervening in \\ \scriptsize this subset of attributes };
\path[draw, very thick, dashed] (c) -> (d);

\draw[step=0.5cm,black,very thin, -] (0,0) grid (4,4); 
\node[rotate=90] at (-0.5,2.25) {\small Attributes};
\end{tikzpicture}
}
\caption{An illustration of the process of acquiring more data (measurements) for a submatrix ---induced by a subset of attributes--- of our original estimated covariance matrix. We then query how would the rest of the matrix be updated had we acquired data for the other attributes as well.}
\end{minipage}
\end{figure*}

\paragraph{Baselines} We do not evaluate Neo-GNNs and SEAL here since these methods consider the neighborhood of a pair and, in this task, the equivalent graph $\G{t_0}$ is often complete. Hence, SEAL would be equivalent to Label-GCN and Neo-GNNs would be equivalent to a standard GNN. Therefore, we chose to evaluate only Label-GCN as a structural pairwise embedding method. Moreover, the sparsity assumption to train Label-GCN does not hold and, hence, we compute each pair representation separately in the mini-batch ---which is equivalent to performing the labeling trick. Our experiments compare Label-GCN to one method of each other node embedding type.

\subsection{Impact on recommender systems}\label{sec:res-rec}

Here we consider two user-item recommendation datasets, the Amazon Electronics (AE)~\cite{wan2020addressing} and the Last FM (LFM)~\cite{melchiorre2021investigating} datasets. We built the observed graph $\G{t_0}$ from user-item interactions occurring between 11/24/2015 and 12/24/2015 in AE and between 2007 and 2013 in LFM. Now, given these observed graphs we face the following counterfactual problem:
\begin{displayquote}
`` At a later time, we can probe by exposing a subgroup of users to items. After observing their interactions, we wonder what would the other users consume had we exposed the items to them?''
\end{displayquote}
In both datasets we consider the subgroup of male users as the one we probe in. At test time, our counterfactual queries are about female users. We use in both train and test interactions happening between 12/24/2015 and 12/31/2015 for AE and in 2014 for LFM. Note that in such datasets we have the outcome of probes that turn out to create edges, but not of probes with nonedge outcomes. To overcome this issue, we select nonedge probe examples by sampling nodes uniformly at random. Since the graph is quite sparse, with high probability the two nodes are unrelated and would have formed a nonedge had we probed in them. 

In these datasets we will evaluate tasks that do not necessarily fulfill all assumptions needed for our application of causal lifting. Apart from fulfilling identifier exchangeability (cf.\ \Cref{ass:id-exc}), it is unclear whether \Cref{ass:time-igno,ass:time-exc,ass:link,def:y-int} hold (\textbf{Q3}) in this application. For instance, in LFM a user listening to a song from an artist might influence their choice to listen to another song of the same artist ---possibly violating \Cref{def:y-int} in the experiments. Then, given this scenario, how do (strictly) positional node embeddings and structural (node and pairwise) embeddings perform? 
The results in \Cref{fig:electronics,fig:music} show that even under these assumption violations, {\em pairwise embedding methods consistently outperform node embedding methods} (the results use the Hits@500 and Hits@50 metrics as described in \Cref{sec:res-kg}). In particular, testing pairs of users unseen in training (probes) makes (strictly) positional node embeddings have poor performance in both tasks. 
This observed behavior can be attributed to the predictor needing to use symmetries between males and females for knowledge transfer in this task, rather than relying solely on a node's positional embedding.
Interestingly, in the AE dataset (\Cref{fig:electronics}) we see that the GCN structural node embedding, although worse than structural pairwise embeddings, still achieves reasonable performance, but its performance is poor in the LFM dataset (cf.\ \Cref{fig:music}), a result that can be attributed to a higher bias of structural node embedding (\Cref{thm:bias}) in the LFM dataset.
Overall, the results clearly show that structural pairwise embeddings are superior for these tasks than node embeddings.

\paragraph{Baselines} Note that in user-item recommendation tasks the graph is bipartite. Thus, Neo-GNN, a method that considers the one-hop neighborhood of a pair would be equivalent to a standard GNN model. Thus, we choose not to evaluate it here.

\begin{figure*}[ht]
\centering
\begin{minipage}[t]{.45\textwidth}
\centering
\includegraphics[width=1\linewidth]{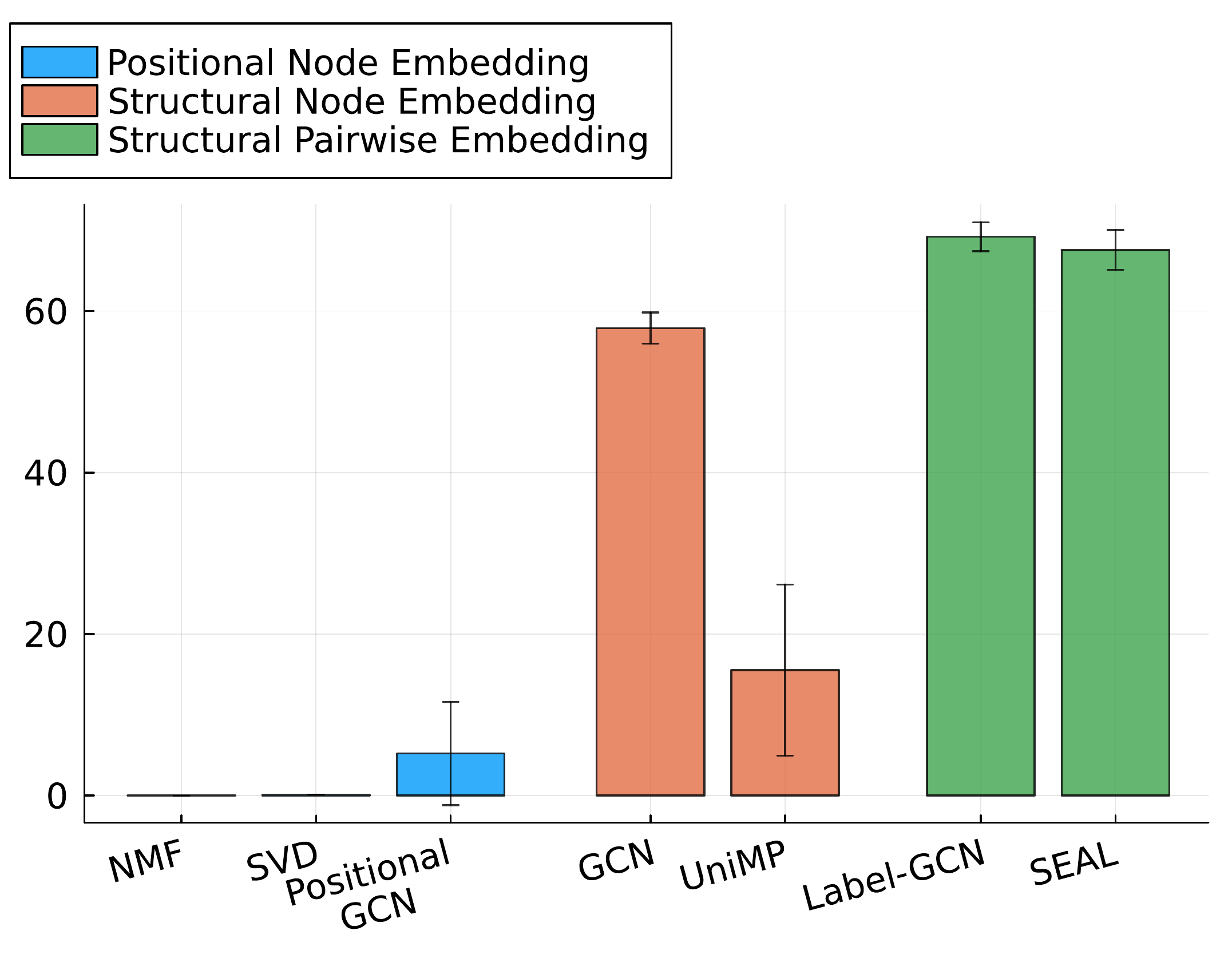}
\caption{Results for the AE dataset. We present the average and the standard deviation over five runs for the Hits@500 metric (out of 6,138 potential edges) in each method. Note that the larger $\uparrow$ the better the performance.}
\label{fig:electronics}
\end{minipage}\qquad
\begin{minipage}[t]{.45\textwidth}
\centering
\includegraphics[width=1\linewidth]{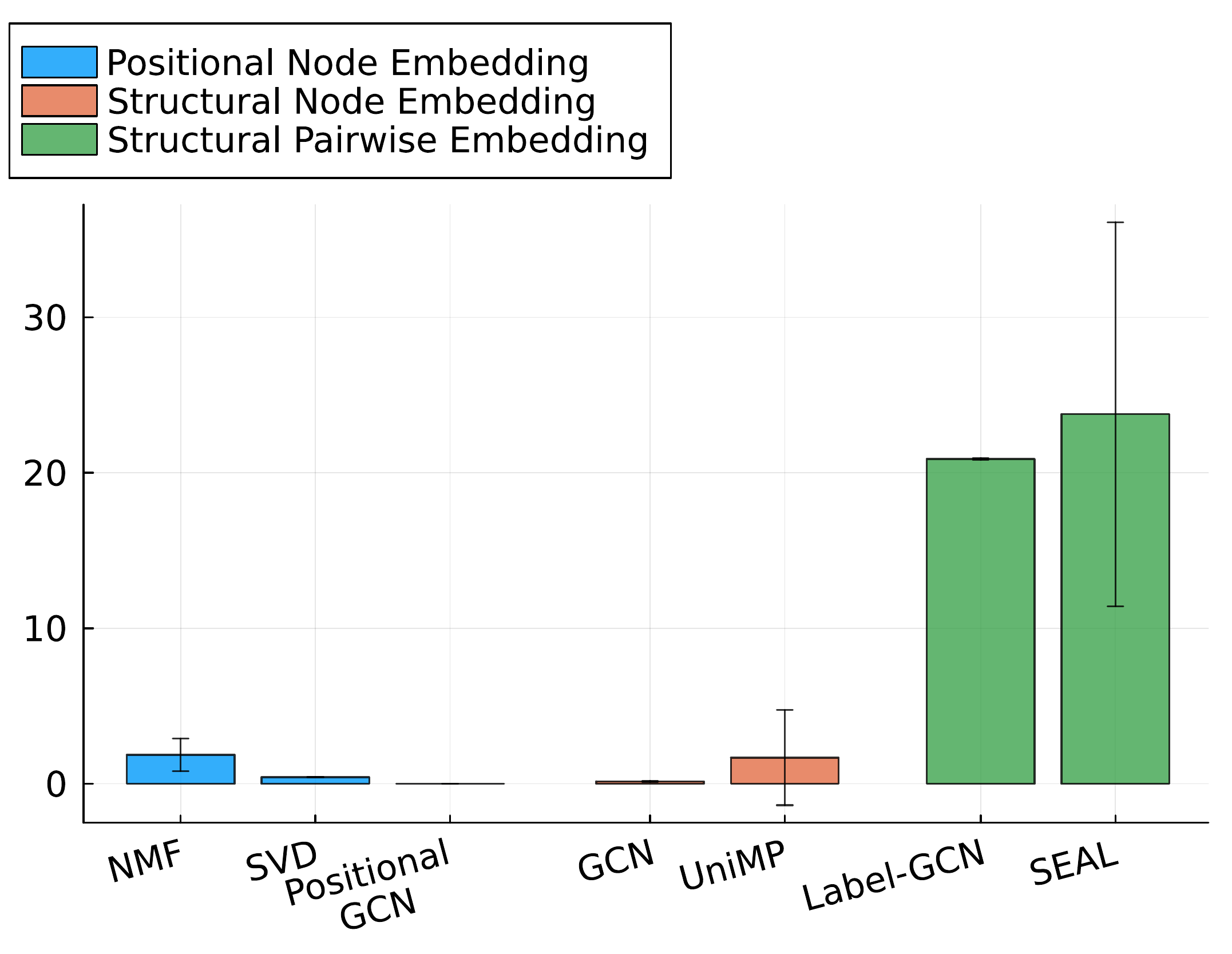}
\caption{Results for the LFM dataset. We present the average and the standard deviation over five runs for the Hits@50 metric (out of 4,256 potential edges)  in each method. Note that the larger $\uparrow$ the better the performance.}
\label{fig:music}
\end{minipage}
\end{figure*}





\section{Conclusions}
In this work we have shown how invariances can play a key role in identifying counterfactual queries with interventional data. Classical back-door adjustments for link prediction rely on over-simplifying causal independence assumptions (DAG-based constructions). Here, we take a different route: Consider a universal class of causal models and define causal mechanism invariances. By doing so, we show the SCM acquires a property that we denote {\em causal lifting}, which enables performing symmetry-based adjustments, \textit{i.e.}, use orbits instead of covariates.

By approaching causal link prediction with symmetry-based adjustments, we are able to consider a wide class of \emph{possibly path-dependent} causal models. To the best of our knowledge, our work is the first to consider such a setting and, moreover, to show the importance of invariant pairwise embeddings as estimators. We hope to shed light into this (until now) overlooked area of invariant representation learning and to incorporate symmetry-based adjustments in future causal literature works.

\section{Acknowledgments}
This work was funded in part by the National Science Foundation (NSF) awards CAREER IIS-1943364, CCF-1918483, and CNS-2212160 and an Amazon Research Award. L.\ Cotta is funded in part, by a postdoctoral fellowship provided by the Province of Ontario, the Government of Canada through CIFAR, and companies sponsoring the Vector Institute. Any opinions, findings and conclusions or recommendations expressed in this
material are those of the authors and do not necessarily reflect the views of the sponsors. 

\enlargethispage{20pt}



\bibliographystyle{apalike}
\bibliography{refs}

\appendix
\newpage
\appendix

\section{Notation and Background}
\label{sec:background}
\paragraph{Graph notation} We consider a graph $G$ with adjacency $a \in \domA^{n \times n}$, where each entry $a_{ij} \in \domA$ belongs to an arbitrary domain $\domA$ with at least two elements. Without loss of generality, we consider the node set $V = [n]:= \{1,\ldots,n\}$. In simple unattributed graphs $a$ is a binary matrix, while for general attributed (multi)graphs $a$ can be seen as a tensor, where its third mode encodes edge and node attributes\footnote{The node attribute of node $i$ can be encoded in its self-loop entry $a_{ii}$.}. Note that following this definition $a$ completely defines $G$. Since our exposition is mostly agnostic to its third mode, unless otherwise stated we consider $a \in \domA^{n \times n}$. Further, we denote the value representing a non-link as $\mathbf{0} \in \domA$. Finally, we use capital letters to denote the corresponding random variable of an observation, \textit{e.g.},\ $A$ is a random variable of $a$.
As such, if $a \in \domA^{n \times n}$ was produced by some mechanism that contains some intrinsic randomness (even noise), then $A$ describes the other possible outcomes with their respective probabilities.
\\~\\
\paragraph{Symmetry definitions} We denote the symmetric group of $[n]$ by $\S_n$, \textit{i.e.}, the set of all permutations of $\{1,\ldots,n\}$. Further, we let $\pi \cdot i$ be the mapping of $i \in [n]$ in permutation $\pi \in \S_n$ and  $\pi^{-1} \cdot i$ its inverse, \textit{i.e.}, the element that $\pi$ maps to $i$. Finally, we let the action of $\pi$ in $a$ be denoted by $ \pi \cdot a $, \textit{i.e.}, $a_{ij} = (\pi \cdot a)_{\pi \cdot i\pi \cdot j}$. Note that $ \pi \cdot a $ permutes the rows and columns of $a$ (or first two modes if $a$ is a tensor) according to the permutation $\pi$. We say that $G \iso H$ are isomorphic graphs if there exists a permutation $\pi \in \S_n$ such that the adjacency of $H$ is equal to $\pi \cdot a$.

Apart from isomorphism between graphs, we now define isomorphism between nodes and node pairs in the same graph. For that, we need to first define the automorphism group of a graph $G$: $\aut{G} \colon \{ \pi \in \S_n \colon \pi \cdot a = a \}$. Note that $\aut{G}$ defines the set of permutations that map the graph $G$ to itself. Now,  we say nodes $i$ and $j$ in $G$ are isomorphic\footnote{In graph theory we also often say $i$ is similar to $j$.} $i \iso j$ if there exists a permutation 
$\pi \in \aut{G}$ where $j = \pi \cdot i$. Similarly, two node pairs $\ij,\uv$ in $G$ are isomorphic $\ij \iso \uv$ if there exists a permutation $\pi \in \aut{G}$ where $\uv = \pi \cdot (i, j)$. Finally, we denote by $\orb{i}: \{ j \in V \colon j \iso i \}$ the orbit of node $i$ in $G$ and by $\orb{ij}: \{ \uv \in V^2 \colon \uv \iso \ij \}$ the orbit of the node pair $\ij$ in $G$. Overall, the orbit of a node (or node pair) defines the set of nodes (or node pairs) isomorphic to it (including itself). Note that the orbit of a node or of a node pair in $G$ contains elements other than themselves only if $\setsize{\aut{G}} > 1$. We call graphs that satisfy such property \emph{symmetric graphs}.

Now that we have defined symmetry between graphs, nodes and node pairs, we can turn our attention to a symmetry in the graph distribution. The random variable of a graph is finite (jointly) exchangeable if any two isomorphic graphs are generated with the same probability. That is, node identifiers, that are what distinguish isomorphic graphs, are not relevant to the task at hand. In \Cref{def:exc} we overload this definition and define finite exchangeable graphs as graphs that come from such a distribution. Throughout this work we consider finite exchangeable graphs as input data to our problem, \textit{i.e.}, $\G{t_0}$ in \Cref{eq:cf,eq:cf-multi}. Although this is an assumption about the data distribution, finite exchangeability is also what distinguishes graph from sequence data. If node identifiers are not arbitrary, we can treat the graph's adjacency as a long sequence of edges $(a_{11}, a_{12}, \ldots)$. Thus, finite exchangeability can be seen as assuming that our input data is a graph.

\begin{definition}[Finite exchangeable graphs] \label{def:exc} We say that an adjacency matrix random variable $A$ is finite (jointly) exchangeable if
$$ P(A=a) = P(A=\pi \cdot a ), \quad \forall \pi \in \S_n, a \in \domA^{n \times n} .$$
We overload finite (joint) exchangeability to ease notation and say that a graph $G$ is finite exchangeable if its adjacency matrix random variable $A$ is finite (jointly) exchangeable.
\end{definition}

\paragraph{Causality definitions} 

Any causal query can be seen as an inquiry about the underlying Structural Causal Model (SCM) of the task~\cite{elias2020PCH}. An SCM is a mathematical description of the mechanisms behind the data generating process of interest. We start by formally defining what an SCM is below.

\begin{definition}[Structural Causal Model (SCM)~\cite{elias2020PCH}]\label{def:scm}
A Structural Causal Model (SCM) is a 4-tuple $\m = ( \U, \V, \F, P(\U) ) $, where

\begin{itemize}

    \item $\U$ is the set of external noise random variables ---also called background or exogenous variables--- which are generated by (unknown) mechanisms outside the model;
    
    \item $\V = \{ \V_1, \V_2 \dots, \V_q \}$ is the set of endogenous random variables, which are generated by variables inside the model ($\V \cup \U$);
    
    \item $\F = \{ f_1, f_2 \dots, f_q\} $ is a set of functions defining a mapping between $\U$ and $\V$ with each $f_i$ mapping from the domains of $\Parents_i$ and $\U_{\V_{i}}$ to $\V_i$ with $\U_{\V_{i}} \subseteq \U, \Parents_i \subseteq \V \backslash \V_i $. In practice, each $f_i \in \F$ is a mechanism that outputs the $i\th$ endogenous variable $\V_i$ given its exogenous variables $\U_{\V_i}$ and its endogenous parent variables $\Parents_i$, \textit{i.e.}
    $$ \V_i = f_i( \Parents_i, \U_{\V_{i}} ) ; $$
    and
    
    \item $P(\EuScript{U})$ is a probability distribution over $\EuScript{U}$.
    
\end{itemize}

\end{definition}

Note from above that an SCM $\m$ is a data generating process for $\V$ and as such induces an (associational) probability distribution $P(\V = v) = \sum_{u} \prod_{i \mid \V_i \in \V } P(v_i \mid \parents_i u_{\V_{i}} ) P(u)$. We can then define the interventional distribution $P( \V(\V_i = v_i) )$ as the distribution induced by an altered model $\m_{\V_i=v_i}$, where $\m_{\V_i=v_i}$ is exactly as $\m$ with the difference that $\V_i$ is replaced by a constant value $v_i$ ---we often denote such distribution by $P_{\V_i=v_i}(\V)$. Now, the counterfactual distribution $P( \V(\V_i = v_i) \mid \V = v' )$ can be defined as the distribution induced by the altered model $\m_{\V_i=v_i}\mid v'$, where $\m_{\V_i=v_i} \mid v'$ is exactly as $\m_{\V_i=v_i}$ with the difference that $P(\U)$ is replaced by $P(\U \mid \V = v')$, \textit{i.e.}, the exogenous variables distribution changes based on our observed evidence $v'$ ---we often denote such distribution by $P_{\V_i=v_i \mid v'}(\V)$. Finally, note that the evidence $v'$ can also be the outcome of an interventional distribution, as we have in \Cref{eq:cf}. In this case, what changes is how we update the exogenous variable distribution, \textit{i.e.}, $P(\U)$ is replaced by $P(\U \mid \V(\V_i = v_i) = v)$. From these definitions, we now have the tools to evaluate interventional and counterfactual quantities, as required by our task (cf.\ \Cref{eq:cf,eq:cf-multi}).

\section{Link Prediction through self-supervision} \label{sec:selfsuper}
In what follows we formalize the \textbf{observational} task of predicting whether a node pair $\ij \in \Vt{t_0} \times \Vt{t_0}$ forms a link in $\G{t_0}$ in the context of Graph Neural Networks and Matrix Factorization. We will show how, for both methods, this task can be viewed as a self-supervised learning one. Let $\A{t_0}{-ij}$ and $\a{t_0}{-ij}$ be respectively the random variable $\A{t_0}{}$ and the adjacency matrix $\a{t_0}{}$ without the pair $(i,j) \in \Vt{t_0} \times \Vt{t_0}$. 
Note how this is different from assuming $(i,j)$ is a nonedge (\textit{i.e.}, that there is no edge between nodes $i$ and $j$).
The adjacency matrix $\a{t_0}{-ij}$ has no information about $\a{t_0}{ij}$, including whether it is an edge or a nonedge.
\\~\\
\paragraph{Relation to Graph Neural Networks (GNNs)} Existing works in link prediction using GNN node embeddings make use of the graph structure to predict missing links~\cite{hu2020ogb}. To predict $\A{t_0}{IJ}$, the embeddings of $I$ and $J$ are obtained by applying a GNN over the training graph $\G{t_0}$ with adjacency $\a{t_0}{}$. Thus, instead of learning the self-supervision task
$P(\A{t_0}{IJ}= \a{t_0}{IJ} \mid \A{t_0}{-IJ} = \a{t_0}{-IJ} )$,
previous works on GNNs for link prediction are learning
$P(\A{t_0}{IJ}= \a{t_0}{IJ} \mid \A{t_0}{} = \a{t_0}{} )$.
Note, however, that $\A{t_0}{IJ}$ is contained in $\A{t_0}{}$, hence this is not a sound statistical learning task. Instead, we should remove the information about $\A{t_0}{IJ}$ and thus use the self-supervised objective. In practice, we note that the information about $\A{t_0}{IJ}$ is not directly encoded in GNN embeddings, \textit{i.e.}, it is used implicitly via the message-passing scheme. Thus, most training procedures are not impacted when using $\a{t_0}{}$ instead of $\a{t_0}{-IJ}$ to compute GNN embeddings.
\\~\\
\paragraph{Relation to matrix factorization} Matrix (or tensor) factorization methods are one of the most used tools for link prediction and clustering tasks in graphs. For undirected unattribued graphs, such methods learn a matrix of embeddings $\phi \in \mathbb{R}^{ n \times d}$, where each row $\phi_i$ is the embedding of node $i$. For directed graphs, $\phi$ is a pair of such matrices, encoding the source and the target embedding of each node. Finally, heterogeneous graphs also add edge type embeddings in $\phi$. How do these methods learn $\phi$?

Matrix factorization-based models learn the joint distribution
$P(\A{t_0}{} = \a{t_0}{} \mid \Phi = \phi )$. The key assumption here is edge (conditional) independence, that is
\begin{equation}
    P(\A{t_0}{} = \a{t_0}{} \mid \Phi = \phi ) = \prod_{\ij \in \Vt{t_0} \times \Vt{t_0}} P(\A{t_0}{ij} = \a{t_0}{ij} \mid \Phi = \phi ).\label{eq:mf}
\end{equation}
These methods learn $\phi^*$ by maximizing the log-likelihood of \Cref{eq:mf}, a problem that can be written as
\begin{equation}
   \phi^* = \argmax_{\phi} \mathbb{E}_{\IJ}\big[ \log P( \A{t_0}{IJ}=\a{t_0}{IJ} \mid \Phi = \phi ) \big],\label{eq:mle-mf}
\end{equation}
with $\IJ \sim \text{Uniform}(\Vt{t_0} \times \Vt{t_0})$.

Finally, what is the relationship between \Cref{eq:mf} and the self-supervised learning task
$P(\A{t_0}{IJ} = \a{t_0}{IJ} \mid  \A{t_0}{-IJ} = \a{t_0}{-IJ} )$
? First, note that we can rewrite the self-supervised learning task when learning parameters $\phi$ as
\begin{equation}
    P(\A{t_0}{IJ} = \a{t_0}{IJ} \mid  \A{t_0}{-IJ} = \a{t_0}{-IJ} ) = \int_{ \phi } P( \A{t_0}{IJ} = \a{t_0}{IJ} \mid \Phi = \phi ) P(\Phi = \phi \mid \A{t_0}{-IJ} = \a{t_0}{-IJ}  ) \,d\phi,\label{eq:ssl-mf}
\end{equation}
which can be expressed as the log-likelihood
\begin{equation}
    \phi^* = \argmax_{\phi} \mathbb{E}_{\IJ}\big[ \log \mathbb{E}_{\Phi = \phi \mid \a{t_0}{-IJ} }[ P( \A{t_0}{IJ} = \a{t_0}{IJ} \mid \Phi = \phi ) ] \big],\label{eq:mle-mf-ssl}
\end{equation}
with $(I,J) \sim \text{Uniform}(\Vt{t_0} \times \Vt{t_0})$. We can see how the difference between the training objectives \Cref{eq:mle-mf-ssl} and \Cref{eq:mle-mf} is the expectation over the prior
$P(\Phi = \phi \mid \A{t_0}{-IJ} = \a{t_0}{-IJ} )$.
Finally, if we assume a flat prior $P(\Phi = \phi \mid \A{t_0}{-ij} = \a{t_0}{-ij} )$ over all $(i,j) \in \Vt{t_0} \times \Vt{t_0}$, the two objectives become the same. Thus, we can see how matrix factorization methods are simply ignoring the rest of the observed graph as an input signal to predict $ \A{t_0}{IJ}$.

\section{A Universal Family of Causal Models for Graphs}\label{sec:scm}

The conditions needed to perform counterfactual lifting in our task (\Cref{eq:cf,eq:cf-multi}) are with respect to its underlying Structural Causal Model (SCM), \textit{i.e.}, the graph's causal generating process. To this end, here we present a universal family of SCMs for graphs, where we are able to define our task and derive sufficient conditions for counterfactual lifting and other identification results. Without loss of generality, we show how to generate $\A{t_0}{}$ ---which completely defines the observed graph $\G{t_0}$. Note that unlike in usual causal models, we observe finite exchangeable graph data and thus our SCM needs to be finite exchangeable with respect to $\A{t_0}{}$. A finite (jointly) exchangeable SCM generates an observation $\a{t_0}{}$ with the same probability as $ \pi \cdot \a{t_0}{}$ for any permutation $\pi \in \S_n$, \textit{i.e.}, any two isomorphic graphs are generated with the same probability. Since an SCM generates every random variable as a function of its parents, causal models have an intrinsic (partial) ordering, \textit{i.e.}, parents must be generated before their children. Thus, designing an SCM for our task implies generating (partially) ordered sequences of random variables. How can we go from (partially) ordered sequences to finite exchangeable random variables? Next, we define a family of SCMs with such property. Intuitively, our causal models achieve finite exchangeability by randomly reassigning node identifiers. Later in \Cref{thm:main0} (i)'s proof we formally show how this family of SCMs is indeed finite exchangeable.

We denote the proposed family of SCMs by $\M$. Each SCM $\m \in \M$ has a specific set of mechanisms $\F$ and exogenous variables distribution $P(\EuScript{U})$. To generate the observed graph, $\m$ has two stages: the data generating process (\Cref{def:g-scm}) and the adjacency observation process (\Cref{def:g-obs}). The data generating process outputs two infinite-size sequences $( \X{t} )_{t=1}^{\infty}, ( \E{t} )_{t=1}^{\infty}$. Each $\X{t}$ has the same support as the entries of $\A{t_0}{}$ ($\supp(\X{t}) = \domA$), while each $\E{t}$ defines a node pair ($\supp(\E{t}) = \Z^+ \times \Z^+$). Since it generates $( \X{t} )_{t=1}^{\infty}$ and $( \E{t} )_{t=1}^{\infty}$ sequentially, the SCM has a notion of time, where at time $t$ it generates the value $\X{t}$ of the interaction (or its absence) between $\E{t}$. At a given observation time $t_0$, the adjacency observation process in \Cref{def:g-obs} generates $\A{t_0}{}$ by using $( \E{t} )_{t=1}^{t_0}$ to map $( \X{t} )_{t=1}^{t_0}$ to a matrix (or tensor) and jointly shuffling its rows and columns (or its first two modes).

\begin{definition}[Data generating process]\label{def:g-scm}

We start by defining $\E{1} = (1,1)$. Then, at time $t$, for two mechanisms $\fx{t}$ and $\fe{t}$ we define the recurrence relations
\begin{align*}
&\E{t} = \fe{t}\Big( (\E{r})_{r=1}^{t-1},  (\X{r})_{r=1}^{t-1}, \Ue{t} \Big),
\\
&\X{t} =
\begin{cases}
    \fx{t}\Big( \Ux{t} \Big) , & \text{if } t = 1,\\
    \fx{t}\Big( (\E{r})_{r=1}^{t},  (\X{r})_{r=1}^{t-1}, \Ux{t} \Big),             & \text{otherwise,}
\end{cases}
\end{align*}
where $\Ux{t} \sim P(\Ux{t} \mid \E{t}, (\Ux{m})_{m=1}^{t-1} )$ are sampled (given $\E{t}$ and all previous exogenous variables) and $\Ue{t} \sim P(\Ue{t} \mid (\Ue{m})_{m=1}^{t-1} )$ are exogenous variables sampled at time $t$. Note that the distribution of $\Ux{t}$ takes $\E{t}$ as a parameter and thus the exogenous variable at time $t$ is dependent not only on the previously generated exogenous variables, but also on the pair being generated\footnote{Note that in fact $\Ux{t}$ is not an exogenous variable, since it is modeled inside our SCM, but we refer to it as such for the sake of simplicity.}. Finally,

\begin{align*}
&\fx{1} \colon \supp( \Ux{1} ) \to \domA \\
&\fx{t} \colon \bigcup_{r=1}^t \Big(\big( \Z^+ \times \Z^+ \big)^r \times \big( \domA \big)^{r-1}\Big) \!\! \times \supp\Big( \Ux{t} \Big) \to \domA, \: t > 1,
\end{align*}
are measurable maps and

\begin{align*}
\fe{t} \colon  \Big( \Z^+ \times \Z^+ \Big)^t \times \Big( \domA \Big)^{t} \!\! \times \supp\Big( \Ue{t} \Big) \to \Z^+ \times \Z^+
\end{align*}
is a measurable map such that $\max(\E{t}) \leq \max\Big( \big(\max(\E{r})\big)_{r=1}^{t-1} \Big) + 1$, \textit{i.e.}, the pair to be generated at time $t$ can contain up to one node that has not yet been generated until time $t-1$.

\end{definition}

The only mechanism restriction is a simple rule on $\fe{t}$: the pair to be generated can contain only up to one unseen node\footnote{An unseen node in the SCM is a node which has not appeared in any generated pairs until time $t$.} and its identifier is the smallest positive integer not seen yet. This way, if $n$ nodes have appeared in interactions at any point of time, their identifiers will be $[n]:=\{1,\ldots,n\}$. Note, however, that we do not observe the sequences that \Cref{def:g-scm} outputs. Instead, we partially observe them as $\G{t_0}$, \textit{i.e.}, a graph at some given point of time $t_0$. The process generating our observed data (\Cref{def:g-obs}) takes as input $( \E{t} )_{t=1}^{t_0},( \X{t} )_{t=1}^{t_0}$ from \Cref{def:g-scm} and outputs $\A{t_0}{}$. It samples a permutation $\pi$ of the node identifiers and maps the sequence to the observed matrix (or tensor) $\A{t_0}{}$. Thus, $t_0$ and $\pi$ are the observation parameters ---when we observe the graph and what we see as arbitrary node identifiers. Note how in \Cref{def:g-scm} a pair $\ij$ can be generated multiple times by $\fe{t}$. Therefore, the observation process only considers the most recent interaction occurred at time $t^\circ$ (cf.\ \Cref{def:g-obs}).

\begin{definition}[Adjacency observation process]\label{def:g-obs}
Let $\n{t_0}:= \max\Big( \big(\max(\E{t})\big)_{t=1}^{t_0} \Big) $ be the number of different nodes appearing in pairs generated by $\m$ until time $t_0$. Then, we can generate $\A{t_0}{}$ by first sampling a permutation of the node identifiers
$$ \pi \sim \text{Uniform}(\S_{\n{t_0}}) $$
and then assigning

\[
  \A{t_0}{ij} =
    \begin{cases}
    \mathbf{0}, \text{if } (\pi^{-1} \cdot i,\pi^{-1} \cdot j) \notin \{ \E{t}\}_{t=1}^{t_0},\\
      \X{t^\circ_{ij}}\text{, } t^\circ_{ij} := \max\big( \{t \colon \E{t} = (\pi^{-1} \cdot i,\pi^{-1} \cdot j), 1 \leq t \leq t_0 \} \big),  \text{otherwise.}
    \end{cases}
\]

\paragraph{Undirected graphs} Note that if $\G{t_0}$ is an undirected graph we have that all mechanisms $\fx{t}, t \geq 1$ are invariant to the order of the input pairs, \textit{i.e.}, each $\E{t}$ is treated as a set rather than a tuple. Finally, we have an extra step here setting $\A{t_0}{ji}$ to $\A{t_0}{ij}$ if $(\pi^{-1}_i,\pi^{-1}_j) \in \{ \E{t} \}_{t=1}^{t_0}$. In the case that $(j,i)$ was also generated by the SCM, \textit{i.e.},  $(\pi^{-1} \cdot j,\pi^{-1} \cdot  i) \in \{ \E{t} \}_{t=1}^{t_0}$, we set $\A{t_0}{ji}$ to $\A{t_0}{ij}$ only if $t^\circ_{ij} > t^\circ_{ji}$.
\end{definition}

We thus define a graph SCM $\m \in \M$ as the coupling of \Cref{def:g-scm,def:g-obs}. At this point, it is worth taking a moment to understand $\m$ ---see \Cref{fig:scm-ex} for a sample generation of a graph with four nodes. The data generating process (\Cref{def:g-scm}) is an underlying evolving process, deciding at each time step which node pairs will be assigned a link (and its value) or a non-link. We do not observe the execution of \Cref{def:g-scm} or the observation parameter $\pi$, \textit{i.e.}, we are not aware of nodes' original identifiers. We only observe the graph's adjacency $\A{t_0}{}$. Finally, note that a nonedge entry in $\A{t_0}{}$ does not imply that the pair was generated as a nonedge. It can also be that it was not yet generated by the causal model. What does this mean in practice? For instance, consider a streaming platform where we do not observe an interaction between user $i$ and movie $j$. In this case, either user $i$ does not know movie $j$ is in the platform or $i$ has actively made the decision to not watch $j$. The underlying causal model explaining the lack of an interaction is hidden from us. Later we show how this notion between nonedges and pairs yet not generated by the SCM is central to identify our counterfactual task.
\\~\\
\paragraph{Generating temporal and dynamic graphs} Up until now, we have referred to our observed graph $\G{t_0}$ as a static graph. That is, a graph where all edges are observed together at once with no temporal information. However, note that although there is no explicit notion of time in $\G{t_0}$, its underlying data generating process (\Cref{def:g-scm}) is temporal. Because of such feature, our model can actually represent not only static, but also temporal and dynamic graphs. In temporal graphs we observe the edge creation time. In this case, our model can output in $\A{t_0}{ij}$ the most recent time step $t^\circ$ (cf.\ \Cref{def:g-obs}) together with the interaction value. In dynamic graphs, an edge can be created multiple times, disappear, take new values and so on. In this setting we would then not only observe the time of an interaction, but also all past interactions. Note that although $\fe{t}$ can generate an interaction between the same pair multiple times, the observation model only outputs the most recent one at $t^\circ$. On the other hand, since $\fx{t}$ takes as input all previous interactions, it can copy the history of interactions between $i$ and $j$ to $\A{t_0}{ij}$ (together with their time steps) making $\A{t_0}{}$ a sample of a dynamic graph.
\\~\\
\paragraph{Exchangeability and expressive power} After designing our family of SCMs $\M$, we turn to the two central theoretical questions around it: i.\ Is $\M$ finite exchangeable? ii.\ How expressive is $\M$? In \Cref{thm:main0} we show that i.\ any SCM $\m \in \M$ generates any two isomorphic graphs with the same probability, hence the entire family $\M$ is finite exchangeable and ii.\ there exists an SCM $\m \in \M$ that generates every pair of non-isomorphic graphs (with countable domain) with different probabilities. Ultimately, \Cref{thm:main0}(i, ii) proves that $\M$ is both a finite exchangeable and universal family of graph SCMs. Next, we present the proofs of items i.\ and ii.\ from \Cref{thm:main0}.

\subsection{Proof of \texorpdfstring{\Cref{thm:main0}}{Theorem 1}}
\zeromainthm*
\begin{proof}
~\\
\noindent i.\ Let $\m(u,\pi,t_0)$ be the output of an SCM $\m \in \M$ with input exogenous variables $u \in \supp(\U)$, permutation $\pi$ and observation time $t_0$. Note that given these three variables assignments $\m$ is a deterministic mapping to an observed graph $a$.
We need to prove that any $\m(u,\pi,t_0)$ gives isomorphic graphs the same probability.
Now, the probability of any model $\m \in \M$ generating graphs $a,a' \in \domA$ is
\begin{equation}\label{eq:prob1}
P( \A{t_0}{} = a \mid \m ) = \sum_{t_0=1}^{\infty} \frac{1}{\setsize{\S_n}} \sum_{\pi \in \S_n } \int_{u \in \supp(\U) } \mathds{1}(\m(u,\pi,t_0) = a) P(\U=u)du,
\end{equation}
    and
\begin{equation}\label{eq:prob2}
P( \A{t_0}{}  = a' \mid \m) = \sum_{t_0=1}^{\infty} \frac{1}{\setsize{\S_n}} \sum_{\pi \in \S_n } \int_{u \in \supp(\U) } \mathds{1}(\m(u,\pi,t_0) = a') P(\U=u)du
\end{equation}
respectively. Note that since $a$ and $a'$ are isomorphic, we can rewrite $a'$ as $\pi' \cdot a$ for some $\pi' \in \S_n$. Now, since $\S_n$ is a group, for every $\pi \in \S_n$ there exists another $\pi^\dagger$ such that $\pi = \pi^\dagger \circ \pi'$. Thus, for every $\pi \in \S_n$ we can define another permutation $\pi^\star := \pi^\dagger \circ \pi'$ giving us $ \pi \cdot a = \pi^\star \cdot a'$. Thus, whenever $\m(u,\pi,t_0) = a$ there exists a $\pi^\star$ such that $\m(u,\pi^\star,t_0) = a'$. As a result of such bijection, since the other terms match in \Cref{eq:prob1,eq:prob2}, we have that $ P( \A{t_0}{}  = a \mid \m)  = P( \A{t_0}{}  = a' \mid \m) $.
\\~\\
\noindent ii.\ We now show that for any finite graph (finite jointly exchangeable random array) $A$ there exists an SCM $\m \in \M$ and a time $t_0 \in \R$ such that $A \stackrel{d}{=} \A{t_0}{}$, where $\A{t_0}{}$ is the graph generated by $\m$ at time $t_0$.
Since $\domA$ is countable, let's enumerate all graphs $A_1,A_2,\ldots \in \domA$ such that $\domA = \cup_{i=1}^\infty \{A_i\}$, noting that graph with distinct indices may be isomorphic. Now, we partition the interval $[0,1)$ into intervals $U_j = \big[\sum_{i=1}^{j} P(A_i),\sum_{i=1}^{j+1} P(A_i) \big)$, where we define $P(A_0) = 0$, noting that by the total law of probability $[0,1) = \cup_{j=1}^\infty U_j$.
Following \Cref{def:g-scm} we can 
construct arbitrary functions $\fx{t},\fe{t},  t=1,\ldots,t_0$, for SCM $\m \in \M$ such that with probability $U_j$ we generate deterministic sequences $\Ux{1},\ldots,\Ux{t_0}$ and $\Ue{1},\ldots,\Ue{t_0}$ that will exactly generate the edges of $A_j$ (in the same order).
The exchangeability of $\A{t_0}{}$ comes from \Cref{def:g-scm}, which permutes the node ids at the observation time $t_0$.

\end{proof}

From the above we highlight the result in (ii). Our family of SCMs is not trivial nor attached to a restricted set of distributions. There exists a causal model in it that is able to generate all non-isomorphic graphs with different probabilities, \textit{i.e.}, it can distinguish them and thus we call it an universal class of causal models for graphs. This result comes at the cost of a (possibly\footnote{We say possibly since in practice the mechanisms of the true and unknown underlying SCM might be ignoring parts of the input and thus removing such dependencies.}) prohibitively large amount of causal dependencies between variables (see \Cref{fig:scm-ex} for an illustration). Since in practice it is hard to have assumptions removing them, we maintain such structural dependencies and still identify and estimate \Cref{eq:cf,eq:cf-multi} by posing invariance restrictions to the causal mechanisms and some experimental (probe) conditions.

\section{Interventional Lifting for Link Prediction}\label{sec:link-lift}

Now that we have formalized or class of SCMs, we can precisely describe \Cref{ass:time-igno,ass:time-exc,ass:link,ass:id-exc} and \Cref{def:y-int}.

\setcounter{assumption}{0}
\begin{assumption}[Time gap ignorability]
 We say that an SCM $\m \in \M$ satisfies time gap ignorability for observation time $t_0$ and first experiment time $t_1$ if
\[ \fx{t_1}\Big( (\E{t})_{t=1}^{t_1}, (\X{t})_{t=1}^{t_1-1}, \Ux{t_1} \Big) =     \fx{t_0+1}\Big( (\E{t})_{t=1}^{t_0+1}, (\X{t})_{t=1}^{t_0} , \Ux{t_0+1}  \Big).
\]
\end{assumption}

\begin{assumption}[Time exchangeability]
 We say that an SCM $\m \in \M$ satisfies time exchangeability for observation time if
 \begin{align*} \fx{t_0+1}\Big( (\E{t})_{t=1}^{t_0+1}, (\X{t})_{t=1}^{t_0}, \Ux{t_0+1} \Big) =
  \fx{t_0+1}\Big( (\E{ \pi_{t} })_{t=1}^{t_0+1}, (\X{ \pi_{t} })_{t=1}^{t_0} , \Ux{t_0+1} \Big), \!\! \\ \quad \!\! \forall \pi \in \S_{t_0+1} \colon \pi_{t_0+1} = t_0+1.
 \end{align*}
\end{assumption}

\begin{assumption}[Non-link ignorability]
Let $\ecal{N}^{(t_0)} := \big\{ t \in [t_0] \colon \X{t} = \mathbf{0} \big\}$ be the set of time steps (until $t_0$) where non-links were created. Then, we say that an SCM $\m \in \M$ satisfies non-link ignorability for observation time $t_0$ if its mechanism $\fx{t_0+1}$ is invariant to the removal of non-links from the input sequence, \textit{i.e.},
\begin{align*}&\fx{t_0+1}\Big( (\E{t})_{t=1}^{t_0+1}, (\X{t})_{t=1}^{t_0}, \Ux{t_0+1} \Big) \\&=     \fx{t_0+1}\Big( \big((\E{t})_{t=1}^{t_0+1}\big)_{-\ecal{N}^{(t_0)}}, \big((\X{t})_{t=1}^{t_0}\big)_{-\ecal{N}^{(t_0)}} , \Ux{t_0+1} \Big).
\end{align*}
\end{assumption}

\begin{assumption}[Identifier exchangeability]
We say that an SCM $\m \in \M$ satisfies identifier exchangeability for observation time $t_0$ and first experiment time $t_0+1$ if 
\begin{align*}
\fx{t_0+1}\Big( (\E{t})_{t=1}^{t_0+1}, (\X{t})_{t=1}^{t_0}, \Ux{t_0+1} \Big) =
\fx{t_0+1}\Big( \big( (\pi_i,\pi_j) \colon (i,j) := \E{t}\big)_{t=1}^{t_0+1}, (\X{t})_{t=1}^{t_0} ,\Ux{t_0+1} \Big),\\ \quad \forall \pi \in \S_n.
\end{align*}
\end{assumption}

\setcounter{definition}{3}
\begin{definition}[Non-interfering probes]
We say that a sequence of probes in $M$ pairs $\big(\IJt{m}\big)_{m=1}^M$ is non-interfering if the following invariance holds 
\begin{align*}
\fx{t_m}\Big( \big((\E{t})_{t=1}^{t_0},(\IJt{m'})_{m'=1}^{m}\big), \big((\X{t})_{t=1}^{t_0},(\IJt{m'})_{m'=1}^{m-1}\big), \Ux{t_m} \Big) =\\
\fx{t_0+1}\Big( \big((\E{t})_{t=1}^{t_0}, \IJt{1}\big), (\X{t})_{t=1}^{t_0}, \Ux{t_0+1} \Big).
\end{align*}
\end{definition}

\subsection{Proof of \texorpdfstring{\Cref{thm:main}}{Theorem 2} }

\mainthm*

\begin{proof} ~\\
ii.\ We start by noting that from \Cref{def:g-scm}, we have that
\[
\Y{IJ} = \fx{t_1} \Big( (\X{t})_{t=1}^{t_1-1}, (\E{t})_{t=1}^{t_1}, \Ux{t_1} \Big),
\]
where $\E{t_1} = \IJ$ is the intervened random variable while $(\X{t})_{t=1}^{t_1-1} \sim P((\X{t})_{t=1}^{t_1-1} \mid \G{t_0})$ and $(\E{t})_{t=1}^{t_1-1} \sim  P((\E{t})_{t=1}^{t_1-1} \mid \G{t_0})$ are (random) sequences that must have generated $\G{t_0}$ until $t_0$. 
We start by considering \Cref{ass:time-igno}, which directly reduces the above problem to an intervention at the next time step $t_0+1$
\[  \Y{IJ} = \fx{t_1} \Big( (\X{t})_{t=1}^{t_0}, (\E{t})_{t=1}^{t_0+1}, \Ux{t_0+1} \Big), \]
where $\E{t_0+1} = \IJ$ is the intervened random variable while $(\X{t})_{t=1}^{t_0} \sim P((\X{t})_{t=1}^{t_0} \mid \G{t_0})$ and $(\E{t})_{t=1}^{t_0} \sim  P((\E{t})_{t=1}^{t_0} \mid \G{t_0})$ are (random) sequences that must have generated $\G{t_0}$. 

Further, from \Cref{thm:main}'s statement, we have that the exogenous variables sampled before time $t_0+1$ are independent from $\Ux{t_0+1}$ given $\IJ$, \textit{i.e.}, $P(\Ux{t_0+1} \mid (\Ux{t})_{t=1}^{t_0}, \IJ  ) = P( \Ux{t_0+1} \mid \IJ  )$. Thus, we will be assuming $\ecal{U}_{\IJ} = \Ux{t_0+1}$ to prove (iii).\ from \Cref{thm:main} it suffices to show that there exists a mechanism $h \colon \supp(\bW_{\ORB{IJ}}) \times \supp(\Ux{t_0+1}) \to \domA $ such that
\begin{equation}\label{eq:h}
\begin{split}
&h( \bW_{\ORB{ij}}, \ux{t_0+1} ) = \fx{t_0+1} \Big( (\x{t})_{t=1}^{t_0}, (\e{t})_{t=1}^{t_0+1}, \ux{t_0+1} \Big), \\ &\forall \x{t} \in \supp(\X{t}), \forall \e{t} \in \supp(\E{t}), t=1,\ldots,t_0, \forall \ux{t_0+1} \in \supp(\Ux{t_0+1}), \\ 
&\forall 1 \leq \n{t_0} \leq t_0, \forall \a{t_0}{} \in \domA^{\n{t_0} \times \n{t_0}}, \forall \ij \in \Vt{t_0} \times \Vt{t_0},\text{ with }\e{t_0+1} = (i,j).
\end{split}
\end{equation}
Let us now define the equivalence relation $\sim$ on the set of sequences $\XE := \supp\Big((\X{t})_{t=1}^{t_0}, (\E{t})_{t=1}^{t_0+1}\Big)$:
\begin{align*}
\begin{split}
\Big((\x{t})_{t=1}^{t_0},& (\e{t})_{t=1}^{t_0+1}\Big) \sim \Big( (\xp{t})_{t=1}^{t_0}, (\ep{t})_{t=1}^{t_0+1} \Big) \text{ if }  \exists \; \pi \in \S_{\n{t_0}} , \exists \; \pi^\dagger \in \S_{t_0+1} \colon \pi^\dagger \cdot (t_0+1) = t_0 + 1
\\ \text{such that } &
\Big(
\big((\x{t})_{t=1}^{t_0}\big)_{-\ecal{N}^{(t_0)}}, \big((\e{t})_{t=1}^{t_0+1}\big)_{-\ecal{N}^{(t_0)}}\Big) \\ &= \Big(\big(\big(\xp{\pi^\dagger \cdot t})_{t=1}^{t_0}\big)_{-\ecal{N}^{(t_0)}}, \big((\pi \cdot \ep{\pi^\dagger\cdot t})_{t=1}^{t_0+1}\big)_{-\ecal{N}^{(t_0)}} 
\Big),
\end{split}
\end{align*}
where $(\cdot)_{-\ecal{N}^{(t_0)}}$ removes the non-links from the sequence as in \Cref{ass:link}.
Let $[\cdot]$ denote the equivalence class of a sequence and $[\XE] := \Big\{\Big[\big((\x{t})_{t=1}^{t_0}, (\e{t})_{t=1}^{t_0+1}\big) \Big] \colon \big((\x{t})_{t=1}^{t_0}, (\e{t})_{t=1}^{t_0+1}\big) \in \XE \Big\}$ be the set of all equivalence classes in $\XE$. Then, it directly follows from the invariances in \Cref{ass:id-exc,ass:link,ass:time-exc} that there exists a function $g \colon [\XE] \to \domA$ such that, $\forall \; \big((\x{t})_{t=1}^{t_0}, (\e{t})_{t=1}^{t_0+1}\big) \in  \XE,$
\begin{align}\label{eq:g}
\begin{split}
g\Big( \Big[\big((\x{t})_{t=1}^{t_0}, (\e{t})_{t=1}^{t_0+1}\big) \Big], \Ux{t_0+1}\Big) = 
\fx{t_0+1} \Big( (\x{t})_{t=1}^{t_0}, (\e{t})_{t=1}^{t_0+1}, \Ux{t_0+1} \Big),
\end{split}
\end{align}
where, as defined earlier, $[\cdot]$ denotes the equivalence class of the sequences. 

Now, we are ready to prove that there exists a bijective mapping from $[\XE]$ to the set of all input orbits $\O^{(t_0)} = \{\ORB{ij} \colon \n{t_0} = 1,\cdots,t_0\, \a{t_0}{} \in \domA^{\n{t_0} \times \n{t_0}}, \ij \in \Vt{t_0} \times \Vt{t_0} \}$. More specifically, consider the bijection as taking an arbitrary temporal sequence $\big((\x{t})_{t=1}^{t_0}, (\e{t})_{t=1}^{t_0}\big) \in \big[\big((\x{t})_{t=1}^{t_0}, (\e{t})_{t=1}^{t_0}\big) \big] $, applying \Cref{def:g-obs} and computing the orbit of $ \e{t_0+1} = \ij$ in the generated graph $\G{t_0}$.
It is straightforward from \Cref{def:g-obs} that any other temporal sequence in $\big[\big((\x{t})_{t=1}^{t_0}, (\e{t})_{t=1}^{t_0+1}\big)\big]$ also generates $\ORB{ij}$. Also, for an arbitrary graph $\G{t_0}$ and $\ij \in \Vt{t_0} \times \Vt{t_0}$, consider the temporal sequence $\big((\x{t})_{t=1}^{t_0}, (\e{t})_{t=1}^{t_0}\big)$ where for any $t^\dagger \in \{1,\ldots,t_0\}$, s.t.\ either (a) $\exists \e{t^\dagger} \in \Et{t_0}$, $\x{t^\dagger} = \a{t_0}{\e{t^\dagger}}$, or (b) $\nexists \e{t^\dagger} \in \Et{t_0}$, and $\x{t'}=\textbf{0}$. It is clear from \Cref{def:g-obs} that this procedure would generate any graph isomorphic to $\G{t_0}$ simply by performing a permutation $\pi \in \S_{\n{t_0}}$. If $\pi$ is the identity permutation we generate the orbit $\ORB{ij}$ with $\big((\x{t})_{t=1}^{t_0}, (\e{t})_{t=1}^{t_0+1}\big)$ where $\e{t_0+1} = \ij$. Thus, \Cref{def:g-obs} defines a surjection from $[\XE]$ to $\O^{(t_0)}$. Next, we will show that it is also an injection.

Now, consider applying \Cref{def:g-obs} to an arbitrary temporal sequence $\big((\xp{t})_{t=1}^{t_0}, (\ep{t})_{t=1}^{t_0} \big)$, with $\ep{t_0+1} = (i',j') \in \Vt{t_0}{}\times \Vt{t_0}{}$.
This temporal sequence defines a (static) graph $G'$. Let $\cO'_{i'j'}$ be the orbit of $(i',j')$ in $G'$. To prove \Cref{def:g-obs} is also an injection, we need to show that if $\cO'_{i'j'} = \ORB{ij}$, then $\big((\xp{t})_{t=1}^{t_0}, (\ep{t})_{t=1}^{t_0} \big) \in \big[ \big((\x{t})_{t=1}^{t_0}, (\e{t})_{t=1}^{t_0} \big) \big]$.
If $\cO'_{i'j'} = \ORB{ij}$, then the graphs are isomorphic $G' \iso \G{t_0}$. Then, apart from having the same number of nodes and edges, there exists a permutation $\pi^+ \in \S_{\n{t_0}}$ such that $a' = \pi^+ \cdot \a{t_0}{}$. Now, let $\bar{\pi} \in \S_{\n{t_0}}$ be the permutation sampled in \Cref{def:g-obs} to generate $\a{t_0}{}$ from an arbitrary sequence $\big((\x{t})_{t=1}^{t_0}, (\e{t})_{t=1}^{t_0+1}\big) \in \big[ \big((\x{t})_{t=1}^{t_0}, (\e{t})_{t=1}^{t_0+1}\big) \big]$ and $\pi' \in \S_{\n{t_0}}$ the one used to generate $G'$. Then, we can define the permutation $\pi^\star := \bar{\pi}^{-1} \circ (\pi^{+})^{-1} \circ \pi'$ and use it to match the two edge sets $\big\{ (\x{t},\e{t}) \colon t=1,\cdots,t_0, \x{t} \neq \textbf{0} \big\} = \big\{ (\xp{t}, \pi^\star \cdot \ep{t}) \colon t=1,\cdots,t_0, \x{t} \neq \textbf{0} \big\}$. Finally, we can define a permutation $\pi^\dagger \in \S_{t_0+1}$ where $\pi \cdot t' = t$ for $ \pi^\star \cdot  \ep{t'} = \e{t} $. We then have
\begin{align*}
\begin{split}
G' \iso \G{t_0} \implies \big((\x{t})_{t=1}^{t_0} , (\e{t})_{t=1}^{t_0+1}\big)_{-\ecal{N}^{(t_0)}} = \big((\xp{t})_{t=1}^{t_0}, (\pi^\star \cdot \ep{\pi^\dagger \cdot t})_{t=1}^{t_0+1} \big)_{-\ecal{N}^{(t_0)}}
 \\ \implies
\big((\xp{t})_{t=1}^{t_0}, (\ep{t})_{t=1}^{t_0+1} \big) \in \big[\big((\x{t})_{t=1}^{t_0} , (\e{t})_{t=1}^{t_0+1}\big)\big],
\end{split}
\end{align*}
which implies that there exists a bijective mapping $q \colon [\XE] \to \O^{(t_0)}$.

Finally, we can now define $\bW_{\ORB{ij}}$ as a bijective mapping $\bW_{\ORB{ij}} \colon \O^{(t_0)} \to \R^d$. Since $q$ and $\bW_{\ORB{ij}}$ are bijective, they are invertible and thus we can define $h(a,b) = g( q^{-1} \circ \bW_{\ORB{ij}}^{-1}(a), b)$ and \Cref{eq:h} follows from \Cref{eq:g}.

ii.\ From i.\ we have that the input to $\Y{IJ}$ and $\Y{\pi \IJ}$ only differ on their input exogenous variables, which are sampled from their marginal distributions. From the theorem statement, their marginal distributions are the same and thus the random variables are equal everywhere.

\end{proof}

\subsection{Proof of \texorpdfstring{\Cref{col:estimator}}{Corollary 1}} \Cref{col:estimator} follows directly from the DAG in \Cref{thm:main} and the backdoor adjustment, see 3.2 in \cite{pearl2009causality}.

\subsection{Proof of \texorpdfstring{\Cref{col:multi-estimator}}{Corollary 2}} Note that the definition of non-interfering probes (cf.\ \Cref{def:y-int}) and the fact that exogenous variables are i.i.d. make the entire set of random variables $\{ \Yt{I}{J}{m} \colon 1 \leq m \leq M \}$ also i.i.d.. Then, \Cref{col:multi-estimator} follows directly from \Cref{col:estimator}.

\subsection{Proof of \texorpdfstring{\Cref{prop:2}}{Proposition 1} }
\begin{proof}
Here we are left to show that under the stated conditions the causal DAG from \Cref{fig:gamma-scm} can be equivalently represented by the one in \Cref{fig:w-scm}. Note that in \Cref{fig:gamma-scm} $\A{t_1}{IJ}$ is generated by a mechanism $f$ that takes as input $\bW_{\ORB{IJ}}$ and $\ecal{U}_{\IJ}$. Now, from \Cref{def:inv-pair-emb} we know that there exists a set of weights $\W'_{\Gamma^\star}$ that makes $\Gamma$ assign a unique representation to each orbit of nodes. Therefore, there exists a surjection  $s \colon \{ \Gamma(i,j,\a{t_0}{};\bW'_{\Gamma^\star}) \colon \ij \in \Vt{t_0} \times \Vt{t_0} \} \to \{\bW_{\ORB{ij}} \colon \ij \in \Vt{t_0} \times \Vt{t_0} \}$. Hence, we can define $\A{t_1}{IJ}$ in \Cref{fig:w-scm} as $\A{t_1}{IJ} = f(s(\Gamma(I,J,\a{t_0}{};\bW'_{\Gamma^\star}) ),\ecal{U}_{\IJ})$ and by the definition of $s$ and $f$ have the same SCM as \Cref{fig:gamma-scm}.

\end{proof}

\section{Graph Embeddings for Causal Link Prediction}\label{sec:appx-learn}

\subsection{Proof of \texorpdfstring{\Cref{thm:bias}}{Theorem 3}}

\begin{proof}

From the definition of \Cref{def:pair-struct}, we can see that most-expressive pairwise representations are under structural (joint) pairwise representations. Therefore, it follows from \cite{Che+2019} that with an expressive enough link function, \textit{e.g.}, multi-layer perceptron with a sufficient number of neurons, we can achieve zero-bias in the task. Now, we are left to show that representations using structural node representations cannot achieve zero-bias for in some outcome distributions.

From \Cref{def:inv-node-emb} we know that if the node pairs are node-wise isomorphic, \textit{i.e.}, $i \iso u, j \iso v$, we have that
\[ Z(i,a^{(t_0) } ; \bW_{Z} ) = Z(u,\a{t_0}{} ; \bW_{Z} ), \]
\[ Z(j,a^{(t_0) } ; \bW_{Z} ) = Z(v,\a{t_0}{} ; \bW_{Z} ). \]
Note that from the definition of isomorphic nodes there exists a permutation $\pi \in \aut{\G{t_0}}$ and a (possibly other) permutation $\pi'$ such that $u = \pi_i, v = \pi'_j$. Then, by defining $\Gamma$ as the concatenation ($[\cdot,\cdot]$) of the two structural node embeddings it follows that if $u \iso i, j \iso v$ we have that
\begin{align} \label{eq:isonodes}
\rho\bigg( \bigg[ Z(i,a^{(t_0) } ; \bW_{Z} ), Z(j,\a{t_0}{} ; \bW_{Z} ) \bigg] ; \bW_\rho \bigg) =
\rho\bigg( \bigg[ Z(u,a^{(t_0) } ; \bW_{Z} ), Z(v,\a{t_0}{} ; \bW_{Z} ) \bigg] ; \bW_\rho \bigg)
.
\end{align}
Now, in a pairwise symmetric graph there are at least two node pairs that are node-wise isomorphic, \textit{i.e.}, $u \iso i, j \iso v$, but not isomorphic, \textit{i.e.}, $(i,j) \not\iso (u,v)$. That is, in such graphs the above equality holds for $\ij$ and $\uv$ and any choice of parameters $\bW_Z, \bW_\rho$ and link function $\rho$. However, since $\ij \not\iso \uv$, there exist distributions such that $P( \Y{ij} = e ) \neq P( \Y{uv} = e ) $ for some $e \in \domA$. Finally, from \Cref{eq:isonodes} we know that structural embedding models cannot output different distributions for $\Y{ij}$ and $\Y{uv}$ and thus cannot achieve zero-bias for such an outcome distribution.

Now, to show that the bound is relevant to our task it is left for us to prove that there exists at least one pairwise symmetric graph $\G{t_0}$ generated by a model $\m \in \M$. Note that the model restrictions are about the generation process after time $t_0$, so we can leverage \Cref{thm:main} directly. It follows from \Cref{thm:main}'s proof that the results i.\ and ii.\ also hold for graphs $\G{t_0}$ observed at a fixed time $t_0$ when $n \leq \sqrt{t_0}$. Consider a graph $\G{t_0}, n \leq \sqrt{t_0}$ that has zero probability in every $\m \in \M$. Now, if $G'^{(t_0)}, n \leq \sqrt{t_0}$ is a graph non-isomorphic to it, it follows from \Cref{thm:main} that there exists a model $\m' \in \M$ that generates them with different probabilities, that is, $G'^{(t_0)}$ can be generated by $\m'$. Therefore, \Cref{thm:main} guarantees that at most one graph with $n \leq \sqrt{t_0}$ (and the others isomorphic to it) cannot be generated by any model $\m \in \M$ at time $t_0$. In \cite{zhu2021node}[Theorem 4.2] the authors show that any graph with two isomorphic components has nodes $i,j,u,v$ satisfying \Cref{eq:isonodes}. Thus, since for a number of nodes $n \geq 6$ this class has more than one graph, there exists a $\G{t_0}$ with nodes $i,j,u,v$ generated by a model $\m \in \M$  with $t_0 \geq 36$ where \Cref{eq:isonodes} is satisfied.

\end{proof}

\subsection{Proof of \texorpdfstring{\Cref{thm:svd}}{Theorem 4}}

\begin{proof}
Let $P_\pi$ be the permutation matrix of $\pi \in \S_n$ and $\theta^{(\text{SVD},r)}_i$ the SVD embedding of node $i$ with respect to the right eigenvectors and $\theta^{(\text{SVD},\ell)}_i$ with respect to the left. Remember that the embedding of a node $\theta_i$ is considered as the concatenation of both $\theta^{(\text{SVD},r)}_i$ and $\theta^{(\text{SVD},\ell)}_i$.

\begin{itemize}[leftmargin=*]
    \item \textbf{Same embeddings} $\implies P_\pi aa^T = aa^T  P_\pi = aa^T $ and $P_\pi a^Ta = a^Ta P_\pi = a^Ta.$
    
    Let $ \Omega(G) := \{ \pi \in \S_n \colon \pi_i = j , \theta^{(\text{SVD},r)}_i = \theta^{(\text{SVD},r)}_j, \theta^{(\text{SVD},\ell)}_i = \theta^{(\text{SVD},\ell)}_j \} $ be the set of permutations that map nodes to other nodes with the same SVD embeddings. Then, for every eigenvector $x$ with corresponding eigenvalue $\lambda$ of $aa^T$ and $\pi \in \Omega(G)$,
    $$ aa^T x = \lambda x, $$
    and since $P_{\pi}x = x$ we have
    $$ aa^T P_{\pi}x = \lambda x, $$
    and
    $$ P_{\pi}aa^Tx = \lambda x.$$
    That is, $P_{\pi} aa^T$, $aa^TP_{\pi}$ and $aa^T$ all have the same set of right eigenvectors and corresponding eigenvalues, thus
    $$ P_\pi aa^T = aa^T P_\pi = aa^T. $$
    
    Note that the same procedure can be applied to the eigenvectors of $a^Ta$ and thus 
    $$ P_\pi aa^T = P_\pi aa^T= aa^T, P_\pi a^Ta = a^Ta P_\pi = a^Ta, \; \forall \pi \in \Omega .$$
    
    \item  \textbf{Same embeddings} $\impliedby P_\pi aa^T = P_\pi aa^T= aa^T $ and $P_\pi a^Ta = a^Ta P_\pi = a^Ta.$

    Then, for every eigenvector $x$ with corresponding eigenvalue $\lambda$ of $aa^T$ and $\pi \in \omega(G)$,
    $$ aa^T x = \lambda x, $$
    multiply by $P_\pi$ on both sides
    $$ P_\pi aa^T x = \lambda P_\pi x. $$ 
    Now, since $ P_\pi aa^T = aa^T $ we have that
    $$ aa^T x = \lambda P_\pi x, $$
    which implies that $P_\pi x= x $ for every eigenvector $x$ of $aa^T$ and $\pi \in \omega(G)$. Note that again the exact same procedure can be applied to $a^Ta$ and thus $\omega(G) = \Omega(G)$.
    
    \item $ P_\pi aa^T =  aa^T, P_\pi a^Ta =  a^Ta \iff a_{kv} = a_{\ell v},  a_{vk} = a_{v\ell} \quad \forall v \in V, k,\ell \in V \colon k \iso \ell.$
    
    \begin{itemize}
        \item  First, if $ aa^T = aa^TP_\pi = P_\pi aa^T $ for $\pi \in \Omega(G)$,
        $$(aa^T)_{ii} = \sum_{v \in V} a_{iv}a_{iv}, $$
        and
        $$ (aa^T)_{\pi_i i} = \sum_{v \in V} a_{\pi_i v}a_{i v}, $$
        which implies that $a_{iv} = a_{\pi_i v} \text{ } \forall v \in V, \pi \in \Omega(G)$.
        \item Now, we apply the same procedure leveraging $P_\pi a^Ta$
        $$(a^Ta)_{ii} = \sum_{v \in V} a_{vi}a_{vi}, $$
        and
        $$ (a^Ta)_{ \pi i } = \sum_{v \in V} a_{v \pi_i}a_{vi}, $$
        which implies that $a_{vi} = a_{v \pi_i} \text{ } \forall v \in V, \pi \in \Omega(G)$.
        \item The other direction is satisfied straightforwardly.
    \end{itemize}
    Note from the last item that $\Omega(G) \subseteq \S_n$ is the set of permutations that swaps only nodes with identical neighborhoods. From that it follows that $\Omega(G) \subseteq \aut{G}$ is a subset of the automorphisms of $G$ as well, thus only isomorphic nodes can get the same SVD embeddings. Overall, two nodes get the same SVD embedding if they have the exact same neighborhood.
\end{itemize}
\end{proof}

\subsection{Proof of \texorpdfstring{\Cref{col:strict-svd}}{Corollary 3}}
\begin{proof}
    As mentioned in the main text, \Cref{col:strict-svd} follows directly from \Cref{thm:svd} and the definition of strictly positional node embeddings (cf.\ \Cref{def:strict}).
\end{proof}

\section{Can structure capture the link formation process of real-world networks?}\label{sec:res-struct}

Here we investigate \textbf{Q4} from \Cref{sec:res}, \textit{i.e.}, the extent to which a shared set of parameters representing node pairs' structure can encapsulate the link formation process of real-world graphs. More specifically, we are interested in testing whether $\Y{IJ} \eqdist \Y{UV} $ when $\IJ$ and $\UV$ are structurally similar. In order to select node pairs with such a property, we consider an arbitrary pair $\IJ$ in the test set and the pair $\UV$ such that $\Gamma(I,J, \A{t_0}{};\bW^\star_\Gamma);\bW_\rho\big) \approx \Gamma(U,V, \A{t_0}{};\bW^\star_\Gamma);\bW_\rho\big)$, where $\Gamma(a,b, \A{t_0}{};\bW^\star_\Gamma)$ is the trained Label GCN pairwise embedding space of pair $(a,b) \in \Vt{t_0}\times \Vt{t_0}$ on graph $\A{t_0}{}$. As an alternate hypothesis, we also consider the case where $\UV$ is a node pair selected uniformly at random (from the test set as well). Note, however, that we only have a sample of $\Y{IJ}$ and $\Y{UV}$. Thus, to construct the probe outcome distribution of each node pair we consider the outcome of its $10$ closest neighbors in the Label GCN pairwise embedding space. Since the two tested node pairs might share neighbors and directly induce the same distribution, we only consider the non-intersecting sets of $10$ neighbors ---if a node pair is a neighbor of both tested pairs, we assign it to the closest tested pair. 
To test if the distributions are the same ($\Y{IJ} \eqdist \Y{UV} $), we use Fisher's exact test~\cite{fisher1922interpretation} with a significance level of $0.05$. Results are shown for all node pairs in the test set of both AE and LFM datasets from \Cref{sec:res-rec} in \Cref{fig:e10,fig:m10}. We can see how indeed structurally similar node pairs tend to have the same probe outcome distribution, in contrast to pairs selected at random, giving credence to the central assumption in our work that in some real-world tasks structural similarity before probing also implies similarity of probe outcomes. 

\begin{figure*}[ht]
\centering
\begin{minipage}[t]{.45\textwidth}
\centering
\includegraphics[width=1\linewidth]{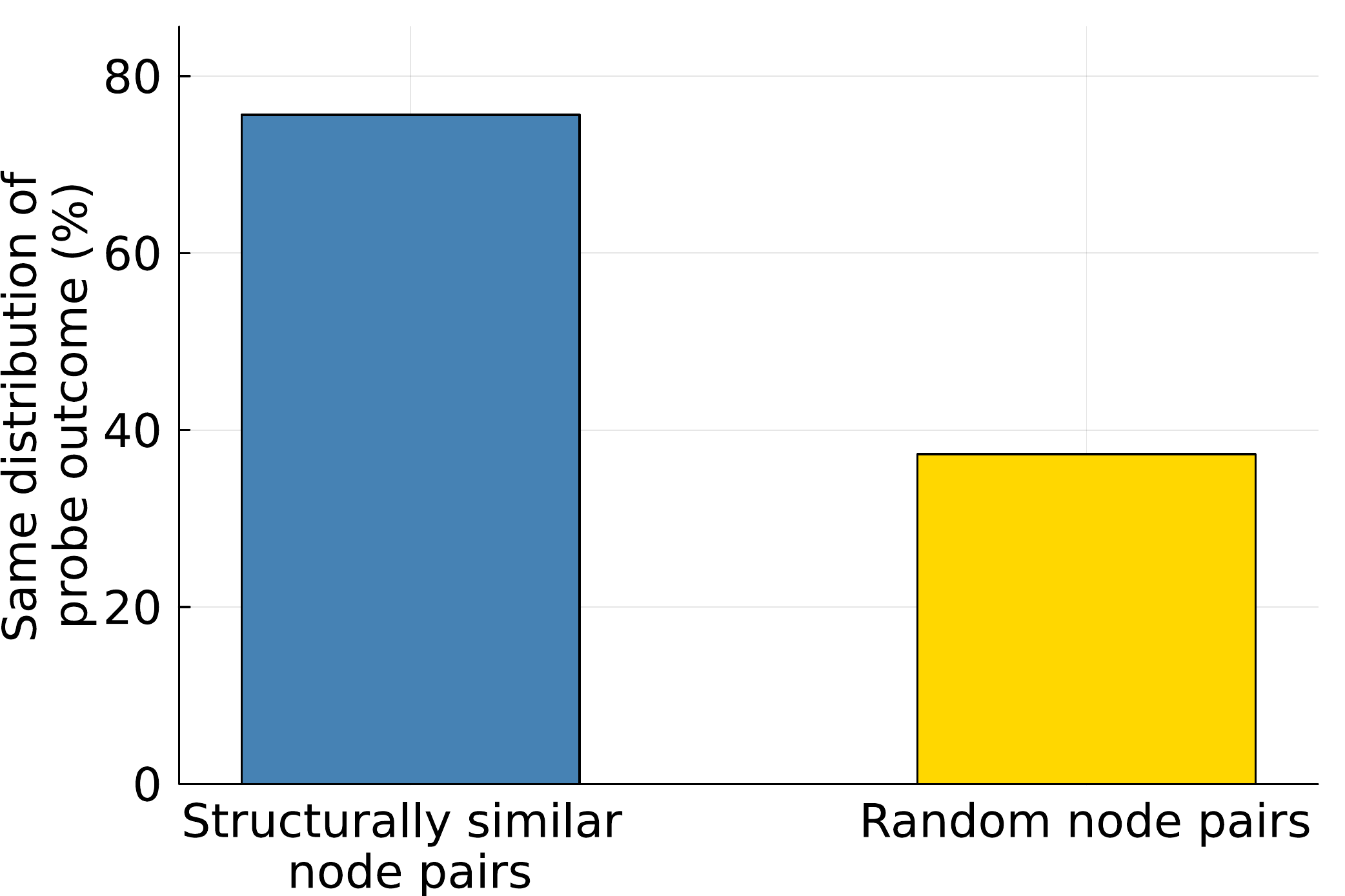}
\caption{Results for the AE dataset. We show the percentage of pairs of node pairs that are structurally similar and have the same distribution vs.\ the ones that are paired uniformly at random.}
\label{fig:e10}
\end{minipage}\qquad
\begin{minipage}[t]{.45\textwidth}
\centering
\includegraphics[width=1\linewidth]{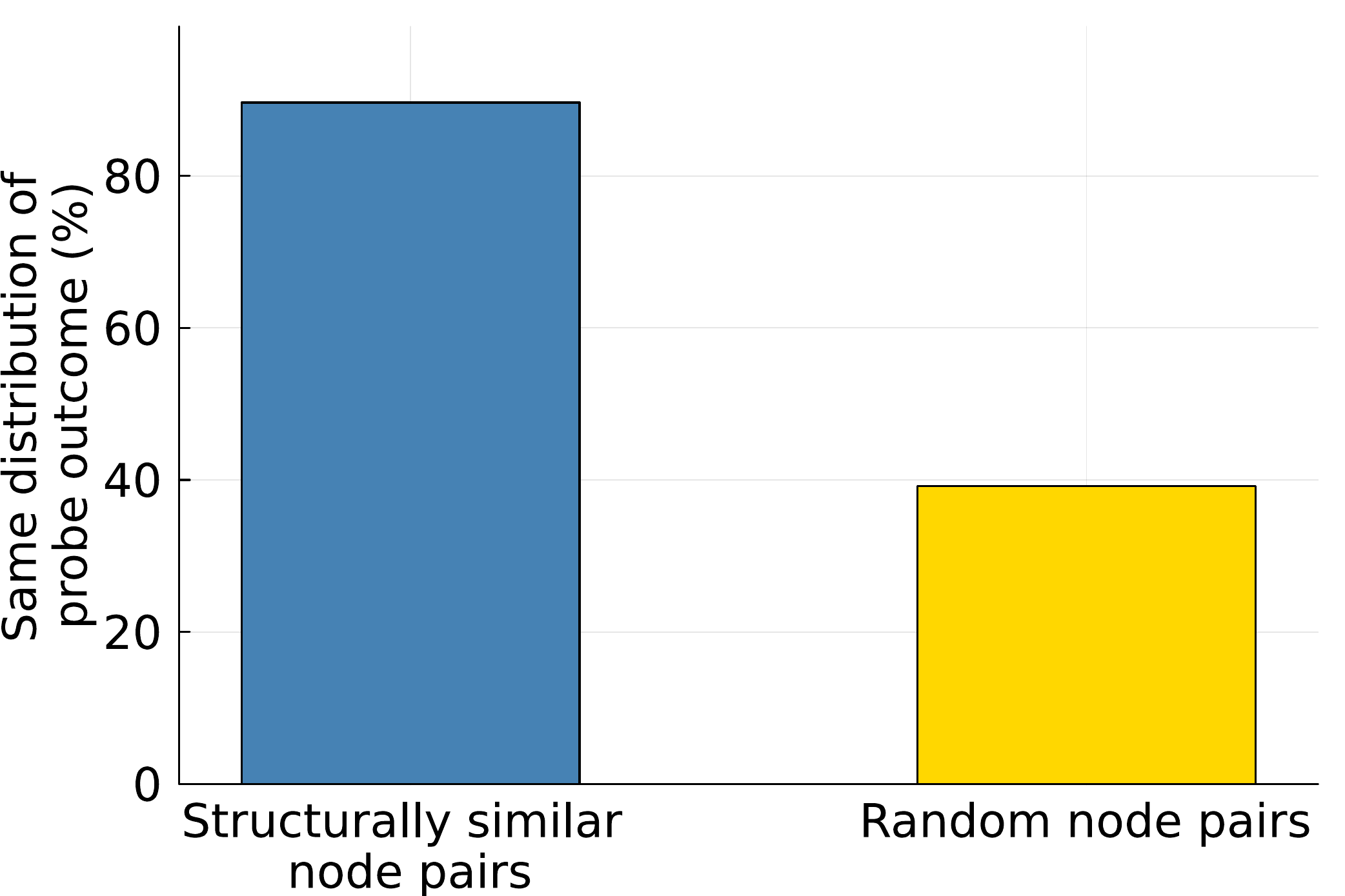}
\caption{Results for the LFM dataset. We show the percentage of pairs of node pairs that are structurally similar and have the same distribution vs.\ the ones that are paired uniformly at random.}
\label{fig:m10}
\end{minipage}
\end{figure*}

\section{Datasets and Models}\label{sec:data}
We conducted our experiments on four datasets: (1) Family tree, (2) Covariance matrix, (3) Amazon Electronics (AE) and (4) Last FM (LFM) datasets. These datasets allowed us to evaluate our findings in different settings, since they are diverse in terms of number of nodes and sparsity of their graphs. We compared (strictly) positional node embeddings, structural node embeddings and structural pairwise node embeddings, whose representative architectures are described in the main text. For the family tree dataset, we also considered knowledge graph embeddings. All models are implemented in Pytorch~\cite{pytorch} and Pytorch-Geometric~\cite{fey2019fast}. Training was done on NVIDIA GeForce RTX 2080 Ti, GeForce GTX 1080 Ti, TITAN V, and TITAN Xp GPUs. Details on the datasets, models and hyperparameter grid can be found in what follows.

\subsection{Datasets}
\paragraph{Family tree} We generated the dataset by first creating family trees using the codebase and the ontology provided in \cite{Hohenecker2020}, and then by combining them to ensure a percentage of the family trees contain two isomorphic subtrees.
Following \cite{Hohenecker2020}, each family tree is generated incrementally, starting from a single person, and it is grown by randomly selecting a person to whom a child or a parent is added. This process is repeated until the family tree reaches the maximum size of 26 nodes or the generation is randomly stopped, which happens with probability 0.002. Furthermore, each tree is constrained to a maximum depth of 5 and a maximum branching factor of 5, and the generated family trees are ensured to be non-isomorphic. After generating the family trees, we randomly selected 100 of the generated trees. In each of those, we randomly selected a leaf node and randomly chose from the previously not selected trees, two trees to be connected to the leaf or one tree to be connected twice, with probabilities $0.7$ and $0.3$ respectively. After this procedure, our dataset contains 100 family trees, with 30\% of them containing two isomorphic subtrees. Finally, we use the ontology in \cite{Hohenecker2020} to compute all possible inferences for each
tree in the dataset via symbolic reasoning. Only the \texttt{parentOf} relation is used as the observed graph, and the goal is to predict the other inferred relations (27 in total). We split the isomorphic subtrees between train and test such that relations within one subtree form the probes' (train) links, and the relations in the other subtree the counterfactual (test) links. In training nonedges are sampled uniformly at random from the training subtree. As for test, the nonedges are built taking one endpoint in the training subtree and another in the test subtree ---to get nonedges that are node-wise isomormophic to edges in the test set.
\\~\\
\paragraph{Covariance matrix} We constructed the dataset making use of the data collected in \cite{guvenir1997supervised}, available in the UCI repository~\cite{Dua2019uci}. The data contains measurements values for $279$ attributes from $452$ patients. We considered only the patients that do not have missing values in any measurement, and we removed nominal and zero variance attributes. After this pre-processing step, our data comprises $68$ patients and $182$ attributes. We used a subset of 40 patients (selected at random) and computed the covariance matrix between the attributes, which is used as the observed adjacency matrix. Then, we recompute the covariance matrix with all the $68$ patients and we split the attributes into two disjoint sets of size $75\%$ and $25\%$ of the total which are used as probes' (train) links and counterfactual (test) links respectively. 
\\~\\
\paragraph{Amazon Electronics and Last FM} We considered the Amazon Electronics (AE)~\cite{wan2020addressing} and the Last FM (LFM)~\cite{melchiorre2021investigating} datasets.
The observed graph is obtained by considering user-item interactions occurring from 11/24/2015 to 12/24/2015 in AE and from 2007 until 2013 in LFM. In train and test we use interactions happening between 12/24/2015 and 12/31/2015 for AE and in 2014 for LFM. In both datasets we experiment with the subgroup of male users, while, at test time, our counterfactual queries are about female users. Note that, as mentioned in the main text, nonedges are obtained by sampling nodes uniformly at random.

\subsection{Trained models}
\noindent For all node embedding models we use as link function a Hadamard product followed by a multi-layer perceptron with one hidden layer, of the same size as the embedding, and ELU activations. In the following, we discuss details of embedding architectures and their hyperparameters.

\subsubsection{Positional Node Embeddings} 
\noindent Below we discuss the architecture's choices for each dataset.
\\~\\
\paragraph{Family tree} We evaluated Nonnegative Matrix Factorization (NMF)~\cite{lee1999learning} using the default implementation provided in Scikit-Learn~\cite{scikit-learn}. For SVD~\cite{halko2009finding} we make use of the low-rank implementation available in Pytorch~\cite{pytorch}. Both models generated embeddings of size $256$.
We implemented positional GCN on top of a GCN architecture by using an additional embedding layer with the same size of the GNN layers. We used $3$ GCN layers with dimension $256$. Our models are trained with batches of sizes $32$ (NMF) and $1024$ (positional GCN) and an Adam optimizer with learning rate $0.005$.
\\~\\
\paragraph{Covariance matrix} Same as \emph{Family tree} for SVD, but with an embedding of size $64$ and a learning rate of $0.01$.
\\~\\
\paragraph{Amazon Electronics} Same as \emph{Family tree} for NMF and SVD but with an embedding of size $8$. Positional GCN uses $2$ GNN layers with dimension $8$. Our models are trained with batches of size $32$ and an Adam optimizer with learning rate $0.005$.
\\~\\
\paragraph{Last FM} Same as \emph{Family tree} for NMF and SVD and positional GCN. Our models are trained with batches of size $32$ and an Adam optimizer with learning rate $0.001$.

\subsubsection{Structural Node Embeddings} 
\noindent Below we discuss the architecture's choices for each dataset.
\\~\\
\paragraph{Family tree} We considered a GCN architecture with $3$ layers with dimension $256$. Our models are trained with batches of size $32$ and an Adam optimizer with learning rate $0.005$.
\\~\\
\paragraph{Covariance matrix} We considered a GCN architecture with $3$ layers with dimension $64$. We additionally make use of batch norm between each convolutional layer. Our models are trained with batches of size $16$ and an Adam optimizer with learning rate $0.01$.
\\~\\
\paragraph{Amazon Electronics} We considered a GCN architecture with $2$ layers with dimension $8$. Our models are trained with batches of size $32$ and an Adam optimizer with learning rate 0.005.
\\~\\
\paragraph{Last FM} We considered a GCN architecture with $3$ layers with dimension $256$. Our models are trained with batches of size $32$ and an Adam optimizer with learning rate $0.001$.

\subsubsection{Knowledge Graph Embeddings}
\noindent For the \emph{Family tree} dataset we used the Torch KGE~\cite{boschin2020torchkge} implementation of ComplEx, TransE and DistMult. Note that like SVD and NMF, the KGE embeddings are obtained through a pre-training procedure ---the embeddings are then used to train a link function with \Cref{eq:obj} as our task's objective. The pre-training procedure uses Torch KGE implementation with a learning rate of $0.005$, $0.00001$ as the $L2$ regularization parameter and $2000$ as the batch size.

\subsubsection{Structural Pairwise Node Embeddings}
\noindent Below we discuss the architecture's choices for each dataset.
\\~\\
\paragraph{Family tree} We use the same training hyperparameters as the GCN model, except for LabelGCN where we reduce the batch size to $8$. For SEAL we use a dimension of $32$ and $k$-hops subgraphs around the node pairs with $k=3$ ---other parameters follow the original work default implementation. For Neo-GNN, we used $3$ GCN layers of dimension $32$, learning rate of $0.005$, batch size $32$, paths of length $2$, node dimension $128$ and edge dimension $8$ --- all other parameters follow the original work default implementation.
\\~\\
\paragraph{Covariance matrix} We use LabelGCN as representative of structural pairwise embeddings. We use $3$ GCN layers of size $64$ with batch norm. We train with batches of size $16$ and an Adam optimizer with learning rate $0.01$.
\\~\\
\paragraph{Amazon Electronics} SEAL uses $k$-hops subgraphs around the node pairs with $k=1$ and a model composed of $2$ GCN layers with dimension $32$. The model is trained with batches of size $32$ and an Adam optimizer with learning rate $0.005$ and weight decay of $0.00005$. For LabelGCN we instead use a model with $2$ GCN layers with dimension $8$. The model is trained with batches of size $8$ and an Adam optimizer with learning rate $0.005$. As for Neo-GNN, we used an $L2$ regularization of $0.00005$, a learning rate of $0.005$, batch size $32$, $2$ GNN layers with dimension $32$, paths of length $1$, node dimension $128$ and edge dimension $8$ --- all other parameters follow the original work verbatim.
\\~\\
\paragraph{Last FM} For LabelGCN we use a model with $3$ GCN layers with dimension $256$. The model is trained with batches of size $10$ and an Adam optimizer with learning rate $0.001$. For SEAL we use a dimension of $16$, $3$ GCN layers, and $k$-hops subgraphs around the node pairs with $k=1$ ---other parameters follow the original work default implementation. For Neo-GNN, we used $2$ GCN layers with dimension $16$, a learning rate of $0.001$, batch size $10$, paths of length $1$, node dimension $128$ and edge dimension $8$ --- all other parameters follow the original work default implementation.

\section{Limitations and Possible Extensions}
Finally, we would like to highlight two limitations of our work that could be explored in future research. 
First, our family of SCMs is designed to model systems with causal link prediction tasks in mind, that is, they are universal models with respect to the observational distribution of the graph, designed to answer causal link prediction queries. For tasks focused on intervening, for instance, in higher-order structures, the query cannot be trivially formulated through our SCMs. Possible extensions could include higher-order generalizations of \Cref{def:scm}. Finally, we also highlight how \Cref{ass:time-igno}, a central assumption in our work, is not satisfied in some specific scenarios. In particular, in very large, highly dynamical, graphs nodes may significantly evolve in the time between a probe and its observed effect. In these scenarios, \Cref{ass:time-igno} may be violated and further work is needed in order to address these cases.

\end{document}